\documentclass[11pt]{article}
\usepackage{latexsym,amsfonts,amsmath,amsthm,amssymb, epsfig}
\usepackage{color}
\usepackage{mathtools}
\usepackage{geometry}
\usepackage{enumitem}
\usepackage{subcaption}
\usepackage{tikz}
\usepackage{hyperref}

\newtheorem{thm}{Theorem}
\newtheorem{lem}[thm]{Lemma}
\newtheorem{prop}[thm]{Proposition}
\newtheorem{cor}{Corollary}

\theoremstyle{plain}
\newtheorem{rem}{Remark}

\newtheorem{dfn}[thm]{Definition}

\usepackage[linesnumbered,algoruled,boxed,lined]{algorithm2e}
\makeatletter
\g@addto@macro{\@algocf@init}{\SetKwInOut{Parameter}{Parameters}}
\makeatother

\usepackage[framemethod=default]{mdframed}
\global\mdfdefinestyle{exampledefault}{%
	linecolor=lightgray,linewidth=1pt,%
	leftmargin=1cm,rightmargin=1cm,
}

\newif\ifhideproofs

\ifhideproofs
\usepackage{environ}
\NewEnviron{hide}{}

\fi

\newif\ifhideproofrems
\hideproofremstrue

\ifhideproofs
\NewEnviron{hide1}{}

\fi

\title{Robust and Resource Efficient Identification of Two Hidden Layer Neural Networks}
\author{Massimo Fornasier$^1$, Timo Klock$^2$, and Michael Rauchensteiner$^3$}
\date{%
    $^1$Department of Mathematics, Boltzmannstrasse 3, 85748, Garching, Germany,\\ Email: \url{massimo.fornasier@ma.tum.de}\\%
    $^2$Simula Research Laboratory, Machine Intelligence Department, Oslo, Norway, \\Email: \url{timo@simula.no}\\
    $^3$Department of Mathematics, Boltzmannstrasse 3, 85748, Garching, Germany, \\Email: \url{michael.rauchensteiner@ma.tum.de}\\[2ex]%
    \today
}
\newcommand{\Nouterlayer}{m_1}
\newcommand{\Ninnerlayer}{m_0}
\newcommand{\Lframe}{c_f}
\newcommand{\Uframe}{C_F}
\newcommand{\Lriesz}{c_r}
\newcommand{\Uriesz}{C_R}
\newcommand{\Perror}{\delta}
\newcommand{\Ferror}{\nu}
\newcommand{\Fiter}{F_{\gamma}}
\newcommand{\Mspaceiter}{\CM_{\epsilon}}
\newcommand{\vertiii}[1]{{\left\vert\kern-0.25ex\left\vert\kern-0.25ex\left\vert #1
    \right\vert\kern-0.25ex\right\vert\kern-0.25ex\right\vert}}

\newcommand{\R}{\mathbb{R}}

\newcommand{\supp}{{\rm supp\ }}

\newcommand{\argmax}{\arg \max}

\newcommand{\opvec}{\operatorname{vec}}

\newcommand{\opspan}{\operatorname{span}}

\DeclarePairedDelimiter{\norm}{\lVert}{\rVert}
\DeclarePairedDelimiter{\abs}{\lvert}{\rvert}
\DeclareMathOperator*{\Unif}{Unif}
\DeclarePairedDelimiterX\set[1]\lbrace\rbrace{\def\given{\;\delimsize\vert\;}#1}

\newcommand{\ve}[1]{#1}
\newcommand{\ma}[1]{#1}
\providecommand*{\Uni}[1]{\operatorname{Uni}({#1})}

\providecommand{\Dim}{\operatorname{dim}}            
\providecommand{\dim}{\Dim}

\providecommand{\supp}{\Supp}

\providecommand{\rank}{\operatorname{rank}}                        
\providecommand{\Rank}{\rank}
\providecommand{\argmax}{\operatorname*{argmax}}  
\providecommand{\Id}{\Op{Id}}                     

\providecommand{\diag}{\operatorname{diag}}

\providecommand{\CA}{{\cal A}}

\providecommand{\CD}{{\cal D}}
\providecommand{\CE}{{\cal E}}
\providecommand{\CF}{{\cal F}}

\providecommand{\CI}{{\cal I}}

\providecommand{\CL}{{\cal L}}
\providecommand{\CM}{{\cal M}}
\providecommand{\CN}{{\cal N}}
\providecommand{\CO}{{\cal O}}

\providecommand{\CW}{{\cal W}}

\providecommand{\bbB}{\mathbb{B}}

\providecommand{\bbE}{\mathbb{E}}

\providecommand{\bbN}{\mathbb{N}}

\providecommand{\bbP}{\mathbb{P}}

\providecommand{\bbR}{\mathbb{R}}
\providecommand{\bbS}{\mathbb{S}}

\providecommand*{\N}[1]{\left\|{#1}\right\|} 
\newcommand*{\SN}[1]{\left|{#1}\right|}      
\providecommand*{\abs}[1]{\left|{#1}\right|} 

\newcommand*{\Op}[1]{\mathsf{#1}} 

\begin{document}
\maketitle
\begin{abstract}
We address the structure identification and the uniform approximation of two fully nonlinear layer neural networks of the type
$f(x)=1^T h(B^T g(A^T x))$ on $\mathbb R^d$, where $g=(g_1,\dots, g_{m_0})$, $h=(h_1,\dots, h_{m_1})$, $A=(a_1|\dots|a_{m_0}) \in \mathbb R^{d \times m_0}$ and $B=(b_1|\dots|b_{m_1}) \in \mathbb R^{m_0 \times m_1}$,
from a small number of  query samples. The solution of the case of two hidden layers presented in this paper is crucial as it can be further generalized to deeper neural networks.
We approach the problem by sampling actively finite difference approximations to
Hessians of the network. Gathering several approximate Hessians allows reliably to approximate the matrix subspace $\mathcal W$ spanned by symmetric tensors  $a_1 \otimes a_1  ,\dots,a_{m_0}\otimes a_{m_0}$ formed
by weights of
the first layer together with the entangled symmetric tensors $v_1 \otimes v_1  ,\dots,v_{m_1}\otimes v_{m_1}$, formed by suitable combinations of the weights of the first and second layer as $v_\ell=A G_0 b_\ell/\|A G_0 b_\ell\|_2$, $\ell \in [m_1]$, for a  diagonal matrix $G_0$ depending on the activation functions of the first layer. The identification of the 1-rank symmetric tensors within $\mathcal W$ is then performed by the solution of a robust nonlinear program, maximizing the spectral norm of the competitors constrained over the unit Frobenius sphere. We provide guarantees of stable recovery under a posteriori verifiable conditions. Once the 1-rank symmetric tensors $\{a_i \otimes a_i,  i\in [m_0]\}\cup \{v_\ell \otimes v_\ell, \ell \in [m_1] \}$ are computed, we address their correct attribution to the first or second layer ($a_i$'s are attributed to the first layer). The attribution to the layers is currently based on a semi-heuristic reasoning, but it shows clear potential of reliable execution.
Having the correct attribution of the $a_i,v_\ell$ to the respective layers and the consequent de-parametrization of the network, by using a suitably adapted gradient descent iteration, it is possible to estimate, up to intrinsic symmetries, the shifts of the activations functions of the first layer and compute exactly the matrix $G_0$.
Eventually, from the vectors $v_\ell=A G_0 b_\ell/\|A G_0 b_\ell\|_2$'s and $a_i$'s one can disentangle the weights  $b_\ell$'s, by simple algebraic manipulations. Our method of identification of the weights of the network is fully constructive, with quantifiable sample complexity, and therefore contributes to dwindle the black-box nature of the network training phase. We corroborate our theoretical results by extensive numerical experiments, which confirm the effectiveness and feasibility of the proposed algorithmic pipeline.
\end{abstract}

\tableofcontents

\section{Introduction}
\label{sec:intro}

Deep learning is perhaps one of the most sensational scientific and technological developments in the industry of the
last years. Despite the spectacular success of deep neural networks (NN) outperforming other pattern recognition methods,
achieving even superhuman skills in some domains \cite{Intro1, Intro3, Intro2}, and confirmations of empirical successes in other areas such as speech recognition \cite{Intro4}, optical charachter recognition \cite{Intro5},
games solution \cite{Intro7,Intro6},
  the mathematical understanding of
 the technology of machine learning is in its infancy. This is not only unsatisfactory
from a scientific, especially mathematical point of view, but it also means that deep
learning currently has the character of a black-box method and its success can not be ensured
yet by a full theoretical explanation. This leads to lack of acceptance in many areas,
where interpretability is a crucial issue (like security, cf. \cite{Intro11}) or for those applications where one wants to extract
new insights from data \cite{Intro12}.

Several general mathematical results on neural networks have been available since the 90's
\cite{anba99,deospe97,li02,li92,pe99,pi97,pi99}, but deep neural networks have special features and in particular
 superior properties in applications that still can not be fully explained from the known results.
In recent years, new interesting mathematical insights have been derived for undestanding
approximation properties (expressivity) \cite{grpeelbo19,shclco15} and stability properties \cite{bruma13,wigrbo18}  of deep neural networks.  Several other crucial and
challenging questions remain open.

A fundamental one is about the number of required training data to obtain a {\it good} neural network, {\it i.e.},  achieving small generalization errors for future data. Classical statistical learning theory splits this error into bias and variance
and gives general estimations by means of the so-called VC-dimension or
Rademacher complexity of the used class of neural networks \cite{shbe14}. However, the currently available
 estimates of these parameters \cite{gorash18} provide very pessimistic barriers in
comparison to empirical success.
In fact, the tradeoff between bias and variance is function of the complexity of a network, which should be estimated
by the number of sampling points to identify it uniquely.
Thus, on the one hand, it is of interest to know which neural networks can be uniquely determined in a stable way by finitely many training points.
On the other hand, the unique identifiability is clearly a form of interpretability. \\
The motivating  problem of this paper is the robust and resource efficient identification of feed forward neural networks.
Unfortunately, it is known that identifying a very simple (but general enough) neural network is indeed NP-hard \cite{BR92,Judd}.
Even without invoking fully connected neural networks, recent work \cite{Fornasier2012,maulvyXX} showed that even the training of one single
neuron (ridge function or single index model) can show any possible degree of intractability, depending on the distribution of the input. Recent results \cite{ba17,Kaw16,memimo19,SoCa16,rova18},
on the other hand, are more encouraging, and  show that minimizing a square loss of a (deep) neural network does not have in general
or asymptotically (for large number of neurons)
poor local minima, although it may retain the presence of critical saddle points.\\
In this paper we present conditions for a fully nonlinear two-layer neural network to be provably  identifiable with a number of samples,
which is polynomially depending  on the dimension of the network. Moreover, we prove that our procedure is robust to perturbations. Our result is clearly of theoretical  nature, but also fully constructive and easily implementable.
To our knowledge, this work is the first, which allows provable de-parametrization of the problem of deep network identification, beyond the simpler case of shallow (one hidden) layer neural networks already considered in very recent literature \cite{ba17,anandkumar15, fornasier2018identification,Kaw16,memimo19,MondMont18,SoCa16,rova18}.
For the implementation we do not require black-box high dimensional optimization methods and no concerns about complex energy loss landscapes
need to be addressed, but only classical and relatively simple calculus and linear algebra tools are used
(mostly function differentiation and singular value decompositions). The results of this paper build upon the work \cite{Fornasier2012,fornasier2018identification}, where the approximation from a finite number of sampling points have been already derived for the single neuron and one-layer neural networks. The generalization of the approach of the present paper to networks with more than two hidden layers is suprisingly simpler than one may expect, and it is in the course of finalization \cite{FFR20}, see Section \ref{sec:openproblems} (v) below for some details.

\subsection{Notation}\label{sec:notation}

Let us collect here some notation used in this paper. Given any integer $m \in \mathbb N$, we use the symbol $[m]:=\{1,2,\dots,m \}$ for indicating the index set of the first $m$ integers. We denote $B_1^d$ the Euclidean unit ball in $\mathbb R^d$, $\mathbb S^{d-1}$ the Euclidean sphere, and $\mu_{\mathbb S^{d-1}}$ is its uniform probability measure. We denote $\ell_q^d$ the $d$-dimensional Euclidean space endowed with the norm $\|x\|_{\ell_q^d} =\left (  \sum_{j=1}^d |x_j|^q \right )^{1/q}$. For $q=2$ we often write indifferently  $\|x\| = \|x\|_{2}= \|x\|_{\ell_2^d}$. For a matrix $M$ we denote $\sigma_k(M)$ its $k^{th}$ singular value. We denote $\mathbb S$ the sphere of symmetric matrices of unit Frobenius norm $\|\cdot \|_F$.  The spectral norm of a matrix is denoted $\|\cdot\|$. Given a closed convex set $C$ we denote $P_C$ the orthogonal projection operator onto $C$ (sometimes we use such operators to project onto subspaces of $\mathbb R^d$ or subspaces of symmetric matrices or onto balls of such spaces). For vectors $x_1,\dots, x_k \in  \mathbb R^d$ we denote the tensor product $x_1 \otimes \dots \otimes x_k$ as the tensor of entries $({x_1}_{i_1} \dots {x_k}_{i_k})_{i_1,\dots,i_k}$. For the case of $k=2$ the tensor product $x \otimes y$ of two vectors  $x,y \in  \mathbb R^d$ equals the matrix $x y^T =  (x_i y_j)_{ij}$. For any matrix $M \in \R^{m \times n}$
\begin{align}
\opvec(M) := (m_{11}, m_{21}, \dots, m_{m1}, m_{12}, m_{22}, \dots, m_{mn})^T \in \R^{mn},
\end{align}
is its vectorization, which is the vector created by the stacked columns of $M$.

\subsection{From one artificial neuron to shallow, and deeper networks}\label{sec:1neuron2nets}

\subsubsection{Meet the neuron}\label{sec:1neuron}

The simplest artificial neural network $f:\Omega \subset \R^d \to \R$ is a network consisting of exactly one artificial neuron, which is modeled by a ridge-function (or single-index model) $f$ as
\begin{align}\label{eq:oneneuron}
f(x) = \phi(a^T x + \theta) =g(a^T x),
\end{align}
where $g:\R \to \R$ is the  shifted \emph{activation function} $\phi( \cdot+ \theta)$ and the vector $a \in \R^d$ expresses the \emph{weight} of the neuron.  Since the beginning of the 90's  \cite{ich93,hjs01}, there is a vast mathematical statistics literature about single-index models, which addresses the problem of  approximating $a$ and possibly also $g$ from a finite number of samples of $f$ to yield an expected least-squares approximation of $f$ on a bounded domain $\Omega \subset \R^d$.  Now assume for the moment that we can evaluate the network $f$  at any point in its domain; we  refer to this setting as \emph{active sampling}. As we aim at uniform approximations, we adhere here to the language of recent results about the sampling complexity of ridge functions from the approximation theory literature, e.g., \cite{Cohen2012,Fornasier2012,maulvyXX}. In those papers, the identification of the neuron is performed by using approximate differentiation. Let us clarify how this method works as it will be of inspiration for the further developments below.
For any $\epsilon>0$, points $x_i$, $i=1,\dots m_{\mathcal X}$, and differentiation directions $\varphi_j$, $j=1,\dots m_{\Phi}$ we have
\begin{equation}\label{eq:1}
\frac{f (x_i + \epsilon \varphi_j)- f (x_i)}{\epsilon}  \approx \frac{ \partial f (x_i)}{\partial \varphi_j} = g'(a^T x_i) a^T\varphi_j.
\end{equation}
Hence, differentiation exposes the weight of a neuron and allows to test it against test vectors $\varphi_j$.
The approximate relationship \eqref{eq:1} forms for every fixed index $i$ a linear system of dimensions $m_\Phi\times d$, whose unknown is $x^*_i=g'(a^T x_i) a$. Solving approximately  and independently the systems for $i=1,\dots m_{\mathcal X}$ yields multiple approximations $\hat a=x^*_i/\|x^*_i\|_2\approx  a$ of the weight, the most stable of them with respect to the approximation error in \eqref{eq:1} is the one for which $\|x^*_i\|_2$ is maximal. Once $\hat a \approx a$ is learned then one can easily construct a function $\hat f(x) = \hat g(\hat a^T x)$ by approximating $\hat g(t) \approx f(\hat a t)$ on further sampling points.
Under assumptions of smoothness of the activation function $g\in C^s([0,1])$, for $s>1$, $g'(0) \neq 0$,  and compressibility of the weight, {\it i.e.}, $\|a\|_{\ell_q^d}$ is small for $0<q \leq 1$, then by using $L$ sampling points  of the function $f$ and the approach sketched above, one can construct a function $\hat f(x) = \hat g(\hat a^T x)$ such that
\begin{equation}\label{eq:2}
	\|f-\hat f\|_{C(\Omega)}\le C\|a\|_{\ell_q^d}\left\{L^{-s}+\|g\|_{C^s([0,1])} \left(\frac{1+\log(d/L)}{L}\right)^{1/q-1/2}\right\}.
\end{equation}
In particular, the result constructs the approximation of the neuron with an error, which has polynomial rate with respect to the number of samples, depending on  the smoothness of the
activation function and the compressibility of the weight vector $a$. The dependence on the  input dimension is only logarithmical. To take advantage of the compressibility of the weight, compressive sensing \cite{fora13} is a key tool to solve the linear systems \eqref{eq:1}.
In \cite{Cohen2012}  such an approximation result was obtained by active and deterministic choice of the input points $x_i$. In order to relax a bit the usage of active sampling, in the paper \cite{Fornasier2012} a random sampling of the points $x_i$ has been proposed and the resulting error estimate would hold with high probability.
The assumption $g'(0) \neq 0$ is somehow crucial, since it was pointed out in \cite{Fornasier2012,maulvyXX} that any level of tractability (polynomial complexity) and intractability (super-polynomial complexity) of the problem may be exhibited otherwise.

\subsubsection{Shallow networks: the one-layer case}
Combining several neurons leads to richer function classes \cite{li02,li92,pe99,pi97,pi99}. A neural network with one hidden layer and one output is simply a weighted sum of neurons whose activation function only differs by a shift, {\it i.e.},
\begin{align}\label{eq:shallow_network}
f(x) = \sum^{m}_{i=1} b_i \phi(a_i^T x + \theta_i) = \sum_{i=1}^m g_i (a_i^Tx),
\end{align}
where $a_i \in \R^m$ and $b_i, \theta_i \in \R$ for all $i = 1, \dots, m$. Sometimes, it may be convenient below the more compact writing $f(x) =1^T g(A^T x)$ where $g=(g_1,\dots, g_m)$ and $A=[a_1|\dots|a_m] \in \mathbb R^{d \times m}$\footnote{Below, with slight abuse of notation, we may use the symbol $A$ also for the span of the weights $\{a_i: i=1,\dots, m\}$.}.
Differently from the case of the single neuron, the use of first order differentiation
\begin{align}\label{eq:firstorder}
\nabla f(x) = \sum^m_{i=1} g_i'(a_i^T x)a_i \in  A = \opspan\left\lbrace a_1, \dots, a_m \right\rbrace,
\end{align}
may furnish information about $A = \opspan\left\lbrace a_1, \dots, a_m \right\rbrace$ (active subspace identification \cite{co15,co14}, see also \cite[Lemma 2.1]{Fornasier2012}), but it does not allow yet to extract information about the single weights $a_i$. For that higher order information is needed. Recent work shows that the identification of a network \eqref{eq:shallow_network} can be related to tensor decompositions \cite{angeja,anandkumar15,fornasier2018identification, MondMont18}. As pointed out in Section \ref{sec:1neuron} differentiation exposes the weights. In fact, one way to relate the network to tensors and tensor decompositions is given by higher order differentiation. In this case the tensor takes the form
\begin{align*}
D^k f(x) = \sum^m_{i=1} g_i^{(k)}(x) \underbrace{a_i \otimes \dots \otimes a_i}_{k-\text{times}},
\end{align*}
which requires that the $g_i$'s are sufficiently smooth. In a setting where the samples are actively chosen, it is generally possible to approximate these derivatives by finite differences. However, even for {\it passive sampling} there are ways to construct similar tensors \cite{anandkumar15,fornasier2018identification}, which rely on Stein's lemma \cite{Stein} or differentiation by parts or weak differentiation. Let us explain how passive sampling in this setting may be used for obtaining tensor representations of the network. If the probability measure of the sampling points $x_i$'s is $\mu_X$ with known (or approximately known \cite{degy85}) density $p(x)$ with respect to the Lebesgue measure, {\it i.e.}, $d \mu_X(x) = p(x) dx$, then we can approximate the expected value of higher order derivatives by using exclusively point evatuations of $f$. This follows from
\begin{align*}
\frac{1}{N} \sum^N_{i = 1} f(x_i)(-1)^k \frac{\nabla^k p(x_i)}{p(x_i)} &\approx \int_{\R^d} f(x)(-1)^k \frac{\nabla^k p(x)}{p(x)} p(x) dx \\
&= \int_{\R^d} \nabla^k f(x) d\mu_X(x) = \bbE_{x\sim \mu_X}[\nabla^k f(x)] \\
&= \sum^m_{i=1} \left (\int_{\R^d} g^{(k)}(a_i^T x) d\mu_X(x) \right ) \underbrace{a_i \otimes \dots \otimes a_i}_{k-\text{times}}.
\end{align*}
In the work \cite{anandkumar15} decompositions of third order symmetric tensors ($k=3$) \cite{angeja,kolda,rob14} have been used for the weights identification of  one hidden layer neural networks. Instead, beyond the classical results about principal Hessian directions \cite{kli92}, in \cite{fornasier2018identification} it is shown that using second derivatives $(k=2)$ actually suffices and the corresponding error estimates reflect positively the lower order and potential of improved stability, see {\it e.g.}, \cite{deli08,hastad,hilim}. The main part of the present work is an extension of the latter approach and therefore we will give a short summary of it with emphasis on active sampling, which will be assumed in this paper as the sampling method. The first step of the approach in \cite{fornasier2018identification} is taking advantage of \eqref{eq:firstorder} to reduce the dimensionality of the problem from $d$ to $m$.
\paragraph{Reduction to the active subspace.}
Before stating the core procedure, we want to introduce a simple and optional method, which can help to reduce
the problem complexity in practice. Assume $f: \Omega \subset \R^d \to \R$ takes the form \eqref{eq:shallow_network}, where $d \geq m$ and that $a_1, \dots, a_m\in \R^d$ are linearly independent.
From a numerical perspective the input dimension  $d$ of the network plays a relevant role in terms of complexity of the procedure. For this reason in \cite{fornasier2018identification} the input dimension is effectively reduced to the number of neurons in the first hidden layer. With this reasoning, in the sections that follow we  also consider networks where the input dimension matches the number of neurons of the first hidden layer.
\\\\
Assume for the moment that the active subspace $A=\opspan\left\lbrace a_1, \dots, a_m \right\rbrace$ is known. Let us choose any orthonormal basis of $A$ and  arrange it as the columns of a matrix $\hat A \in \R^{d \times m}$. Then
\[
f(x) =f(P_A x) = f(\hat A \hat A^T x),
\]
which can be used to define a new network
\begin{align}\label{eq:fhat_shallow}
\hat f(y) := f(\hat A y ) : \R^m \to \R.
\end{align}
whose weights are $\alpha_1 = \hat A^T a_1, \dots, \alpha_m = \hat A^T a_m$, all the other parameters remain unchanged. Note that $\hat A \alpha_i = P_A a_i = a_i$, and therefore $a_i$ can be recovered from $\alpha_i$.
In summary, if the active subspace of $f$ is approximately known, then we can construct $\hat f$, such that the identification of $f$ and $\hat f$ are equivalent. This allows us to reduce the problem to the identification of $\hat f$ instead of $f$, under the condition that we approximate $P_A$ well enough \cite[Theorem 1.1]{fornasier2018identification}.
As recalled in \eqref{eq:firstorder} we can produce easily approximations to vectors in $A$ by approximate first order differentiation of the original network $f$ and, in an ideal setting, generating $m$ linear independent gradients would suffices to approximate $A$. However, in general, there is no way  to ensure a priori such linear independence and we have to account for the error caused by approximating  gradients by finite differences. By suitable assumptions on $f$ (see the full rank condition on the matrix $J[f]$ defined in \eqref{def:Jf} below) and using Algorithm \ref{alg:active_subspace} we obtain the following approximation result.
\begin{algorithm}[t]
	\KwIn{Given a shallow neural network $f$ as in \eqref{eq:shallow_network}, step-size of finite differences $\epsilon>0$, number of samples $m_X$}
	\Begin{
		Draw $x_1, \dots, x_{m_X}$ uniformly from the unit sphere $\bbS^{d-1}$\\
		Calculate the estimated gradients $\Delta_\epsilon f(x_1), \dots, \Delta_\epsilon f(x_{m_X})$ by first order finite differences with stepsize $\epsilon$.\\
		Compute the singular value decomposition
		\begin{equation*}
		\left( \Delta_\epsilon f(x_1) \middle| \dots \middle| \Delta_\epsilon f(x_{m_X}) \right) = \left (\begin{array}{lll}
		\hat U_{1}&\hat U_{2}\end{array}\right )
		\left (\begin{array}{ll}\hat \Sigma_{1}& 0\\
		0& \hat \Sigma_{2 }\\\end{array}\right )
		\left (\begin{array}{l}\hat V_{1 }^T\\\hat V_{2}^T\end{array}
		\right ),
		\end{equation*}
		where $\hat \Sigma_{1}$ contains the $m$ largest singular values.
		Set $P_{\hat A} = \hat U_1 \hat U_1^T$.
	}
	\KwOut{$P_{\hat A}$}
	\caption{Active subspace identification \cite{fornasier2018identification}}\label{alg:active_subspace}
\end{algorithm}
\begin{thm}[\cite{fornasier2018identification},Theorem 2.2]\label{thm:reduction_shallow_case}
	 Assume the vectors $(a_i)_{i=1}^m$ are linear independent and of unit norm. Additionally, assume that the $g_i$'s are smooth enough.
	Let $P_{\hat A}$ be constructed as described in Algorithm \ref{alg:active_subspace}  by sampling $m_{X} (d+1)$ values of $f$. Let $0<s<1$, and assume that the matrix
	\begin{eqnarray}
 J[f]&:=& \bbE_{X \sim \mu_{\mathbb S^{d-1}}} \nabla f(X) \otimes \nabla f(X) \nonumber \\
 &=& \int_{{\mathbb S}^{d-1}}\nabla f(x) \nabla f(x)^T d \mu_{{\mathbb S}^{d-1}}(x) \label{def:Jf}
	\end{eqnarray}
	has  full rank, i.e., its $m$-th singular value  fulfills	$\sigma_m\left (J[f]\right )\ge \alpha>0$.
	Then
	$$
	\|P_A-P_{\hat A}\|_F \le \frac{2C_1\epsilon m}{\sqrt{\alpha(1-s)}-C_1\epsilon m}, 
	$$
	with probability at least $1 - m \exp\Bigl(-\frac{m_{X}\alpha s^2  }{2 m^2 C_2^2}\Bigr)$,
	where $C_1, C_2 > 0$ are absolute constants depending on the smoothness of $g_i$'s.
\end{thm}
\paragraph{Identifying the weights.}
As clarified in the previous section we can assume from now on that  $d=m$ without loss of generality.
Let $f$ be a network of the type \eqref{eq:shallow_network}, with three times differentiable activation functions  $(g_i)_{i=1,\dots, m}$, and independent weights $(a_i)_{i=1,\dots m} \in \R^m$ of unit norm. Then $f$ has  second derivative
\begin{align}\label{eq:decom_hessian_shallow}
\nabla^2 f(x) = \sum^m_{i=1} g_i''(a_i^T x) a_i \otimes a_i \in  \CA =\opspan\left\lbrace a_1 \otimes a_1, \dots, a_m \otimes a_m\right\rbrace,
\end{align}
whose expression represents a non-orthogonal rank-1 decomposition of the Hessian. The idea is, first of all, to modify the
network by an {\it ad hoc} linear transformation (withening) of the input
\begin{align}
f(W^T x) = \sum_{i=1}^{m} g_i(a_i^T W^T x)
\end{align}
in such a way that $(Wa_i/\|Wa_i\|_2)_{i=1,\dots, m}$ forms an orthonormal system.
The computation of $W$ can be performed by spectral decomposition of any positive definite matrix
$$G \in \hat \CA\approx \CA, \gamma I \preccurlyeq G.$$
In fact, from the spectral decomposition of $G = UDU^T$, we define $W = D^{-\frac{1}{2}}U^T$ (see \cite[Theorem 3.7]{fornasier2018identification}). This procedure is called {\it whitening} and allows to reduce the problem to networks with nearly-orthogonal weights, and presupposes to have obtained $ \hat \CA \approx \CA=\opspan\left\lbrace a_1 \otimes a_1, \dots, a_m \otimes a_m\right\rbrace$. By using \eqref{eq:decom_hessian_shallow} and a similar approach as Algorithm \ref{alg:active_subspace} (one simply substitutes there the approximate gradients with vectorized approximate Hessians), one can compute $\hat \CA$ under the assumption that also the second order matrix
\begin{eqnarray*}
H[f] &:=& \bbE_{X \sim \mu_{\mathbb S^{m-1}}}  \operatorname{vec}(\nabla^2 f (X)) \otimes   \operatorname{vec}(\nabla^2 f(X))\\
& =&  \int_{\mathbb S^{m-1}} \operatorname{vec}(\nabla^2 f (x)) \otimes   \operatorname{vec}(\nabla^2 f(x)) d \mu_{\mathbb S^{m-1}}(x)
\end{eqnarray*}
is of full rank, where $\operatorname{vec}(\nabla^2 f (x))$ is the vectorization of the Hessian $\nabla^2 f (x)$.

After whitening  one could assume without loss of generality that the vectors $(a_i)_{i=1,\dots m} \in \R^m$ are nearly orthonormal in the first place.  Hence the representation \eqref{eq:decom_hessian_shallow} would be a near spectral decomposition of the Hessian and the components $a_i \otimes a_i$ would represent the approximate eigenvectors. However, the numerical stability of spectral decompositions is ensured only under spectral gaps \cite{rellich1969perturbation, bhatia2013matrix}.
In order to maximally stabilize the approximation of the $a_i$'s,  one seeks for matrices $M \in \hat \CA$ with the maximal spectral gap between the first and second largest eigenvalues. This is achieved by the maximizers of the following nonconvex program
\begin{align}
\label{eq:nonlinear_programInt}
M = \argmax \N{M}\quad \textrm{s.t.}\quad M \in \hat \CA,\quad \N{M}_F \leq 1,
\end{align}
where $\|\cdot\|$ and $\|\cdot\|_F$ are the spectral and Frobenius norms respectively. This program can be solved by a suitable projected gradient ascent, see for instance \cite[Algorithm 3.4]{fornasier2018identification} and Algorithm \ref{alg:approximation_neural_network_profiles} below, and any resulting maximizer has the eigenvector associated to the largest eigenvalue in absolute value close to one of the $a_i$'s. Once approximations $\hat a_i$ to all the $a_i$'s are retrieved, then it is not difficult to perform the identification of the activation functions $g_i$, see \cite[Algorithm 4.1, Theorem 4.1]{fornasier2018identification}.
The recovery of the network resulting from this algorithmic pipeline is summarized by the following statement.

\begin{thm}[\cite{fornasier2018identification},Theorem 1.2]Let 
$f$ be a real-valued function defined on the neighborhood of $\Omega=B_1^d$, which takes the form
$$
f(x) = \sum_{i=1}^m g_i(a_i \cdot x),
$$
for $m\leq d$. Let $g_i$ be three times continuously differentiable on a neighborhood of $[-1,1]$ for all $i=1,\dots,m$,
and let $\{a_1,\dots,a_m\}$ be linearly independent. We additionally assume both $J[f]$ and $H[f]$ of maximal rank $m$.
Then, for all $\epsilon>0$ (stepsize employed in the computation of finite differences), using at most $m_{\mathcal X} [(d+1)+ (m+1)(m+2)/2]$ random exact point evaluations of $f$,
the nonconvex program \eqref{eq:nonlinear_programInt} constructs approximations $\{\hat a_1,\dots,\hat a_m\}$ of the weights $\{a_1,\dots,a_m\}$
up to a sign change for which
\begin{equation}\label{mainres1}
\bigg ( \sum_{i=1}^m \|\hat a_i-a_i\|_2^2 \bigg )^{1/2} \lesssim \varepsilon,
\end{equation}
with probability at least $1 - m \exp\Bigl(-\frac{m_{\mathcal X}c  }{2 \max\{C_1,C_2\}^2 m^2}\Bigr)$,
for a suitable constant $c>0$ intervening (together with some fixed power of $m$) in the asymptotical constant of the approximation \eqref{mainres1}.
Moreover, once the weights are retrieved one constructs an approximating function $\hat f:B_1^d \to \mathbb R$  of the form
$$
\hat f(x) = \sum_{i=1}^m \hat g_i(\hat a_i \cdot x),
$$
such that
\begin{equation}\label{mainres2}
\| f- \hat f\|_{C(\Omega)} \lesssim \epsilon.
\end{equation}
\end{thm}
While this result have been generalized to the case of passive sampling in \cite{fornasier2018identification} and through whitening allows for the identification of non-orthogonal weights, it is restricted to the  case of $m \leq d$ and linearly independent weights $\{a_i: i=1,\dots,m\}$.

The main goal of this paper is generalizing  this approach to account for both the identification of two fully nonlinear hidden layer neural networks and the case where $m > d$ and the weights are not necessarily nearly orthogonal or even linearly independent (see Remark \ref{rem:redcase} below).
\subsubsection{Deeper networks: the two layer case}
What follows further extends the theory discussed in the previous sections to a wider class of functions, namely neural networks with two hidden layers. By doing so, we will also address a  relevant open problem that was stated in \cite{fornasier2018identification}, which deals with the identification of shallow neural networks where the number of neurons is larger than the input dimension. First, we need a precise definition of the architecture of the neural networks we intend to consider.
\begin{dfn}\label{def:f}
	Let $0 < \Nouterlayer \leq \Ninnerlayer \leq d$, and $\{a_1,\ldots,a_{\Ninnerlayer}\}\subset \bbR^{d}$,
	$\{b_1,\ldots,b_{\Nouterlayer}\}\subset \bbR^{\Ninnerlayer}$ be sets of unit vectors, and denote
	$A := [a_1|\ldots|a_{\Ninnerlayer}] \in \bbR^{d\times \Ninnerlayer}$, $B := [b_1|\ldots|b_{\Nouterlayer}] \in \bbR^{\Ninnerlayer \times \Nouterlayer}$.
	Let $g_1,\ldots,g_{\Ninnerlayer}$ and $h_1,\ldots,h_{\Nouterlayer}$ be univariate functions,
	and denote $G_0 = \diag \left ( g_1'(0),\ldots,g_{\Ninnerlayer}'(0)\right )$.
	We define
	\begin{align}
	\label{eq:deff}
	\CF(d,\Ninnerlayer, \Nouterlayer) := \left\{ f : \mathbb{R}^d \to \mathbb{R} : f(x) =  \sum_{\ell=1}^{\Nouterlayer} h_\ell\left(\sum_{i=1}^{\Ninnerlayer} b_{i\ell}g_i \left(a_i^T x\right) \right)\right\},
	\end{align}
	with $\{a_1,\ldots,a_{\Ninnerlayer}\}\subset \bbR^{d}$, $\{b_1,\ldots,b_{\Nouterlayer}\}\subset \bbR^{\Ninnerlayer}$, $g_1,\ldots,g_{\Ninnerlayer}$
	and $h_1,\ldots,h_{\Nouterlayer}$ satisfying
	\begin{enumerate}[label=({A\arabic*})]
		\item\label{prop:P1} $g_i'(0) \neq 0 \quad \forall i=1,\ldots,\Ninnerlayer$,
		\item\label{prop:P2} a frame condition for the system $\{a_1,\ldots,a_{\Ninnerlayer}, v_1,\ldots,v_{\Nouterlayer}\}$ with $v_\ell := \frac{AG_0 b_\ell}{\N{AG_0 b_\ell}}$, \emph{i.e.} there exist constants $c_f,C_F>0$ such that
		\begin{equation}
		\label{eq:frame_properties_A_B}
		\Lframe\N{x}^2 \leq \sum\limits_{i=1}^{\Ninnerlayer} \left\langle x, a_i\right\rangle^2 + \sum\limits_{\ell=1}^{\Nouterlayer} \left\langle x, v_\ell\right\rangle^2 \leq \Uframe \N{x}^2,
		\end{equation}
		for all $x \in \mathbb R^d$,
		\item\label{prop:P3} the derivatives of $g_i$ and $h_\ell$ are uniformly bounded according to
		\begin{equation}
		\quad \max\limits_{i=1,\ldots,\Ninnerlayer}\sup\limits_{t \in \bbR}\SN{g_i^{(k)}(t)} \leq \kappa_{k}, \quad \textrm{ and }\quad \max\limits_{i=1,\ldots,\Nouterlayer}\sup\limits_{t \in \bbR}\SN{h_\ell^{(k)}(t)} \leq \eta_{k},\quad k=0,1,2,3.
		\end{equation}
	\end{enumerate}
\end{dfn}
Sometimes it may be convenient below the more compact writing $f(x) =1^T h(B^Tg(A^T x))$ where $g=(g_1,\dots, g_{m_0})$, $h=(h_1,\dots, h_{m_1})$.
In the previous section we presented a dimension reduction that can be applied to one layer neural networks, and which can be useful to reduce the dimensionality from the input dimension to the number of neurons of the first layer. The same approach can be applied to networks defined by the class $\CF(d,\Ninnerlayer, \Nouterlayer)$. For the approximation error of the active subspace, we end up with the following corollary of Theorem \ref{thm:reduction_shallow_case}.
\begin{cor}[cf. Theorem \ref{thm:reduction_shallow_case}]\label{thm:reduction_2l_case}
	Assume that $f \in \CF(d,\Ninnerlayer, \Nouterlayer)$ and let $P_{\hat A}$ be constructed as described in Algorithm \ref{alg:active_subspace} by sampling $m_X(d+1)$ values of $f$. Let $0<s<1$, and assume that the $\Ninnerlayer$-th singular value of $J[f]$ fulfills $ \sigma_{\Ninnerlayer}\left ( J[f]\right )\ge \alpha>0$.
	Then  we have
	\begin{align*}
	\norm{P_A - P_{\hat{A}}}_F \leq \frac{2 C_3 \epsilon   \Ninnerlayer\Nouterlayer}{\sqrt{(1-s)\alpha} - C_3 \epsilon   \Ninnerlayer\Nouterlayer},
	\end{align*}
	with probability at least $1-\Ninnerlayer\exp(-\frac{s^2 m_x \alpha }{2 C_4 \Nouterlayer})$ and constants $C_3, C_4 > 0$ that depend only on $\kappa_j, \eta_j$ for $j=0,\dots,3$.
\end{cor}
 From now on we assume $d=m_0$.


\section{Approximating the span of tensors of weights}\label{sec:approx_matrix_space}
In the one layer case, which was described earlier, the unique identification of the weights is made possible by constructing a matrix space whose rank-1 basis elements are outer products of the weight profiles of the network. This section illustrates the extension of this approach beyond shallow neural networks. Once again, we will make use of differentiation and overall there will be many parallels to the approach in \cite{fornasier2018identification}. However, the intuition behind the matrix space will be less straightforward, because we can not anymore directly express the second derivative of a two layer network as a linear combination of  symmetric rank-1  matrices. This is due to the fact that the Hessian matrix of a network $f \in \CF(\Ninnerlayer, \Ninnerlayer, \Nouterlayer)$ has the form
\begin{align*}
\nabla^2 f(\ve{x}) = &\sum^{\Nouterlayer}_{\ell=1}h_\ell'(\ve{b}_\ell^T g(\ma{A}^T \ve{x}))\sum_{i=1}^{\Ninnerlayer} b_{i\ell}g_i''(\ve{a}_i^T \ve{x}) \ve{a}_i\otimes \ve{a}_i\\
+& \sum^{\Nouterlayer}_{\ell = 1}\sum^{\Ninnerlayer}_{i,j=1} h_\ell''\left(\ve{b}_\ell^Tg(\ma{A}^T \ve{x})\right)b_{i\ell}b_{j\ell}g_i'(\ve{a}_i^T \ve{x})g_j'(\ve{a}_j^T \ve{x}) (\ve{a}_i\otimes \ve{a}_j + \ve{a}_j\otimes \ve{a}_i)
\end{align*}
Therefore, $\nabla^2 f(\ve{x}) \in \opspan \set{\ve{a}_i\otimes \ve{a}_j + \ve{a}_j\otimes \ve{a}_i \given i,j = 1,\dots,\Ninnerlayer}$, which has dimension $\frac{\Ninnerlayer(\Ninnerlayer+1)}{2}$ and is in general not spanned by symmetric rank-1  matrices.
This expression is indeed quite complicated, due to the chain rule and the mixed tensor contributions, which are consequently appearing. At a first look, it would seem impossible to use a similar approach as the one for shallow neural networks recalled in the previous section.
Nevertheless a relatively simple algebraic manipulation allows to recognize some useful structure: For a fixed $x \in \R^{\Ninnerlayer}$ we rearrange the expression as
\begin{align*}
\nabla^2 f(\ve{x}) &= \sum^{\Nouterlayer}_{\ell=1}h_\ell'(\ve{b}_\ell^T g(\ma{A}^T \ve{x}))\sum_{i=1}^{\Ninnerlayer} b_{i\ell}g_i''(\ve{a}_i^T \ve{x}) \ve{a}_i\otimes \ve{a}_i\\
&+\sum^{\Nouterlayer}_{\ell = 1}  h_\ell''\left(\ve{b}_\ell^Tg(\ma{A}^T \ve{x})\right)\left[\sum^{\Ninnerlayer}_{i=1}b_{i\ell}g_i'(\ve{a}_i^T \ve{x})\ve{a}_i\right]\otimes\left[ \sum^{\Ninnerlayer}_{j=1} b_{j\ell}g_j'(\ve{a}_j^T \ve{x})\ve{a}_j\right],
\end{align*}
which is a combination  of symmetric rank-1 matrices since $\sum^{\Ninnerlayer}_{j=1} b_{j\ell}g_j'(\ve{a}_j^T \ve{x})\ve{a}_j \in \R^{\Ninnerlayer}$. We  write the latter expression more compactly by introducing the notation
\begin{align}\label{eq:hessian_decomp}
\nabla^2 f(x) = \sum_{i=1}^{\Ninnerlayer} \gamma_{i}(x) a_i\otimes a_i + \sum^{\Nouterlayer}_{\ell = 1} \tau_\ell(x) v_{\ell}(x) \otimes v_{\ell}(x),
\end{align}
where $G_x = \operatorname{diag}\left(g_1'(a_1^T x), \dots, g_{\Ninnerlayer}'(a_{\Ninnerlayer}^T x)\right) \in \R^{\Ninnerlayer \times \Ninnerlayer}$ and
\begin{align}
v_{\ell}(x) &= A G_x b_\ell \in \R^{\Ninnerlayer} &\text{for }\ell \in [\Nouterlayer],\label{def:smallvx}\\
\gamma_i(x) &= g_i''(\ve{a}_i^T \ve{x})\sum^{\Nouterlayer}_{\ell=1}h_\ell'(\ve{b}_\ell^T g(\ma{A}^T \ve{x})) b_{i\ell} \in \R &\text{ for } i \in [\Ninnerlayer], \label{def:gamma}\\
\tau_\ell(x) &= h_\ell''\left(\ve{b}_\ell^Tg(\ma{A}^T \ve{x})\right) \in \R &\text{ for } \ell \in [\Nouterlayer].\label{def:tau}
\end{align}
\begin{figure}\label{fig:geometryCW}
	\centering
	\begin{tikzpicture}[scale=2]
	\filldraw[gray!14] (0,0) -- (200:2.13) -- (-2,-2) -- (250:2.13) -- (0,0);
	\filldraw[gray!14] (0,0) -- (20:2.13) -- (2,2) -- (70:2.13) -- (0,0);
	\draw[thick, dashed] (0,0) -- (38:2.55);

	\draw[thick, color=blue!30, fill = blue!8]
	(1,1) .. controls (0.8, 0.7) and (0.6,0.5) .. (-1.5,-2) --
	(-2,-2) -- (-2,-1) --
	(-2,-1) .. controls (-1,-0.8) and (0.8, 0.9)  .. (1,1)--
	(1,1) .. controls (1.2, 1.25) and (1.2,1.5) .. (1.3,2)--
	(2,2)--(2,1.6)--
	(2,1.6) .. controls (1.6,1.1) and (1.3, 1.1) .. (1,1);
	\draw[step=0.2cm, gray!20, very thin] (-1.9,-1.9) grid (1.9,1.9);
	\draw[<->] (0,-2) -- (0,2);
	\draw[<->] (-2,0) -- (2, 0);
	\filldraw[gray] (1,1) circle (1pt) node[color=black,above left]{$\nabla^2 f(0)$} ;
	\draw[very thick] (-2,-2) -- node[above left,pos=0.9]{$\CW$}(2,2);
	\draw[thick, dashed] (0,0) -- node[above,left]{$\hat \CW$} (218:2.55);
	\draw[thick, dashed] (0,0) -- (38:2.55);
	\end{tikzpicture}
	\caption{Illustration of the relationship between $\CW$ (black line) and $\opspan\left\lbrace \nabla^2 f(x) \middle| x \in \R^{\Ninnerlayer} \right\rbrace$ (light blue region) given by two non-linear cones that fan out from $\nabla^2 f(0)$. There is no reason to believe that the these cones are symmetric around $\CW$. The gray cones show the maximal deviation of $\hat \CW$ from $\CW$.}
	\label{fig:two_layer_subspace}
\end{figure}
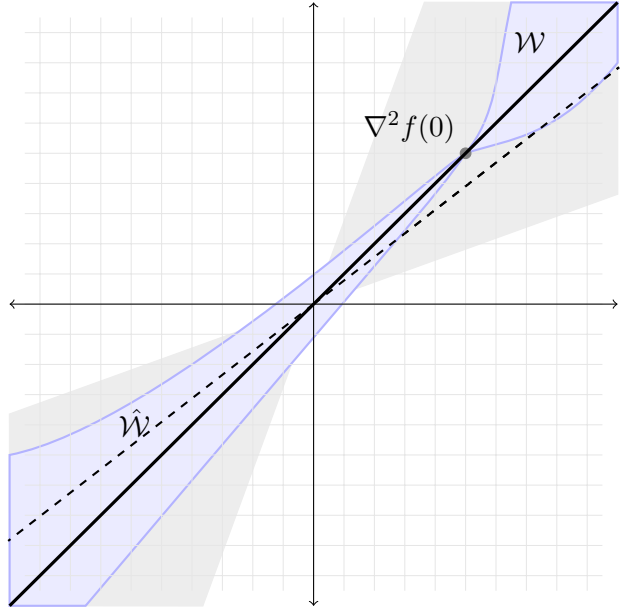
Let us now introduce the
fundamental matrix space
\begin{align}\label{def:CW}
\CW = \CW(f) :=  \opspan \left\lbrace a_1 \otimes a_1, \dots, a_{\Ninnerlayer} \otimes a_{\Ninnerlayer}, v_1 \otimes v_1, \dots, v_{\Nouterlayer} \otimes  v_{\Nouterlayer} \right\rbrace,
\end{align}
where $a_1, \dots, a_{\Ninnerlayer}$ are the weight profiles of the first layer and $$v_\ell := v_\ell(0) / \N{v_\ell(0)}_2 = A G_0 b_\ell/\| A G_0 b_\ell\|_2,$$ for all $\ell =1, \dots, \Nouterlayer$ encode entangled information about $b_1, \dots, b_{\Nouterlayer}$. For this reason, we call the $v_\ell$'s entangled weights.
Let us stress at this point that the definition and the constructive approximtion of the space $\CW$ is perhaps the most crucial and relevant contribution of this paper.
In fact, by inspecting carefully the expression  \eqref{eq:hessian_decomp}, we immediately notice that $\nabla^2 f(0) \in \CW$, and also that the first sum in \eqref{eq:hessian_decomp}, namely $\sum_{i=1}^{\Ninnerlayer} \beta_i(x) a_i\otimes a_i$,  lies in $\CW$ for all $x \in \R^{\Ninnerlayer}$. Moreover, for arbitrary sampling points $x$, deviations of $\nabla^2 f(x)$ from $\CW$ are only due to the second term in \eqref{eq:hessian_decomp}. The intuition is that for suitable centered distributions of sampling points $x_i$'s so that $a_j^T x_i \approx 0$ and $G_{x_i} \approx G_0$,  the Hessians $(\nabla^2 f(x_i))_{i=1,\dots,m_X}$ will distribute themselves somehow around the space $\CW$, see Figure \ref{fig:two_layer_subspace} for a two dimensional sketch of the geometrical situation. Hence, we would attempt an approximation of $\CW$ by PCA of a collection of such approximate Hessians.
Practically,  by active sampling (targeted evaluations of the network $f$) we first construct estimates $(\Delta_\epsilon^2 f(x_i))_{i=1,\dots,m_X}$ by finite differences of the Hessian matrices $(\nabla^2 f(x_i))_{i=1,\dots,m_X}$ (see Section \ref{subsection:finite_difference}), at sampling points $x_1, \dots, x_{m_X} \in \R^m$  drawn independently from a suitable distribution $\mu_X$. Next, we  define the matrix
$$\hat W = \left(\opvec(\Delta_\epsilon^2 f(x_1)), \dots, \opvec(\Delta_\epsilon^2 f(x_{m_X}))\right),$$
whose columns are the vectorization of the approximate Hessians.
Finally, we produce the approximation  $\hat\CW$  to $\CW$  as the span of the first $\Ninnerlayer+\Nouterlayer$ left singular vectors of the  matrix $\hat W$.
The whole procedure of calculating $\hat \CW$ is given in Algorithm \ref{alg:approximation_matrix_space}. It should be clear that the choice of $\mu_X$ plays a crucial role for the quality of this method.
In the analysis that follows,  we  focus on distributions that are centered and concentrated.
Figure \ref{fig:two_layer_subspace} helps to form a better geometrical intuition of the result of the procedure. It shows the region covered by the Hessians, indicated by the light blue area, which envelopes the space $\CW$ in a sort of nonlinear/nonconvex cone originating from $\nabla^2 f(0)$. In general, the Hessians do not concentrate around $\CW$ in a symmetric way, which means that the ``center of mass'' of the Hessians can never be perfectly aligned with the space $\CW$, regardless of the number of samples. In this analogy, the center of mass is equivalent to the space estimated by Algorithm \ref{alg:approximation_matrix_space}, which essentially is a non-centered principal component analysis of observed Hessian matrices.
The primary result of this section is Theorem \ref{thm:approx_space_bound}, which provides an estimate of the approximation error of Algorithm \ref{alg:approximation_matrix_space} depending on the subgaussian norm of the sample distribution $\mu_X$ and the number of neurons in the respective layers.
More precisely, this result gives a precise worst case estimate of the error caused by the imbalance of mass. For reasons mentioned above, the error does not necessarily vanish with an increasing number of samples, but the probability under which the statement holds will tend to $1$. In Figure \ref{fig:two_layer_subspace}, the estimated region is illustrated by the gray cones that show the maximal, worst case deviation of $\hat \CW$.
One crucial condition for Theorem \ref{thm:approx_space_bound} to hold is that there exists an $\alpha>0$ such that
\begin{align}
\sigma_{\Ninnerlayer + \Nouterlayer}\left(\bbE_{X \sim \mu_X}\opvec(\nabla^2 f(X))\otimes \opvec(\nabla^2 f(X))\right) \geq \alpha.
\end{align}
This assumption makes sure that the space spanned by the observed Hessians has, in expectation, at least dimension $\Ninnerlayer + \Nouterlayer$. Aside from this technical aspect this condition implicitly helps to avoid network configurations, which are reducible, for certain weights can not be recovered. For example, we can define a network in $\CF(2,2,1)$ with weights given by
\[
a_1 = \begin{pmatrix}
\frac{1}{\sqrt{2}} \\
\frac{1}{\sqrt{2}}
\end{pmatrix},\,
a_2 = \begin{pmatrix}
\frac{1}{\sqrt{2}} \\
-\frac{1}{\sqrt{2}}
\end{pmatrix},
b_1 = \begin{pmatrix}
1 \\
0
\end{pmatrix}.
\]
It is easy to see that $a_2$ will never be used during a forward pass through the network, which makes it impossible to recover $a_2$ from the output of the network.\\
In the theorem below and in the proofs that follow we will make use of the subgaussian norm $\N{\cdot}_{\psi_2}$ of a random variable. This quantity measures how fast the tails of a distribution decay and such a decay plays an important role in several concentration inequalities. {More in general, for $p\geq 1$, the $\psi_p$-norm of a scalar random variable $Z$ is defined as
\begin{align*}
	\N{Z}_{\psi_p} = \inf \left\lbrace t>0: \bbE \exp(|Z / t|^p )\leq 2\right\rbrace.
\end{align*}
For a random vector $X$ on $\bbR^d$ the $\psi_p$-norm is given by
\begin{align*}
 \N{X}_{\psi_p} = \sup_{x \in \mathbb{S}^{d-1}} \N{\abs{\langle X, x \rangle}}_{\psi_p}.
\end{align*}
The random variables for which $\N{X}_{\psi_1} <\infty$ are called subexponential and those for which  $\N{X}_{\psi_2} <\infty$ are called subgaussian. More in general, the Orlicz space $L_{\psi_p}= L_{\psi_p}(\Omega, \Sigma, \mathbb P)$ consists of all real random variables $X$ on the probabillity space $(\Omega, \Sigma, \mathbb P)$ with finite $\N{X}_{\psi_p}$ norm and its elements are called $p$-subexponential random variagles.
Below, we  mainly focus on subgaussian random variables. In particular, every bounded random variable is subgaussian, which covers all the cases we  discuss in this work. We refer to \cite{HDP18} for more details.
One example of a subgaussian distribution is the uniform distribution on the unit-sphere, which has subgaussian norm $\N{X}_{\psi_2} = \frac{1}{\sqrt{d}}, X \sim \Unif(\bbS^{d-1})$.}

\begin{algorithm}[t]
	\KwIn{Neural network $f$, number of estimated Hessians $m_X$, step-size of the finite difference approximation $\epsilon > 0$, probability distribution $\mu_X$}
	\Begin{
		Draw $x_1, \dots x_{m_X}$ independently from $\mu_X$\\
		Calculate the matrix $M=\left( \opvec (\Delta_\epsilon^2 f(x_1))| \dots | \opvec(\Delta_\epsilon^2 f(x_{m_X}))\right)$\\
		Set $U\Sigma V^T = \operatorname{SVD}(M)$\\
		Denote by $U_1$ the first $\Ninnerlayer + \Nouterlayer$ columns of $U$.
		$P_{\hat \CW}= (u_1| \dots | u_{\Ninnerlayer+\Nouterlayer})(u_1| \dots| u_{\Ninnerlayer+\Nouterlayer})^T$\\
	}
	\KwOut{$P_{\hat \CW}$}
	\caption{Approximating $\CW$}
	\label{alg:approximation_matrix_space}
\end{algorithm}
\begin{thm}\label{thm:approx_space_bound}
	Let $f \in \CF(\Ninnerlayer, \Ninnerlayer, \Nouterlayer)$ be a neural network within the class described in Definition \ref{def:f} and consider the space $\CW$ as defined in \eqref{def:CW}. Assume that $\mu_X$ is a probability measure with $\supp(\mu_X) \subset B^{\Ninnerlayer}_1$, $\bbE X = 0$, and that there exists an $\alpha > 0$ such that
	\begin{align}\label{eq:conditionsthmrepeat}
	\sigma_{\Ninnerlayer + \Nouterlayer}\left(\bbE_{X \sim \mu_X} \opvec(\nabla^2 f(X)) \otimes \opvec(\nabla^2 f(X)) \right) \geq \alpha.
	\end{align}
	Then, for any $\epsilon >0$, Algorithm \ref{alg:approximation_matrix_space} returns a projection $P_{\hat \CW}$ that fulfills
	\begin{align}\label{errorbound}
	\N{P_{\CW^*} -P_{\hat \CW}}_F \leq \frac{\left( C_{\Delta}\epsilon \Nouterlayer \Ninnerlayer^{\frac{3}{2}} + C\N{A}^2 \N{B}^2 \N{X}_{\psi_2}\sqrt{\Nouterlayer\log(\Ninnerlayer+1)}\right)}{\sqrt{\frac{\alpha}{2}} - C_{\Delta}\epsilon \Nouterlayer \Ninnerlayer^{\frac{3}{2}}},
	\end{align}
	for a suitable subspace $\CW^* \subset \CW$ (we can actually assume that $\CW^* = \CW$ according to Remark \ref{remmaz} below) with probability at least
	\begin{align*}
	1 - 2 e^{-c_1 L_1 m_X \min\left\lbrace L_1 m_X, 1\right\rbrace} - (\Ninnerlayer + \Nouterlayer) e^{-L_2 m_X}
	\end{align*}
	where
	\begin{align*}
	L_1 &:= \N{X}^2_{\psi_2}\log(\Ninnerlayer + 1)\Nouterlayer,
	L_2 := \alpha (8 C_1 \N{A}^4\N{B}^4 \Nouterlayer)^{-1},
	\end{align*}
	$c_1,c_2$ are absolute constants and $C, C_1, C_{\Delta} >0 $ are constants depending on the  constants $\kappa_j, \eta_j$ for $j=0,\dots,3$.
\end{thm}
\begin{rem}\label{remmaz}
	If $\epsilon>0$ is sufficiently small, due to \eqref{eq:conditionsthmrepeat} the space $\hat \CW$ returned by Algorithm \ref{alg:approximation_matrix_space} has dimension $\Ninnerlayer + \Nouterlayer$.
	If the error bound \eqref{errorbound} in Theorem \ref{thm:approx_space_bound} is such that $\N{P_{\CW^*} -P_{\hat \CW}}_F < 1$, then $\hat \CW$ and $\CW^*$ must have the same dimension. Moreover, $\CW^* \subset \CW$ and $\operatorname{dim}(\CW) = \Ninnerlayer + \Nouterlayer$  would necessarily imply that $\CW = \CW^*$. Hence, for   $\N{P_{\CW^*} -P_{\hat \CW}}_F < 1$ and $\epsilon>0$ sufficiently small, we have $\CW^* = \CW$.
\end{rem}
As already mentioned above, for $\mu_X= \Unif(\bbS^{\Ninnerlayer-1})$ we have $\N{X}_{\psi_2} = \frac{1}{\sqrt{\Ninnerlayer}}$. In this case the error bound \ref{errorbound} behaves like
$$
\N{P_{\CW} -P_{\hat \CW}}_F  \leq \mathcal O \left (\epsilon \Nouterlayer \Ninnerlayer^{\frac{3}{2}}   +  \sqrt{\frac{\Nouterlayer}{\Ninnerlayer}\log(\Ninnerlayer+1)}\right),
$$
which is small for $\epsilon>0$ small and $\Ninnerlayer \gg \Nouterlayer$. The latter condition seems favoring networks, for which the inner layer has a significantly larger number of neurons than the outer layer. This expectation is actually observed numerically, see Section \ref{sec:numerical_NNprofiles}. We have to add, though, that the parameter $\alpha>0$ that intervenes in the error bound \eqref{errorbound} might also depend on $\Ninnerlayer, \Nouterlayer$ (as it is in fact an estimate of an $(\Ninnerlayer+ \Nouterlayer)^{th}$ singular value as in \eqref{eq:conditionsthmrepeat}). Hence, the dependency on the network dimensions is likely more complex and depends on the interplay between the input distribution $\mu_X$ and the network architecture.
In fact,  at least judging from our numerical experiments, the error bound \eqref{errorbound} is rather pessimistic and it certainly describes a worst case analysis. One more reason might be that some crucial estimates in its proof could be significantly improved. Another reason could be the rather great generality of the activation functions of the networks,  which we analyze in this paper, as described in Definition \ref{def:f}. Perhaps the specific instances used in the numerical experiments are enjoying better identification properties.
\subsection{Estimating Hessians of the network by finite differences}\label{subsection:finite_difference}
Before addressing  the proof of Theorem \ref{thm:approx_space_bound}, we  give a precise definition of the finite differences we are using to approximate the Hessian matrices.
Denote by $e_i$ the $i$-th Euclidean canonical basis vector in $\mathbb{R}^d$.
We denote by $\Delta^2 f(x):= \Delta_\epsilon^2 f(x)$ the second order finite difference approximation of $\nabla^2 f(x)$,  given by
\begin{align}\label{def:finite_differences_hessian}
\Delta_\epsilon^2 f(x)_{ij} := \frac{f(\ve{x} + \epsilon \ve{e}_i + \epsilon \ve{e}_j) - f(\ve{x} + \epsilon \ve{e}_i) - f(\ve{x} + \epsilon \ve{e}_j) + f(\ve{x})}{\epsilon^2}
\end{align}
for $i,j = 1,\dots, d=\Ninnerlayer$ and a step-size $\epsilon>0$. When it is not necessary, we will drop the step-size in the notation and simply write $\Delta^2 f(x)$.
\begin{lem}\label{lem:errorhessian}
	Let $f \in \CF(\Ninnerlayer,\Ninnerlayer, \Nouterlayer)$ be a neural network. Further assume that $\Delta_\epsilon^2 f(x)$ is constructed as in \eqref{def:finite_differences_hessian} for some $\epsilon>0$. Then we have
	\begin{align*}
	\sup_{x \in B^d_1} \norm{\nabla^2 f(x) - \Delta^2_{\epsilon}(\ve{x})}_F \leq C_{\Delta} \epsilon \Nouterlayer \Ninnerlayer^{\frac{3}{2}},
	\end{align*}
	where $C_{\Delta}>0$ is a constant depending on the constants $\kappa_j, \eta_j$ for $j=0,\dots,3$.
\end{lem}
For the proof of Lemma \ref{lem:errorhessian} we simply use the Lipschitz continuity of the functions $g,h$ and of their derivatives, and make use of $\N{a}_2, \N{b}_2 \leq 1$. The details can be found in the Appendix (Section \ref{subsec:proofofsec:approx_matrix_space}).
\color{black}
\subsection{Span of tensors of (entangled) network weights: Proof of Theorem \ref{thm:approx_space_bound}}
The proof can essentially be divided into two separate bounds. Both will be addressed separately with the two lemmas below. For both lemmas we will assume that $X_1, \dots, X_{m_X} \sim \mu_X$ independently and that $\supp(\mu_X) \subseteq B^{\Ninnerlayer}_1$. Additionally,
we define the random matrices
\begin{align}
W &:= (\opvec(\nabla^2 f(X_1)) | \dots | \opvec(\nabla^2 f(X_{m_X}))), \label{def:matrixW}\\
\hat W &:= (\opvec(\Delta^2 f(X_1)) | \dots | \opvec(\Delta^2 f(X_{m_X})),\label{def:matrixhatW}\\
W^* &:= (\opvec(P_{\CW} \nabla^2 f(X_1)) | \dots | \opvec(P_{\CW} \nabla^2 f(X_{m_X}))),\label{def:matrixWstar}
\end{align}
where $P_{\CW}$ denotes the orthogonal projection onto $\CW$ (cf. \eqref{def:CW}). For reader's convenience, we recall here from \eqref{eq:hessian_decomp}
that the Hessian matrix of $f\in \CF(\Ninnerlayer, \Ninnerlayer, \Nouterlayer)$ can be expressed as
$$\nabla^2 f(x) = \sum_{i=1}^{\Ninnerlayer} \gamma_{i}(x) a_i\otimes a_i + \sum^{\Nouterlayer}_{\ell = 1} \tau_\ell(x) v_{\ell}(x) \otimes v_{\ell}(x),$$
where $\gamma_i(x), \tau_\ell(x), $ and $v_\ell(x)$ are introduced in \eqref{def:smallvx} - \eqref{def:tau}. We further simplify this expression by introducing the notations
\begin{align}
V_x &= \left(v_1(x) | \dots | v_{\Nouterlayer}(x)\right) =AG_x B\label{def:Bigvx},\\
\Gamma_x &= \operatorname{diag}\left(\gamma_1(x), \dots, \gamma_{\Ninnerlayer}(x)\right) \label{def:Biggamma},\\
T_x &= \operatorname{diag}\left(\tau_1(x), \dots, \tau_{\Nouterlayer}(x)\right)\label{def:Bigtau}.
\end{align}
which allow us to re-write \eqref{eq:hessian_decomp} in terms of matrix multiplications
\begin{align}
\nabla^2 f(x) = A \Gamma_x A^T + V_x T_x V_x^T.
\end{align}
\begin{lem}\label{lem:bound_num_wedin}
	Let $f \in \CF(\Ninnerlayer, \Ninnerlayer, \Nouterlayer)$ and let $\hat W, W^*$ be defined as in \eqref{def:matrixhatW}-\eqref{def:matrixWstar}, where $\mu_X$ is a subgaussian distribution with subgaussian norm $\N{X}_{\psi_2}$. Then the bound
\begin{align*}
\N{\hat W - W^*}_F  \leq \sqrt{m_X}\left( C_{\Delta}\epsilon \Nouterlayer \Ninnerlayer^{\frac{3}{2}} + C\N{A}^2 \N{B}^2 \N{X}_{\psi_2}\sqrt{\Nouterlayer\log(\Ninnerlayer+1)}\right)
\end{align*}
holds with probability at least
\begin{align*}
	1- 2 \exp\left( -c_1 \N{X}^2_{\psi_2}\log(\Ninnerlayer + 1)\Nouterlayer m_X  \min\left(c_2 \N{X}^2_{\psi_2}\log(\Ninnerlayer + 1)\Nouterlayer m_X, 1 \right)\right),
\end{align*}
where $c_1,c_2 > 0$ are absolute constants and $C, C_{\Delta}>0$ depend only on the constants $\kappa_j, \eta_j$ for $j=0,\dots,3$.
\end{lem}

\begin{proof}
	By triangle inequality we get
	\begin{align}\label{eq:triang_bound_num_wedin}
	\N{\hat W - W^*}_F &\leq \N{\hat W - W}_F + \N{ W^*-W}_F.
	\end{align}
	For the first term on the right hand side we can use the worst case estimate from Lemma \ref{lem:errorhessian}, which yields
	\begin{align}\label{eq:pippo1}
	\N{\hat W - W}_F \leq \sqrt{m_X} \sup_{x \in B_1^d}\norm{\Delta^2 f(x) -\nabla^2 f(x)}_F
	\leq  \sqrt{m_X} C_{\Delta}\epsilon \Nouterlayer \Ninnerlayer^{\frac{3}{2}}
	\end{align}
	for some constant $C_{\Delta}>0$. The second term in \eqref{eq:triang_bound_num_wedin} can be bounded by (the explanation of the individual identities and estimates follows immediately below)
	\begin{align*}
	\N{W - W^*}_F^2 &= \left(\sum^{m_X}_{i=1}\N{\opvec(\nabla^2 f(X_i)) - \opvec(P_{\CW} \nabla^2 f(X_i))}_2^2\right) \\
	&= \sum_{i=1}^{m_X} \N{V_{X_i}T_{X_i}V_{X_i}^T - V_0 T_{X_i} V_0^T}^2_F \leq \sum_{i=1}^{m_X} 4\N{(V_{X_i} - V_{0})T_{X_i}V_{X_i}^T}_F^2\\
	&\leq 4 \sum_{i=1}^{m_X} \left(\N{(V_{X_i} - V_{0})} \N{T_{X_i}}_F \N{V_{X_i}}\right)^2 \\
	&\leq 4\N{A}^4 \N{B}^4 \sum^{m_X}_{i=1} \left( \N{G_{X_i}-G_{0}} \N{T_{X_i}}_F \N{G_{X_i}} \right)^2 \\
	&\leq 4\N{A}^4 \N{B}^4  \Nouterlayer \kappa_1^2 \eta_2^2 \sum^{m_X}_{i=1} \left( \N{G_{X_i}-G_{0}} \right)^2 \leq 4 \kappa_1^2 \kappa_2^2 \eta_2^2  \N{A}^4 \N{B}^4  \Nouterlayer \sum^{m_X}_{i=1} \N{A^T X_i}^2_\infty.
	\end{align*}
	In the first two equalities we made use of the fact that $A\Gamma_x A^T \in \CW$ and that by definition of an orthogonal projection $\N{V_{X_i}T_{X_i}V_{X_i}^T - P_{\CW}V_{X_i}T_{X_i}V_{X_i}^T}_F \leq \N{V_{X_i}T_{X_i}V_{X_i}^T - V_0 T_{X_i} V_0^T }_F$. The remaining inequalities follow directly from the submultiplicativity of $\N{\cdot}_F$ and $\N{\cdot}$ combined with the Lipschitz continuity of the activation functions and their derivatives (cf. \eqref{def:f} A3).
	Since $\N{a_j} \leq 1$, we can estimate the sub-exponential norm of $\N{A^T X_i}^2_\infty = \max_{1 \leq j \leq \Ninnerlayer} \langle X_i, a_j \rangle^2$ by
	\begin{align*}
	\N{\max_{1 \leq j \leq \Ninnerlayer} \langle X_i, a_j \rangle^2}_{\psi_1} &\leq c_1 \log(\Ninnerlayer+1) \max_{1 \leq j \leq \Ninnerlayer} \N{\langle X_i, a_j \rangle^2}_{\psi_1} \\
	&= c_1 \log(\Ninnerlayer+1) \max_{1 \leq j \leq \Ninnerlayer} \N{\langle X_i, a_j \rangle}_{\psi_2}^2 \leq c_1 \log(\Ninnerlayer+1) \N{X}_{\psi_2}^2,
	\end{align*}
	for an absolute constant $c_1>0$, where we applied \cite[Lemma 2.2.2]{vanVaart96} in the first inequality and used that
	   $\N{Y}_{\psi_2}^2 = \N{Y^2}_{\psi_1}$ for any scalar random variable $Y$ together with the fact that the subgaussian norm of a vector is defined by $\N{X}_{\psi_2} = \sup_{x\in \bbS^{d-1}} |\langle x, X\rangle|$ (cf. \cite{HDP18}). The random vectors $X_i\sim \mu_X$ are i.i.d., which allows us to drop the dependency on $i$  in the last step. The previous bound also guarantees a bound on the expectation, which is due to  $\bbE[|Y|^p] \leq p! \N{Y}_{\psi_1}$  (cf. \cite{vanVaart96}), namely, for $p=1$ and $Y=\max_{1 \leq j \leq \Ninnerlayer} \langle X, a_j \rangle^2$
	\begin{align}
	\bbE\left[\max_{1 \leq j \leq \Ninnerlayer} \langle X, a_j \rangle^2\right]  \leq \N{\max_{1 \leq j \leq \Ninnerlayer} \langle X_i, a_j \rangle^2}_{\psi_1}  \leq c_1 \log(\Ninnerlayer+1) \N{X}_{\psi_2}^2.
	\end{align}
	Denote $Z_i := \N{A^T X_i}^2_\infty$ for all $i = 1, \dots, m_X$, then
	\begin{align}\label{eq:difftoprojectionassum}
	\N{W - W^*}_F^2 &\leq 4 \kappa_1^2 \kappa_2^2 \eta_2^2 \N{A}^4 \N{B}^4  \Nouterlayer \sum^{m_X}_{i=1} Z_i.
	\end{align}
	Therefore, applying the Bernstein inequality for sub-exponential random variables \cite[Theorem 2.8.1]{HDP18} to the right sum in \eqref{eq:difftoprojectionassum} yields
	\begin{align*}
	\N{W - W^*}_F^2 \leq  4 \kappa_1^2 \kappa_2^2 \eta_2^2 \N{A}^4 \N{B}^4  (c_1 m_X \Nouterlayer\log(\Ninnerlayer+1) \N{X}_{\psi_2}^2 + t),
	\end{align*}
	with probability at least
	\begin{align*}
	1- 2 \exp\left( -c \min\left(\frac{t^2}{\sum_{i=1}^{m_X} \N{Z_i}^2_{\psi_1}}, \frac{t}{\max_{i\leq m_X}\N{Z_i}_{\psi_1}} \right)\right) \\= 1- 2 \exp\left( -c\frac{t}{\N{X}^2_{\psi_2}} \min\left(\frac{t}{m_X \N{X}^2_{\psi_2}}, 1 \right)\right),
	\end{align*}
	for all $t \geq 0$ and an absolute constant $c >0$. Then, by choosing $t = c_1 m_X \Nouterlayer\log(\Ninnerlayer+1) \N{X}_{\psi_2}^2$ and $c_2 = c \cdot c_1$, we get
	\begin{align}\label{eq:pippo2}
	\N{W - W^*}_F^2 \leq  8 \kappa_1^2 \kappa_2^2 \eta_2^2 \N{A}^4 \N{B}^4  c_1 m_X \Nouterlayer\log(\Ninnerlayer +1) \N{X}_{\psi_2}^2
	\end{align}
	with probability at least
	\begin{align}\label{probability_num}
	1- 2 \exp\left( -c_2 \N{X}^2_{\psi_2}\log(\Ninnerlayer + 1)\Nouterlayer m_X  \min\left(c_1 \N{X}^2_{\psi_2}\log(\Ninnerlayer + 1)\Nouterlayer m_X, 1 \right)\right).
	\end{align}
	From \eqref{eq:triang_bound_num_wedin}, combining \eqref{eq:pippo1} and \eqref{eq:pippo2} yields
	\begin{align*}
	\N{\hat W - W^*}_F  &\leq  \sqrt{m_X} \sup_{x \in \bbB^d}\norm{\Delta^2 f(x) -\nabla^2 f(x)}_F + \sqrt{8c_1} \kappa_1 \kappa_2 \eta_2 \N{A}^2 \N{B}^2 \N{X}_{\psi_2}\sqrt{m_X \Nouterlayer\log(\Ninnerlayer+1)}\\
	&\leq  \sqrt{m_X} C_{\Delta} \epsilon \Nouterlayer \Ninnerlayer^{\frac{3}{2}} + \sqrt{8c_1} \kappa_1 \kappa_2 \eta_2 \N{A}^2 \N{B}^2 \N{X}_{\psi_2}\sqrt{m_X \Nouterlayer\log(\Ninnerlayer+1)},
	\end{align*}
	where we used Lemma \ref{lem:errorhessian} in the second inequality, and the results holds at least with the probability given as in \eqref{probability_num}. Setting $C:=\sqrt{8c_1} \kappa_1 \kappa_2 \eta_2 >0$ finishes the proof.
\end{proof}
\begin{lem}\label{lem:bound_denom_wedin}
	Let $X \in \mu_X$ be centered and subgaussian. Furthermore, assume that $f \in \CF(\Ninnerlayer, \Ninnerlayer, \Nouterlayer)$ and that $\hat W$ is given by \eqref{def:matrixhatW} with step-size $\epsilon > 0$.
	If
	\[
	\sigma_{\Ninnerlayer + \Nouterlayer}\left(\bbE_{X \sim \mu_X}\opvec(\nabla^2 f(X)) \otimes \opvec(\nabla^2 f(X))\right) \geq \alpha > 0,
	\]
	then we have
	\[
	\sigma_{\Ninnerlayer + \Nouterlayer}(\hat W) \geq \sqrt{m_X}\left(\sqrt{\frac{\alpha}{2}} - C_{\Delta}\epsilon\Nouterlayer\Ninnerlayer^{\frac{3}{2}}\right),
	\]
	with probability at least $1 - (\Ninnerlayer + \Nouterlayer)\exp\left( -\frac{m_X \alpha}{8 C_1 \N{A}^4\N{B}^4 m_1}\right)$, where $C_\Delta, C_1 > 0$ depend only on the constants $\kappa_j, \eta_j$ for $j=0,\dots,3$.
\end{lem}
\begin{proof}
By Weyl's inequality we obtain
\begin{align}\label{eq:bound_denom_wedin}
\sigma_{\Ninnerlayer + \Nouterlayer}(\hat W) \geq \sigma_{\Ninnerlayer + \Nouterlayer}(W) - \N{W - \hat W} \geq \sigma_{\Ninnerlayer + \Nouterlayer}(W) - C_{\Delta} \epsilon \Nouterlayer\Ninnerlayer^{\frac{3}{2}}.
\end{align}
For the first term of the right hand side we have $\sigma_{\Ninnerlayer + \Nouterlayer}(W)^2 = \sigma_{\Ninnerlayer + \Nouterlayer}(W W^T)$,
which can be written as a sum of the outer products of the columns
\begin{align*}
\sigma_{\Ninnerlayer + \Nouterlayer}(W W^T)
&= \sum_{i=1}^{m_X} \opvec(\nabla^2 f(X_i)) \otimes \opvec(\nabla^2 f(X_i)),
\end{align*}
additionally, the matrices $\opvec(\nabla^2 f(X_i)) \otimes \opvec(\nabla^2 f(X_i))$ are independent and positive definite random matrices. The Chernov bound for the eigenvalues for sums of random matrices, due to Gittens and Tropp \cite{tropp_gittens} applied to the right hand side of the last equation yields the following lower bound:
	\begin{align}\label{eq:chernov_tropp_inner}
		\sigma_{\Ninnerlayer + \Nouterlayer}\left( \sum_{i=1}^{m_X} \opvec(\nabla^2 f(X_i)) \otimes \opvec(\nabla^2 f(X_i))\right) \geq t m_X \alpha \text{ for } t \in [0,1],
	\end{align}
	with probability at least
	\begin{align*}
	1 - (\Ninnerlayer + \Nouterlayer)\exp\left( -(1 - t)^2 \frac{m_X \alpha}{2K}\right),
	\end{align*}
	where we set $K = \max_{x \in B^d_1} \N{\opvec(\nabla^2 f(x)) \otimes \opvec(\nabla^2 f(x))}$, which we wish to estimate more explicitly. First, we have to bound the norm of the Hessian matrices. Let $X \sim \mu_X$, then
	\begin{align*}
	\N{\nabla^2 f(X)}_F &\leq \sup_{x \in B_1^{d}} \N{\nabla^2 f(x)}_F = \sup_{x \in \bbB^{d}} \N{A \Gamma_x A^T + V_x T_x V_x^T}_F \\
	&\leq \sup_{x \in \bbB^{d}} \N{A}^2 \left(  \kappa_2 \norm{\ma{B}} \sqrt{\sum_{\ell =1}^{\Nouterlayer} h'_\ell(b_\ell^T g(A^T x + \theta))}  + \N{B}^2 \kappa_1^2 \norm{\Gamma_x}_F\right)\\
	&\leq  \N{A}^2 \left(  \kappa_2 \norm{\ma{B}} \eta_2 \sqrt{\Nouterlayer}  + \N{B}^2 \kappa_1^2 \eta_2 \sqrt{\Nouterlayer}\right) \leq \sqrt{C_1} \N{A}^2 \N{B}^2 \sqrt{\Nouterlayer},
	\end{align*}
	for some constant $C_1 > 0$.	Now we can further estimate $K$ by
	\begin{align*}
		K &= \max_{x\in B_1^d}\N{\opvec(\nabla^2 f(x)) \otimes \opvec(\nabla^2 f(x))} \leq \max_{x\in B_1^d}\N{\opvec(\nabla^2 f(x)) \otimes \opvec(\nabla^2 f(x))}_F\\
		&\leq \max_{x\in B_1^d}\N{\nabla^2 f(x)}_F^2 \leq C_1 \N{A}^4\N{B}^4 m_1.
	\end{align*}
	Finally, we can finish the proof by plugging the above into \eqref{eq:bound_denom_wedin} and by setting $t = \frac{1}{2}$.
\end{proof}
\begin{proof}[Proof of Theorem \ref{thm:approx_space_bound}]
	The proof is a combination of the previous lemmas together with an application of Wedin's bound \cite{st90,we72}. Given $\hat W,W^*$, let $\hat U \hat \Sigma \hat V^T, U^* \Sigma^* {V^*}^T$ be their respective singular value decompositions.
	Furthermore, denote by $\hat U_1, U_1^*$ the matrices formed by only  the first $\Ninnerlayer+\Nouterlayer$ columns of $\hat U, U^*$, respectively. According to this notation, Algorithm \ref{alg:approximation_matrix_space} returns the orthogonal projection $P_{\hat \CW}= \hat U_1 \hat U_1^T.$ We also denote by $P_{\CW^*}$ the projection given by $P_{\CW^*} = U_1^* {U_1^*}^T.$
	Then we can bound the difference of the projections by applying Wedin's bound
	\begin{align*}
	\N{P_{\hat \CW} -P_{ \CW^*}}_F = \N{\hat U_1 \hat U_1^T -U_1^* {U_1^*}^T}_F \leq  \frac{2 \N{\hat W - W^*}_F}{\bar \alpha},
	\end{align*}
	as soon as $\bar \alpha>0$ satisfies
	\[
	\bar \alpha \leq \min_{\substack{1<j\leq \Ninnerlayer + \Nouterlayer\\  \Ninnerlayer + \Nouterlayer+1 \leq k}} \abs{ \sigma_j(\hat W)- \sigma_k( W^*) } \text{ and }\bar{\alpha} \leq \min_{1\leq j \leq \Ninnerlayer + \Nouterlayer}  \sigma_{j}(\hat W).
	\]
	Since $ \CW$ has dimension $\Ninnerlayer + \Nouterlayer$, we have $\max_{ k \geq \Ninnerlayer + \Nouterlayer+1} \sigma_k(W^*) = 0$. Therefore the second inequality is equivalent to the first, and we can choose $\bar{\alpha} =  \sigma_{\Ninnerlayer + \Nouterlayer}(\hat W) \leq   \min_{1\leq j \leq \Ninnerlayer + \Nouterlayer}  \sigma_{j}(\hat W)$. Thus, we end up with the inequality
	\begin{align*}
	\N{P_{\hat \CW} -P_{\CW^*}}_F \leq \frac{2 \N{\hat W - W^*}_F}{\sigma_{\Ninnerlayer + \Nouterlayer}(\hat W)}.
	\end{align*}
	 Applying the union bound for the two events in Lemma \ref{lem:bound_num_wedin} and Lemma \ref{lem:bound_denom_wedin} in combination with the respective inequalities yields
	\begin{align}
	\N{P_{\hat \CW} -P_{\CW^*}}_F \leq  \frac{\left( C_{\Delta}\epsilon \Nouterlayer \Ninnerlayer^{\frac{3}{2}} + C\N{A}^2 \N{B}^2 \N{X}_{\psi_2}\sqrt{\Nouterlayer\log(\Ninnerlayer+1)}\right)}{\sqrt{\frac{\alpha}{2}}- C_{\Delta}\epsilon \Nouterlayer \Ninnerlayer^{\frac{3}{2}}}
	\end{align}
	with probability at least $1 - 2e^{\left(-c_1 L_1 m_X \min\left\lbrace c_2 L_1 m_X, 1\right\rbrace\right)} - (\Ninnerlayer + \Nouterlayer)e^{ -L_2 m_X }$, where
	\begin{align*}
	L_1 &:= \N{X}^2_{\psi_2}\log(\Ninnerlayer + 1)\Nouterlayer,\quad, \quad L_2 := \frac{\alpha}{8C_1 \N{A}^4\N{B}^4 \Nouterlayer},\\
	\end{align*}
	and $C,C_1,C_{\Delta}, c_1, c_2 > 0$ are the constants from the lemmas above.
\end{proof}

\section{Recovery of individual (entangled) neural network weights}
\label{sec:individual_and_assignment}
The symmetric rank-$1$ matrices $\{a_i \otimes a_i: i \in [m_0] \}\cup \{ v_\ell\otimes v_\ell: \ell \in [m_1] \}$ made of tensors of (entangled) neural network weights are the spanning elements of $\CW$,
which in turn can be approximated by $\hat \CW$ as has been proved above. In this section, we explain under which conditions it is possible to stably identify approximations to the network profiles $\{a_i: i \in [m_0] \}\cup \{ v_\ell: \ell \in [m_1] \}$ by a suitable selection process, Algorithm \ref{alg:approximation_neural_network_profiles}.

To simplify notation, we drop the differentation between weights $a_i$ and $v_\ell$
and simply denote $\CW = \left\{w_1 \otimes w_1,\ldots, w_{m} \otimes w_{m}\right\}$,
where $m = \Ninnerlayer + \Nouterlayer$, and every $w_\ell$ equals either one of the $a_i$'s or one of the $v_\ell$'s. Thus, $m$ may be larger than $d$.
We also use the notations $W_j := w_j \otimes w_j$, and
$\hat W_j:= P_{\hat \CW}(W_j)$. Provided that the
approximation error $\Perror := \N{P_{\CW} - P_{\hat \CW}}_F$ satisfies $\Perror < 1$ (cf. Theorem \ref{thm:approx_space_bound}),
$\{\hat W_j: j \in [m]\}$ is the image of a basis under a bijective map, and thus
can be used as a basis for $\hat \CW$ (see Lemma \ref{lem:bijection} in the Appendix).
We quantify the \emph{deviation from orthonormality} by $\Ferror := \Uframe - 1$, see \eqref{eq:frame_properties_A_B}.
As an example of suitable frames, normalized tight frames achieve the bounds $c_f =C_F= m/d$ \cite[Theorem 3.1]{befi03}, see also \cite{cale08}.
For instance, for such frames $m=\lceil 1.2 d \rceil >d $ would allow for $\nu=0.2$.
These finite frames are related to the Thomson problem of spherical equidistribution, which involves finding the optimal way in which to  place $m$ points  on  the  sphere $\mathbb S^{d-1}$  in $\R^d$ so  that  the  points  are  as  far  away  from each other as possible.
We further note that if $0<\Ferror < 1$ then $\{W_j: j \in [m]\}$ is a system of linearly independent matrices, hence a Riesz basis (see Lemma \ref{lem:linear_indepdency} and \eqref{eq:riesz_constants} in the Appendix). We denote the corresponding lower and upper Riesz
constants by $\Lriesz,\Uriesz$.
\\
 Finally,
for any real, symmetric matrix $X$, we let $X = \sum_{j=1}^{d}\lambda_j(X)u_j(X)\otimes u_j(X)$
be the spectral decomposition ordered according to $\|X\|=\lambda_1(X) \geq \ldots \geq \lambda_d(X)$
(in case $\lambda_1(X)= - \|X\|$, we actually consider $-X$ instead of $X$).
In the following we are able to provide in Theorem \ref{prop:largest_eigenvector} general recovery guarantees of network weights provided by the eigenvector associated to the largest eigenvalue in absolute value of any suitable matrix $M \in \hat \CW \cap \mathbb S$.

\begin{rem}\label{rem:redcase}The problem considered in this section is how to approximate
the individual $w_\ell \otimes w_\ell$  within the space $\CW$ or more precisely
by using its approximation $\hat \CW$. As the analysis below is completely unaware of how
the space $\hat \CW$ has been constructed, in particular it does not rely on the fact that
it comes from second order differentiation of a two hidden layer network, here we are actually implicitly able of addressing  also
the problem of the identification of weights for one hidden layer networks \eqref{eq:shallow_network} with a number  $m$
of neurons larger than the input dimension $d$, which was left as an open problem from \cite{fornasier2018identification}.
\end{rem}

\subsection{Recovery guarantees}
\label{subsec:nonlinear_program}
The network profiles $\{w_j, j \in [m] \}$ are (up to sign) uniquely defined by matrices $\{W_j: j \in [m] \}$
as they are precisely the  eigenvectors corresponding to the unique nonzero eigenvalue. Therefore it suffices
to recover $\{W_j: j \in [m] \}$, and we have to study when such
matrices can be uniquely characterized within the matrix space $\CW$ by their rank-$1$ property.  Let us stress that this problem is strongly related to similar and very relevant ones appearing recently in the literature addressing nonconvex programs to identify sparse vectors and low-rank matrices in linear subspaces, see, e.g., in \cite{nakatsukasa2017finding,qusuwrXX}.
In Appendix \ref{subsec:appendix_sec_4} (Lemma \ref{lem:example_rank_one_matrices} and
Corollary \ref{cor:uniqueness of matrices}) we prove that unique identification is possible
if any subset of $\lceil m/2\rceil + 1$ vectors of $\{w_j: j \in [m] \}$
is linearly independent, and that such subset linear independence is actually implied by the frame bounds \eqref{eq:frame_properties_A_B}
if $\Ferror =C_F-1 < \lceil\frac{m}{2}\rceil^{-1}$. Unfortunately, this assumption seems a bit too
restrictive in our scenario, hence we instead resort to a weaker and robust version
 given by the following result. In particular, we prove that any
near rank-$1$ matrix in $\hat \CW$ of unit Frobenius norm is not too far from one of the $W_j$'s,
provided that $\Perror$ and $\Ferror$ are small.
\begin{thm}
\label{prop:largest_eigenvector}
Let $M \in \hat \CW \cap \bbS$ and assume $\max\{\Perror,\Ferror\}\leq 1/4$. If $\lambda_1(M) > \max\{2\Perror, \lambda_2(M)\}$ then
\begin{align}\label{eq:errorbnd}
\min_{\substack{j=1,\ldots,m,\\
s \in \{-1,1\}}}\N{s w_{j} - u_1(M)} \leq \sqrt{8}\frac{\Lriesz^{-1/2}\sqrt{\Ferror} + \Ferror + 2\Perror}{\lambda_1(M) - \lambda_2(M)}.
\end{align}
\end{thm}
\noindent
Before proving Theorem \ref{prop:largest_eigenvector} we need the following technical result.
\begin{lem}
\label{lem:max_sigma_pos}
For any $M = \sum_{j=1}^{m}\sigma_j \hat W_j\in \hat \CW \cap \bbS$ with $\lambda_1(M) \geq \Perror/(1-\Perror)$ we have
$\max_{i}\sigma_i \geq 0$.
\end{lem}
\begin{proof}
Assume, to the contrary, that $\max_{j}\sigma_j < 0$, and denote $Z = \sum_{j=1}^{m} \sigma_j W_j$ with $M = P_{\hat \CW}(Z)$.
$Z$ is negative definite, since $v^TZv = \sum_{j=1}^{m}\sigma_j \left\langle w_j, v\right\rangle^2$,
and $\sigma_j < 0$ for all $i=1,\ldots,m$. Moreover, we have $\N{Z}_F \leq (1-\Perror)^{-1}$ by Lemma \ref{lem:bijection},
and thus we get a contradiction by
\begin{align*}
\frac{\Perror}{1-\Perror} \leq \lambda_1(M) \leq \lambda_1(Z) + \N{M - Z}_F < \N{M-Z}_F \leq \frac{\Perror}{1-\Perror}.
\end{align*}
\end{proof}
\begin{proof}[Proof of Theorem \ref{prop:largest_eigenvector}]
Let $\lambda_1 := \lambda_1(M)$, $u_1 := u_1(M)$ for short in this proof.
We can represent $M$ in terms of the basis elements of $\hat \CW$ as $M = \sum_{j=1}^{m}\sigma_j \hat W_j$, and let
$Z \in \CW$ satisfy $M = P_{\hat  \CW}(Z)$. Furthermore, let $\sigma_{j^*} = \max_j \sigma_j \geq 0$ where the
non-negativity follows from Lemma \ref{lem:max_sigma_pos}.
Using $Z = \sum_{j=1}^{m}\sigma_j w_j\otimes w_j$ and $\N{Z}_F \leq (1-\Perror)^{-1}$, we first notice that
\begin{equation}
\label{eq:lambda_1_lower_bound}
\begin{aligned}
\lambda_1 &= \left\langle M, u_1 \otimes u_1\right\rangle = \left\langle Z, u_1 \otimes u_1\right\rangle + \left\langle M-Z, u_1 \otimes u_1\right\rangle\\
&\leq \sum_{j=1}^{m}\sigma_j \left\langle w_j, u_1\right\rangle^2 + \N{M - Z}_F \leq \sigma_{j^*} \Uframe + 2\Perror \leq \sigma_{j^*} + \Ferror + 2\Perror,
\end{aligned}
\end{equation}
and
\begin{equation}
\label{eq:lambda_1_upper_bound}
\begin{aligned}
\lambda_1 &= \left\langle M, u_1 \otimes u_1\right\rangle \geq \max_j \left\langle Z, w_j \otimes w_j\right\rangle - 2\Perror
\geq \sigma_{j^*} + \sum_{i \neq j^*}\sigma_i \left\langle w_i, u_1\right\rangle^2 - 2 \Perror\\
&\geq \sigma_{j^*} - \N{\sigma}_{\infty}\Ferror - 2 \Perror \geq  \sigma_{j^*} -2\Ferror - 2 \Perror,
\end{aligned}
\end{equation}
where we used $\N{\sigma}_{\infty}\leq (1-\Perror)^{-1}(1-\Ferror)^{-1} \leq 2$ according to Lemma \ref{lem:technical_coefficient_and_sigma}. Hence
$\SN{\lambda_1 - \sigma_{j^*}}\leq 2\Perror + 2\Ferror$.
{Define now $Q := \Id - u_1 \otimes u_1$.
Choosing $s \in \{-1, 1\}$ so that $s\left\langle w_{j^*},u_1\right\rangle \geq 0$
we can bound the left hand side in \eqref{eq:errorbnd} by 
\begin{align*}
\N{s w_{j^*} - u_1}^2 &= 2(1 - \left\langle s w_{j^*},u_1\right\rangle) \leq 2(1 - \left\langle w_{j^*},u_1\right\rangle^2) = 2\N{Qw_{j^*}}^2 =2\N{Q W_{j^*}}_F^2.
\end{align*}
Viewing  $W_{j^*} = w_{j^*}\otimes w_{j^*}$ as the orthogonal projection onto the eigenspace
of the matrix $\lambda_1 W_{j^*}$, corresponding to eigenvalues in $[\infty, \lambda_1]$,
we can use Davis-Kahans Theorem in the version of \cite[Theorem 7.3.1]{bhatia2013matrix} to further obtain
\begin{align}
\label{eq:first_bound_wj_u1}
\N{s w_{j^*} - u_1} &\leq  \sqrt{2}\N{Q W_{j^*}}_F
\leq  \sqrt{2}\frac{\N{Q(\lambda_1 W_{j^*} - M)W_{j^*}}_F}{\lambda_1 - \lambda_2}\leq  \sqrt{2}\frac{\N{(\lambda_1 W_{j^*} - M)W_{j^*}}_F}{\lambda_1 - \lambda_2}.
\end{align}
To bound the numerator, we first use $ \N{Z - M}_F \leq \delta/(1-\delta)$ in the decomposition
\[
\N{(\lambda_1 W_{j^*} - M)W_{j^*}}_F \leq \N{(\lambda_1 W_{j^*}  - Z)W_{j^*}}_F + \N{Z - M}_F \leq \N{(\lambda_1 W_{j^*}  - Z)W_{j^*}}_F  +\frac{\Perror}{1-\Perror},
\]
and then bound the first term using $\SN{\lambda_1 - \sigma_{j^*}}\leq 2\Perror + 2\Ferror$ and the frame property \eqref{eq:frame_properties_A_B} by
\begin{align*}
\N{(\lambda_1 W_{j^*}  - Z)W_{j^*}}_F &= \N{(\lambda_1 - \sigma_{j^*})W_{j^*} + \sum\limits_{j\neq j^*}\sigma_j (w_j \otimes w_j)W_{j^*}}_F\\
&\leq \SN{\lambda_1 - \sigma_{j^*}} + \N{\sum\limits_{j\neq j^*}\sigma_j \left\langle w_{j^*}, w_j\right\rangle w_{j^*}\otimes w_j}_F \leq 2\Perror + 2\Ferror +  \sum\limits_{j\neq j^*}\SN{\sigma_j} \SN{\left\langle w_{j^*}, w_j\right\rangle}\\
&\leq 2\Perror + 2\Ferror + \N{\sigma}_2 \sqrt{\sum\limits_{j\neq j^*}\left\langle w_{j^*}, w_j\right\rangle^2} \leq 2\Perror + 2\Ferror + \N{\sigma}_2 \sqrt{\Ferror}.
\end{align*}}
Combining these estimates with \eqref{eq:first_bound_wj_u1} and $\delta/(1-\delta)\leq 2\delta$, we obtain
\begin{align*}
\N{sw_{j^*} - u_1(M)} \leq \sqrt{2}\frac{2\Perror + 2\Ferror + \N{\sigma}_2 \sqrt{\Ferror} + 2\Perror}{\lambda_1 - \lambda_2}
= \sqrt{2}\frac{\N{\sigma}_2 \sqrt{\Ferror}  + 2\Ferror+ 4\Perror}{\lambda_1 - \lambda_2}
\end{align*}
The result follows since $\{w_j \otimes w_j: j \in [m]\}$ is a Riesz basis
and thus $\N{\sigma}_2 \leq \Lriesz^{-1/2}\N{Z}_F \leq 2\Lriesz^{-1/2}$.
\end{proof}

The {preceding} result provides recovery guarantees for network weights provided by the eigenvector associated to the largest eigenvalue in absolute value of any suitable matrix $M \in \hat \CW \cap \mathbb S$. The estimate  is inversely proportional to the spectral gap $\lambda_1(M) - \lambda_2(M)$. The problem then becomes the constructive identification of matrices $M$ belonging to $\hat \CW \cap \mathbb S$, which {simultaneously} maximize the spectral gap.
Inspired by the results in  \cite{fornasier2018identification} we propose to consider the following nonconvex program as selector of such matrices
\begin{align}
\label{eq:nonlinear_program}
M = \argmax \N{M}\quad \textrm{s.t.}\quad M \in \hat \CW,\quad \N{M}_F \leq 1.
\end{align}
By maximizing the spectral norm under a Frobenius norm constraint, a local maximizer
of the program should be as nearly rank one as possible within a given neighborhood. Moreover,
if rank one matrices exist in $\hat \CW$, these are precisely the global optimizers.

\subsection{A nonlinear program: properties of local maximizers of \eqref{eq:nonlinear_program}}
\label{subsec:properties_of_local_maximizers}
In this section we prove that, except for spurious cases, local maximizers of \eqref{eq:nonlinear_program}
are generically almost rank-$1$ matrices in $\hat \CW$. In particular we show that local maximizers
either satisfy $\N{M}^2 \geq 1 - c\Perror - c'\Ferror$, for some small constants $c, c'$, implying near minimal rankness,
or $\N{M}^2 \leq c\Perror + c' \Ferror$, \emph{i.e.}, all eigenvalues of $M$ are small,  the mentioned spurious cases.
Before addressing these estimates, we provide a characterization of the first and second order optimality conditions for \eqref{eq:nonlinear_program}, see \cite{fornasier2018identification} and also \cite{Tao_RM,Tao_Blog}.
\begin{thm}[{Theorem 3.4} in \cite{fornasier2018identification}]
\label{thm:optimality_conditions}
Let $M \in \hat \CW \cap \bbS$ and assume there exists
a unique $i^* \in [d]$ satisfying $\SN{\lambda_{i^*}(M)} = \N{M}$.
If $M$ is a local maximizer \eqref{eq:nonlinear_program} then it fulfills the stationary or first order optimality condition
\begin{align}
\label{eq:first_order_optimality}
u_{i^*}(M)^T X u_{i^*}(M) = \lambda_{i^*}(M) \left\langle X, M\right\rangle
\end{align}

\noindent
for all $X \in \hat \CW$.
A stationary point $M$ (in the sense that $M$ fulfills \eqref{eq:first_order_optimality}) is a local maximizer of \eqref{eq:nonlinear_program} if and only if for all $X \in \hat \CW$
\begin{align}
\label{eq:second_order_optimality}
2 \sum\limits_{i\neq i^*} \frac{(u_{i^*}(M)^T X u_k(M))^2}{\SN{\lambda_{i^*}(M) - \lambda_k(M)}} \leq \SN{\lambda_{i^*}(M)} \N{X - \left\langle X, M\right\rangle M}_F^2.
\end{align}
\end{thm}
\begin{proof}
For simplicity we drop the argument $M$ in $\lambda_{i}$, $u_i$, and
without loss of generality we assume $\lambda_{i^*} = \N{M}$, otherwise we consider $-M$.
Following the analysis in \cite{fornasier2018identification}, for $X \in \hat \CW \cap \bbS$ we can consider the function
\begin{align*}
f_{X}(\alpha) = \frac{\N{M + \alpha X}}{\N{M + \alpha X}_F},
\end{align*}
because $M$ is a local maximizer if and only if $\alpha = 0$ is a local maximizer of $f_{X}$
for all $X \in \hat \CW \cap \bbS$.

\noindent
Let us consider $X \in \hat \CW \cap \bbS$ with $X\perp M$ first. We note that the simplicity of
$\lambda_{i^*}$ implies that there exist analytic functions $\lambda_{i^*}(\alpha)$
and $u_{i^*}(\alpha)$ with $(M+\alpha X)u_{i^*}(\alpha) = \lambda_{i^*}(\alpha) u_{i^*}(\alpha)$
for all $\alpha$ in a neighborhood around $0$ \cite{magnus1985differentiating,rellich1969perturbation}.
Therefore we can use a Taylor expansion $\N{M+\alpha X} = \lambda_{i^*} + \lambda'_{i^*}(0)\alpha + \lambda''_{i^*}(0)\alpha^2/2 + \CO(\alpha^3)$
and combine it with $\N{M+\alpha X}_F = \sqrt{1 + \alpha^2} = 1 - \alpha^2/2 + \CO(\alpha^4)$ to get
\begin{align*}
f_X(\alpha) = \left(1 - \alpha^2/2\right)\left(\lambda_{i^*} + \lambda'_{i^*}(0)\alpha + \lambda''_{i^*}(0)\alpha^2/2\right) + \CO(\alpha^3) \quad \textrm{as $\alpha \rightarrow 0$. }
\end{align*}
Differentiating once we get $f_X'(0) = \lambda'_{i^*}(0)$, hence $\alpha = 0$ is
a stationary point if and only if $\lambda'_{i^*}(0)$ vanishes. Following the computations in \cite{fornasier2018identification},
we find that $\lambda'_{i^*}(0) = u_{i^*}(0)^T X u_{i^*}(0) = 0$, and thus \eqref{eq:first_order_optimality} follows for any $X \perp M$.
For general $X$, we split $X = \left\langle X, M\right\rangle M + X_{\perp}$, and get
$u_{i^*}(0)^T X u_{i^*}(0) = \left\langle X, M\right\rangle u_{i^*}(0)^T M u_{i^*}(0) = \lambda_{i^*}(0)\left\langle X, M\right\rangle$.

\noindent
For \eqref{eq:second_order_optimality}, we have to check additionally $f''_X(\alpha) \leq 0$. The second derivative of $f_X(\alpha)$ at zero
is given by $f_X''(0) = \lambda''_{i^*}(0) - \lambda_{i^*}(0)$, hence the condition for attaining  a local maximum is
$\lambda''_{i^*}(0) \leq \lambda_{i^*}(0)$. Again, we can follow the computations in \cite{fornasier2018identification}
to obtain
\begin{align*}
\lambda''_{i^*}(0) = 2 \sum_{i\neq i^*}\frac{(u_{i^*}^T(0) X u_{k}(0))^2}{\SN{\lambda_{i^*}(0)-\lambda_k(0)}},
\end{align*}
and \eqref{eq:second_order_optimality} follows immediately for any $X \perp M$, $\N{X}_F = 1$. For general $X$
we decompose it into $X = \left\langle X, M\right\rangle M + X_{\perp}$. Since $u_{i^*}^T(0) M u_{k}(0) = 0$ for all $k\neq i^*$, we get
\begin{align*}
2 \sum_{i\neq i^*}\frac{(u_{i^*}^T(0) \left(\left\langle X, M\right\rangle M + X_{\perp}\right) u_{k}(0))^2}{\SN{\lambda_{i^*}(0)-\lambda_k(0)}}
=2\N{X_{\perp}}_F^2\sum_{i\neq i^*}\frac{\left(u_{i^*}^T(0) \left(\frac{X_{\perp}}{\N{X_{\perp}}_F}\right) u_{k}(0)\right)^2}{\SN{\lambda_{i^*}(0)-\lambda_k(0)}}
\leq \lambda_{i^*}(0)\N{X_{\perp}}_F^2,
\end{align*}
and the result follows from $\N{X_{\perp}}_F = \N{X - \left\langle X, M\right\rangle M}_F$.
\end{proof}

For simplicity, we denote $u_i := u_i(M)$ and $\lambda_i = \lambda_i(M)$ throughout
the rest of this section.
Moreover, we assume $M$ satisfies
\begin{enumerate}[label=({A\arabic*})]
\item\label{ass:A1} $\lambda_1 = \N{M}$ (this is without loss of generality because $-M$ and $M$ may be both local maximizers),
\item\label{ass:A2} $\lambda_1 > \lambda_2$ (this is a useful technical condition in order to use the second order optimality condition \eqref{eq:second_order_optimality}).
\end{enumerate}
%

To derive the bounds for $\lambda_1$, we establish an inequality
$0\leq \lambda_1^2(\lambda_1^2 - 1) + c\Perror + c'\Ferror$, which implies that
$\lambda_1^2(M)$ is either close to $0$ or close to $1$. A first ingredient for obtaining the inequality is
\begin{align}
\label{eq:inequality_main_theorem_first_part}
\N{\hat W_ju_1}_2^2 \geq u_1^T \hat W_j u_1 - 2\Perror = \lambda_1\left\langle \hat W_j, M\right\rangle - 2 \Perror,
\end{align}
where we used $\SN{\Vert\hat W_ju_1\Vert^2 - u_1^T \hat W_j u_1} \leq 2\Perror$ in the inequality, see Lemma \ref{lem:technical_coefficient_and_sigma}
in Appendix \ref{subsec:appendix_sec_4}, and
\eqref{eq:first_order_optimality} in the equality. The other useful technical estimate is provided in the following Lemma,
which is proven by leveraging the second order optimality condition \eqref{eq:second_order_optimality}.
\begin{lem}
\label{lem:Xu1}
Assume that $M$ is a local maximizer satisfying (A1) and (A2) and let $\max\{\Perror, \Ferror\} < 1/4$.
For any $X \in \hat \CW$ with $\N{X}_F \leq 1$ we have
\begin{equation}
\label{eq:inequality_second_part}
 \N{Xu_1}_2^2 \leq \lambda_1^2\frac{1 + \left\langle X, M\right\rangle^2}{2} +5\Perror + 2\Ferror.
\end{equation}
\end{lem}
\noindent
For the proof of Lemma \ref{lem:Xu1} we need a lower bound for the smallest eigenvalue (see Appendix \ref{subsec:appendix_sec_4}
for the proof of Lemma \ref{lem:lower_bound_lambdam}).
\begin{lem}
\label{lem:lower_bound_lambdam}
Assume that $M$ is a stationary point of \eqref{eq:nonlinear_program} satisfying (A1) and (A2). If $\max\{\Perror, \Ferror\} < 1/4$, then $\lambda_D \geq -2\Perror \lambda_1^{-1} - 8\Perror - 4\Ferror$.
\end{lem}
\noindent
\begin{proof}[Proof of Lemma \ref{lem:Xu1}]
We first use \eqref{eq:first_order_optimality} and \eqref{eq:second_order_optimality} to get
\begin{align*}
\frac{2}{\lambda_1 - \lambda_D}\left(\N{Xu_1}_2^2 -\lambda_1^2\left\langle X, M\right\rangle^2\right) &=
\frac{2}{\lambda_1 - \lambda_D}\left(\N{Xu_1}_2^2 -\left\langle Xu_1, u_1\right\rangle^2\right) =
\frac{2}{\lambda_1 - \lambda_D}
\sum\limits_{i=2}^{D}\left\langle Xu_1, u_k\right\rangle^2\\
&\leq 2\sum\limits_{i=2}^{D}\frac{(u_1^T X u_k)^2}{\lambda_1 - \lambda_k} \leq \lambda_1 \N{X - \left\langle X, M\right\rangle M}_F^2,
\end{align*}
and then rearrange the inequality to obtain
\begin{align*}
\N{Xu_1}_2^2 &\leq \frac{\lambda_1(\lambda_1 - \lambda_D)}{2}\left(\N{X}_F^2 - \left\langle X, M\right\rangle^2\right) + \lambda_1^2\left\langle X, M\right\rangle^2 \leq\frac{\lambda_1(\lambda_1 - \lambda_D)}{2}  +\frac{\lambda_1(\lambda_1 + \lambda_D)}{2}\left\langle X, M\right\rangle^2\\
&= \lambda_1^2\frac{1 + \left\langle X, M\right\rangle^2}{2} - \lambda_1\lambda_D \frac{1-\left\langle X, M\right\rangle^2}{2}.
\end{align*}
Using the lower bound for $\lambda_D$ from Lemma \ref{lem:lower_bound_lambdam}, and $\lambda_1\leq 1$, we get
\begin{align*}
\N{Xu_1}_2^2 &\leq \lambda_1^2\frac{1 + \left\langle X, M\right\rangle^2}{2} +\lambda_1
\left(2\Perror\lambda_1^{-1} + 8\Perror + 4\Ferror\right)
\frac{1-\left\langle X, M\right\rangle^2}{2}\\
&\leq\lambda_1^2\frac{1 + \left\langle X, M\right\rangle^2}{2} + (10\Perror + 4\Ferror)
\frac{1-\left\langle X, M\right\rangle^2}{2}= \lambda_1^2\frac{1 + \left\langle X, M\right\rangle^2}{2} + 5\Perror + 2\Ferror.
\end{align*}
\end{proof}
\noindent
By combining \eqref{eq:inequality_main_theorem_first_part}
and \eqref{eq:inequality_second_part}
the bounds for $\lambda_1$ follow.
\begin{thm}
\label{thm:lower_bound_eigenvalue}
Assume that $M$ is a local maximizer of \eqref{eq:nonlinear_program} satisfying \ref{ass:A1} and \ref{ass:A2}, and assume
$38\Perror + 13\Ferror < 1/4$. Then we have $\lambda_1^2 \geq 1- 38\Perror - 13\Ferror$ or $\lambda_1^2 \leq 38\Perror + 13\Ferror.$
\end{thm}
\begin{proof}
Let $j^* = \argmax_j \sigma_j$. We first note that we can assume $\sigma_{j^*} \geq 0$
without loss of generality by Lemma \ref{lem:max_sigma_pos}, since there is nothing to show if $\lambda_1\leq 2\Perror$.
Now we consider \eqref{eq:inequality_main_theorem_first_part}
and \eqref{eq:inequality_second_part} for $X = \hat W_{j^*}$ to get the inequality
\begin{equation}
\label{eq:main_thm_inequalities}
\begin{aligned}
&\phantom{or, equivalently, } \lambda_1^2\frac{1 + \left\langle \hat W_{j^*}, M\right\rangle^2}{2} +5\Perror + 2\Ferror \geq \lambda_1\left\langle \hat W_{j^*}, M\right\rangle - 2 \Perror,\\
&\mbox{or, equivalently, } 0\leq \lambda_1^2 - 1 + \left(1 - \lambda_1\left\langle \hat W_j, M\right\rangle\right)^2 + 14\Perror + 4\Ferror\\
&\mbox{or, equivalently, }  0\leq \lambda_1^2 - 1 + \left(1 - \lambda_1 \sigma_{j^*} \N{\hat W_j}_F^2 + \lambda_1 \left(\sigma_{j^*} \N{\hat W_j}_F^2 - \left\langle \hat W_j, M\right\rangle\right) \right)^2 + 14\Perror + 4\Ferror.
\end{aligned}
\end{equation}
We separate two cases. In the first case we have $\sigma_{j^*} > 1$, which implies
$\langle \hat W_j, M\rangle > 1 - 5\Perror - 2\Ferror$ and thus $\langle W_j, M\rangle > 1 - 6\Perror - 2\Ferror$
by Lemma \ref{lem:technical_coefficient_and_sigma}  and $\max\{\Perror,\Ferror\} < 1/4$. Since $\langle W_j, M\rangle = w_j^T M w_j$, this implies
$\lambda_1 > 1 - 6\Perror - 2\Ferror$, \emph{i.e.}, the result is proven. We continue with the case $\sigma_{j^*} \leq 1$, which implies
$\lambda_1 \sigma_{j^*} \Vert \hat W_j\Vert_F^2 \leq 1$. Using Lemma \ref{lem:technical_coefficient_and_sigma}
to bound $\sigma_{j^*} \Vert\hat W_j\Vert_F^2 - \langle \hat W_j, M\rangle$, $\lambda_1 < 1$
and $\Vert \hat W_j\Vert _F^2 \geq 1 - 2\Perror$,
the last inequality in \eqref{eq:main_thm_inequalities} implies
\begin{align}
\label{eq:aux_inequality_almost_there}
0\leq \lambda_1^2 - 1 + \left(1 - \lambda_1 \sigma_{j^*} + 6\Perror + 2\Ferror  \right)^2 + 14\Perror + 4\Ferror.
\end{align}
Furthermore, by following the computation we performed for \eqref{eq:lambda_1_lower_bound},
we get $\sigma_{j^*} \geq \lambda_1 - \Ferror - 2\Perror$, and inserting it in
\eqref{eq:aux_inequality_almost_there} we obtain
\begin{align*}
0&\leq \lambda_1^2 - 1 + \left(1 - \lambda_1^2 + 8\Perror + 3\Ferror  \right)^2 + 14\Perror + 4\Ferror,
\mbox{ implying } 0 \leq \lambda_1^2\left(\lambda_1^2 - 1\right) + 38\Perror + 13\Ferror.
\end{align*}
Provided that $38\Perror + 13\Ferror < 1/4$, this quadratic inequality (in the unknown $\lambda_1^2$) has solutions $\lambda_1^2 \geq 1 - 38\Perror - 13\Ferror$, or $\lambda_1^2 \leq 38\Perror + 13\Ferror$.
\end{proof}

\subsection{Analysis of the projected gradient ascent iteration}

In Section \ref{subsec:properties_of_local_maximizers} we analyze local maximizers of
\eqref{eq:nonlinear_program} and show that there exist small constants $c,c'$
such that either $\N{M}^2 \geq 1 - c\Perror + c'\Ferror$, or $\N{M}^2 \leq c\Perror + c'\Ferror$.
Therefore, a local maximizer of \eqref{eq:nonlinear_program}
is either almost rank-$1$, or it has its energy distributed across many eigenvalues.
This criterion can be easily checked in practice, and therefore maximizing \eqref{eq:nonlinear_program}
is a suitable approach for finding near rank-$1$ matrices in $\hat \CW$.
In this section, we show how those individual symmetric rank-$1$ tensors can be approximated
by a simple iterative algorithm, Algorithm \ref{alg:approximation_neural_network_profiles}, making exclusive use of the projection $P_{\hat \CW}$. Algorithm \ref{alg:approximation_neural_network_profiles} strives  to
solve the nonconvex program \eqref{eq:nonlinear_program}, by iteratively increasing the spectral norm of its iterations.
Our approach is closely related to the projected gradient ascent iteration \cite[Algorithm 4.1]{fornasier2018identification}, but
we introduce some modifications, {in particular we exchange the order of the normalization and the projection onto  $\hat \CW$.} The proof of convergence of \cite[Algorithm 4.1]{fornasier2018identification} takes advantage of that different ordering of these operations to address
 the case where $\CW$ is spanned by at most $m \leq d$ rank-$1$ matrices formed as tensors of nearly orthonormal vectors (after whitening). In fact,  its analysis is heavily based on
approximated singular value or spectral decompositions. Unfortunately in our case the decomposition ${M=\sum_{j=1}^m \sigma_j w_j \otimes w_j}$ does not approximate the singular value or spectral decomposition since the $w_j$'s are redundant (they form a frame) and therefore are not properly nearly orthonormal in the sense required in \cite{fornasier2018identification}.

 Algorithm \ref{alg:approximation_neural_network_profiles} is based on the iterative application of the operator $\Fiter$
 defined by
\begin{align}
\label{eq:def_Fiter}
\Fiter(X) := P_{\bbS}\circ P_{\hat \CW}(X + \gamma u_1(X) \otimes u_1(X)),
\end{align}
with $\gamma>0$ and $P_{\bbS}$ as the projection onto the sphere $\bbS = \{X : \N{X}_F = 1\}$. The following Lemma shows that, if $\lambda_1(X) > 0$, the operator
$\Fiter$ is well-defined, in the sense that it is a single-valued operator.

\begin{algorithm}[t]

    \KwIn{$P_{\hat \CW}$ with arbitrary basis $\{b_i\}_{i=1,\ldots,m}$, stepsize $\gamma > 0$, number of iterations $J$}

    \Begin{

        Sample $g \sim \CN(0,\Id_{m})$, and let $M_0 := P_{\bbS}(\sum_{i=1}^{m}g_i b_i)$.\\

        If $\N{M_0}$ is not an eigenvalue, take $-M_0$ instead. \\

        \For{j = 1,\ldots,J}{

        $M_{j+1} = P_{\bbS} \circ P_{\hat \CW}(M_j + \gamma u_1(M_j)\otimes u_1(M_j))$.

        }

    }

    \KwOut{$u_1(M_{J})$}

    \caption{Approximating neural network profiles}

    \label{alg:approximation_neural_network_profiles}

\end{algorithm}
\begin{lem}
\label{lem:iteration_well_defined}
Let $X\in \hat \CW \cap\bbS$ with $\lambda_1(X) > 0$ and $\gamma>0$. Then $\N{P_{\hat \CW}(X + \gamma u_1(X) \otimes u_1(X))}_F^2 = 1 + 2\gamma \lambda_1(X)+ \gamma^2 \N{P_{\hat \CW}(u_1(X)\otimes u_1(X))}_F^2$.
In particular, $\Fiter(X)$ is well-defined and can be explicitly expressed as
\begin{align*}
\Fiter(X) = \frac{P_{\hat \CW}(X + \gamma u_1(X) \otimes u_1(X))}{\N{P_{\hat \CW}(X + \gamma u_1(X) \otimes u_1(X))}_F}.
\end{align*}
\end{lem}
\begin{proof}
The result follows from $\left\langle X, P_{\hat \CW}(u_1(X)\otimes u_1(X))\right\rangle = \lambda_1(X)$ and computing explicitly the squared norm $\N{P_{\hat \CW}(X + \gamma u_1(X) \otimes u_1(X))}_F^2$.
\end{proof}
We analyze next the sequence $(M_j)_{j \in \bbN}$ generated by Algorithm \ref{alg:approximation_neural_network_profiles}.
We show that $(\lambda_1(M_j))_{j \in \bbN}$ is a strictly monotone increasing sequence,
converging to a well-defined limit $\lambda_{\infty}=\lim_{j\rightarrow \infty}\lambda_1(M_j)$,
and, if $\lambda_1(M_j) > 1/\sqrt{2}$ for some $j$, all convergent subsequences of $(M_j)_{j \in \bbN}$ converge to fixed points
of $\Fiter$. Moreover, we prove
that such fixed points  satisfy \eqref{eq:first_order_optimality}, and are
thus stationary points of \eqref{eq:nonlinear_program}. We begin by providing two equivalent characterizations of \eqref{eq:first_order_optimality}.
\begin{lem}
\label{lem:equivalency_to_characterization}
For $M \in \hat \CW$ and $c \neq 0$ we have
\begin{align*}
v^T X v = c \left\langle X,M\right\rangle\quad \textrm{ for all } X \in \hat \CW\quad  \textrm{ if and only if }\quad M = c^{-1} P_{\hat \CW}(v \otimes v).
\end{align*}
\end{lem}
\begin{proof} Assume that $v^T X v = c \left\langle X,M\right\rangle$  for all $X$.
We notice that the assumption is equivalent to $\left\langle X, v\otimes v - c M\right\rangle = 0$ for all $X\in \hat \CW$.
Therefore $P_{\hat \CW}(v\otimes v - c M) = 0$, and the result follows from $M \in \hat \CW$. In the case where $M = c^{-1} P_{\hat \CW}(v \otimes v)$, we compute
$c\left\langle X, M\right\rangle = \left\langle X, P_{\hat \CW}(v \otimes v)\right\rangle = v^T X v$ since $X\in \hat \CW$.
\end{proof}
\begin{lem}
\label{lem:equivalent_characterization_lemma}
Let $X \in \hat \CW \cap \bbS$.
We have $\N{P_{\hat \CW}(u_j(X)\otimes u_j(X))}_F \geq \SN{\lambda_j(X)}$
with equality if and only if $X = \lambda_j(X)^{-1}P_{\hat \CW}(u_j(X)\otimes u_j(X))$.
\end{lem}
\begin{proof}
We drop the argument $X$ for $\lambda_j(X)$ and $u_j(X)$ for simplicity.
We first calculate
\begin{align}
\label{eq:aux_1}
\N{P_{\hat \CW}(u_j \otimes u_j)}_F = \N{P_{\hat \CW}(u_j \otimes u_j)}_F\N{X}_F \geq \SN{\left\langle P_{\hat \CW}(u_j \otimes u_j), X\right\rangle} = \SN{\left\langle u_j \otimes u_j, X\right\rangle} = \SN{\lambda_j}.
\end{align}
Moreover, we have equality if and only if $\N{P_{\hat \CW}(u_j \otimes u_j)}_F =  \SN{\lambda_j}$, hence \eqref{eq:aux_1} is actually
a chain of equalities. Specifically,
\[
\N{P_{\hat \CW}(u_j \otimes u_j)}_F\N{X}_F = \SN{\left\langle P_{\hat \CW}(u_j \otimes u_j), X\right\rangle},
\]
which implies $X = c P_{\hat \CW}(u_j \otimes u_j)$ for some scalar  $c$. Since $\N{X}_F = 1$, $c= \lambda_j^{-1}$ follows from
\begin{align*}
1 = \left\langle c P_{\hat \CW}(u_j \otimes u_j), X\right\rangle = c \left\langle u_j \otimes u_j, X\right\rangle = c \lambda_j.
\end{align*}
\end{proof}
\noindent
Lemma \ref{lem:equivalency_to_characterization} and Lemma \ref{lem:equivalent_characterization_lemma} show that
the stationary point condition \eqref{eq:first_order_optimality} for $M$ with $\N{M} = \SN{\lambda_{i^*}(M)}$ and isolated $\lambda_{i^*}$ is equivalent to
both
{
\[
M = \lambda_{i^*}^{-1}P_{\hat \CW}(u_{i^*}(M) \otimes u_{i^*}(M)),\quad\textrm{and }\quad
\N{P_{\hat \CW}(u_{i^*}(M)\otimes u_{i^*}(M))}_F = \SN{\lambda_{i^*}(X)}.
\]}
A similar condition appears naturally if we characterize the fixed points of $\Fiter$.
\begin{lem}
\label{lem:iteration_and_stopping_condition}
Let $\gamma > 0$ and $X \in \hat \CW \cap \bbS$ with $\lambda_1(X) > 0$.
Then we have
\begin{align}
\label{eq:iteration_increases_eigenvalue}
0 < \lambda_1(X) < \N{P_{\hat \CW}(u_1(X) \otimes u_1(X))}_F\quad &\textrm{ if and only if }\quad \lambda_1(F(X)) > \lambda_1(X),\\
\label{eq:stopping_condition}
\lambda_1(X) = \N{P_{\hat \CW}(u_1(X) \otimes u_1(X))}_F\quad &\textrm{ if and only if }\quad \Fiter(X) = X.
\end{align}
\end{lem}
\begin{proof}
For simplicity we denote $u:=u_1(X)$ and $\lambda = \lambda_1(X)$ in this proof.
{We first prove that $0 < \lambda < \N{P_{\hat \CW}(u \otimes u)}_F$ implies $ \lambda_1(F(X)) > \lambda$.}
It suffices to show that there exists any unit vector $v$ such that $v^T \Fiter(X) v> \lambda$. In particular,
we can test $\Fiter(X)$ with $v=u$, which yields the identity
\begin{align*}
u^T \Fiter(X) u - \lambda &= \N{P_{\hat \CW}(X + \gamma u \otimes u)}_F^{-1}\left\langle P_{\hat \CW}(X + \gamma u \otimes u), u\otimes u\right\rangle - \lambda\\
&=\N{P_{\hat \CW}(X + \gamma u \otimes u)}_F^{-1}\left(\left\langle  X, u\otimes u\right\rangle + \gamma \left\langle P_{\hat \CW}(u \otimes u), u\otimes u\right\rangle\right)- \lambda\\
&=\N{P_{\hat \CW}(X + \gamma u \otimes u)}_F^{-1}\left(\lambda + \gamma \N{P_{\hat \CW}(u \otimes u)}_F^2\right)- \lambda\\
&=\frac{\lambda \left( 1 - \N{P_{\hat \CW}(X + \gamma u \otimes u)}_F \right) + \gamma \N{P_{\hat \CW}(u \otimes u)}_F^2}{\N{P_{\hat \CW}(X + \gamma u \otimes u)}_F}.
\end{align*}
By using now $\lambda < \N{P_{\hat \CW}(u \otimes u)}_F$, we can bound
\begin{align*}
&1 - \N{P_{\hat \CW}(X + \gamma u\otimes u)}_F = 1 - \sqrt{\N{P_{\hat \CW}(X + \gamma u\otimes u)}_F^{2}} = 1 - \sqrt{\N{X + \gamma P_{\hat \CW}(u \otimes u)}_F^2} \\
&=1 - \sqrt{1 + \gamma^2\N{P_{\hat \CW}(u \otimes u)}_F^2  + 2\gamma\left\langle X, P_{\hat \CW}(u \otimes u)\right\rangle} =1 - \sqrt{1 + \gamma^2\N{P_{\hat \CW}(u \otimes u)}_F^2  + 2\gamma\lambda}\\
&>1 - \sqrt{1 + \gamma^2\N{P_{\hat \CW}(u \otimes u)}_F^2  + 2\gamma \N{P_{\hat \CW}(u \otimes u)}_F} = 1 - \sqrt{\left(1 + \gamma \N{P_{\hat \CW}(u \otimes u)}_F\right)^2}\\
&= - \gamma \N{P_{\hat \CW}(u \otimes u)}_F.
\end{align*}
Inserting this inequality in the previous identity, we obtain the wished result  by
\begin{equation}
\label{eq:eigenvalue_increase}
\begin{aligned}
u \Fiter(X)u - \lambda &=\frac{\lambda \left( 1 - \N{P_{\hat \CW}(X + \gamma u \otimes u)}_F \right) + \gamma \N{P_{\hat \CW}(u \otimes u)}_F^2}{\N{P_{\hat \CW}(X + \gamma u \otimes u)}_F} \\
&>\frac{- \lambda \gamma \N{P_{\hat \CW}(u \otimes u)}_F + \gamma \N{P_{\hat \CW}(u \otimes u)}_F^2}{\N{P_{\hat \CW}(X + \gamma u \otimes u)}_F} > 0.
\end{aligned}
\end{equation}
{We show now that $\Fiter(X) = X$ implies $\lambda = \N{P_{\hat \CW}(u_1(X) \otimes u_1(X))}_F$.} We notice that $\Fiter(X) = X$ implies $\lambda(\Fiter(X)) = \lambda$, and thus $\lambda \geq \Vert P_{\hat \CW}(u \otimes u)\Vert_F$ according to \eqref{eq:iteration_increases_eigenvalue}. Since generally
$\lambda \leq \Vert P_{\hat \CW}(u \otimes u)\Vert_F$ by Lemma \ref{lem:equivalent_characterization_lemma}, equality follows.\\
\noindent
{We address now the converse, {\it i.e.}, $\lambda = \N{P_{\hat \CW}(u \otimes u)}_F$ implies $\Fiter(X) = X$, and we note that $\lambda = \Vert P_{\hat \CW}(u \otimes u)\Vert_F$ implies
$X = \lambda^{-1}P_{\hat \CW}(u \otimes u)$ by Lemma \ref{lem:equivalent_characterization_lemma}.}
Using this, and the definition of $\Fiter(X)$ we get
\begin{align}
\label{eq:termination_equation}
\Fiter(X) &=  \frac{P_{\hat \CW}(X + \gamma u \otimes u)}{\N{P_{\hat \CW}(X + \gamma u \otimes u)}_F} = \frac{(\lambda^{-1} + \gamma) P_{\hat \CW}(u \otimes u)}{(\lambda^{-1} + \gamma)\N{P_{\hat \CW}(u \otimes u)}_F} =
\frac{P_{\hat \CW}(u \otimes u)}{\N{P_{\hat \CW}(u \otimes u)}_F} = X
\end{align}
{To conclude the proof it remains to show $\lambda_1(F(X)) > \lambda$ implies $0 < \lambda < \N{P_{\hat \CW}(u \otimes u)}_F$.} As $\lambda \leq \Vert P_{\hat \CW}(u \otimes u)\Vert_F$ and $\lambda_1(F(X)) > \lambda$ implies
$\Fiter(X) \neq X$ and therefore $\lambda \neq \Vert P_{\hat \CW}(u \otimes u)\Vert_F$, then necessarily $\lambda <\Vert P_{\hat \CW}(u \otimes u)\Vert_F$.
\end{proof}

\noindent
The preceding Lemma implies the convergence of $(\lambda_1(M_j))_{j \in \bbN}$ by monotonicity. Moreover,
we can also use such convergence to establish $\|M_{j+1}-M_{j} \|_F \rightarrow 0$.
\begin{lem}
\label{lem:convergence_results_1}
Let $\gamma > 0$, $M_0 \in \hat \CW \cap \bbS$ with $\lambda_1(M_0) > 0$, and let $M_j := \Fiter(M_{j-1})$. The sequence $(\lambda_1(M_j))_{j \in \bbN}$
converges to a well-defined limit $\lambda_{\infty}$, and $\lim_{j\rightarrow \infty} \|M_{j+1}-M_{j} \|_F= 0$.
\end{lem}
\begin{proof}
Denote $U_j := P_{\hat \CW}(u(M_j) \otimes u(M_j))$, $\lambda_j = \lambda(M_j)$ for simplicity.
The sequence $(\lambda_j)_{j\in \bbN}$ is monotone in the bounded domain $[0,1]$ by Lemma \ref{lem:iteration_and_stopping_condition}
and therefore converges to a limit $\lambda_{\infty}$. To prove $\|M_{j+1}-M_{j} \|_F\rightarrow 0$,
we will exploit $(\lambda_{j+1} - \lambda_{j})\rightarrow 0$. We first have $(\N{U_j}_F - \lambda_j) \rightarrow 0$
since \eqref{eq:eigenvalue_increase} yields
\begin{equation}
\label{eq:convergence_1}
\lambda_{j+1} - \lambda_{j} \geq \frac{\gamma\N{U_j}_F}{\N{M_j + \gamma U_j}}\left(\N{U_j}_F - \lambda_j\right) \geq \frac{\gamma}{1+\gamma}\N{U_j}_F\left(\N{U_j}_F - \lambda_j\right),
\end{equation}
and $\N{U_j}_F \geq \lambda_j \geq \lambda_0$ for all $j$.
Define the shorthand $\Delta_j := \N{U_j}_F - \lambda_j$. We will now show that $\N{M_{j+1} - M_j}_F \leq C\Delta_j$ for
some constant $C$. First notice that
\[
\N{M_j - \lambda_j^{-1}U_j}_F = \sqrt{1 + \frac{\N{U_j}_F^2}{(\lambda_j)^2} - 2\lambda_j^{-1}\left\langle M_j, U_j\right\rangle}
= \sqrt{\frac{\N{U_j}_F^2}{(\lambda_j)^2} - 1}  = \sqrt{\frac{\N{U_j}_F^2 - (\lambda_j)^2}{(\lambda_j)^2}} \leq \lambda_0^{-1}\sqrt{2\Delta_j}.
\]
Therefore there exists a matrix $E_j$ with $M_j = \lambda_j^{-1}U_j + E_j$ and $\N{E_j}\leq \lambda_0^{-1}\sqrt{2\Delta_j}$.
Furthermore, by the triangle inequality we have
\begin{align*}
\N{M_{j+1} - M_j}_F \leq \N{M_{j+1} - \lambda_j^{-1}U_j}_F + \lambda_{0}^{-1}\sqrt{2\Delta_j},
\end{align*}
hence it remains to bound the first term. Using $M_j = \lambda_j^{-1}U_j + E_j$ and $M_{j+1} =  \N{M_j + \gamma U_j}_F^{-1} (M_j + \gamma U_j)$, we
have $\N{M_j + \gamma U_j}_FM_{j+1} = (\lambda_j^{-1} + \gamma)U_j + E_j$
and thus
\begin{align*}
&\N{\N{M_j + \gamma U_j}_F(M_{j+1}- \lambda_j^{-1}U_j)}_F = \N{(\lambda_j^{-1}+\gamma)U_j + E_j -\N{(\lambda_j^{-1}+\gamma)U_j + E_j}_F \lambda_j^{-1} U_j}_F\\
&\quad\quad\quad\leq \SN{\lambda_j^{-1}+\gamma - \N{(\lambda_j^{-1}+\gamma)U_j + E_j}_F\lambda^{-1}}\N{U_j}_F + \N{E_j}_F\\
&\quad\quad\quad\leq \left((\lambda_j^{-1}+\gamma)\N{U_j}_F\lambda_j^{-1} - (\lambda_j^{-1}+\gamma) + 2\N{E_j}_F \lambda_j^{-1}\right) \N{U_j}_F + \N{E_j}_F\\
&\quad\quad\quad\leq (\lambda_j^{-1}+\gamma)\left(\N{U_j}_F\lambda_j^{-1} - 1\right) \N{U_j}_F + (1 + 2\lambda_0^{-1})\N{E_j}_F \\
&\quad\quad\quad\leq (\lambda_0^{-1} + \gamma)\lambda_0^{-1}\Delta_j + (1 + 2\lambda_0^{-1})\sqrt{\Delta_j}.
\end{align*}
Since $\N{M_j + \gamma U_j}_F \geq 1$ according to Lemma \ref{lem:iteration_well_defined}, $\N{M_{j+1} - M_j} \rightarrow 0$ follows.
\end{proof}

\noindent
It remains to show that convergent subsequences of $(M_j)_{j \in \bbN}$ converge to fixed points
of $\Fiter$. Then by \eqref{eq:stopping_condition}, Lemma \ref{lem:equivalency_to_characterization}, and Lemma \ref{lem:equivalent_characterization_lemma},
fixed points satisfy the first order optimality condition \eqref{eq:first_order_optimality},
and are stationary points of \eqref{eq:nonlinear_program}. To prove convergence of subsequences
to fixed points, we require continuity of $\Fiter$. The following Lemma
shows that $\Fiter$ is continuous for matrices $X$ satisfying $\lambda_1(X) > 1/\sqrt{2}$, \emph{i.e.}, if
the largest eigenvector is isolated and $u_1(X)$ is a continuous function of $X$.
\begin{lem}
\label{lem:properties_of_F}
Let $\gamma > 0$, $\epsilon > 0$ arbitrary, and define $\Mspaceiter := \{M \in \hat \CW \cap \bbS : \lambda(M) \geq (\frac{1}{2} + \epsilon)^{-1/2}\}$.
Then $\Fiter(X) \in \CM_{\epsilon}$ for all $X \in \Mspaceiter$, and
$\Fiter$ is $\N{\cdot}_F$-Lipschitz continuous,  with Lipschitz constant $(1 + \gamma/\epsilon)$.
\end{lem}
\begin{proof}
$\Fiter(X) \in \Mspaceiter$ follows directly from Lemma \ref{lem:iteration_and_stopping_condition}, \emph{i.e.}, from the fact that
the largest eigenvalue is only increased by applying $\Fiter$. For the continuity, consider
$X,Y \in \Mspaceiter$. We first note that by using \cite[Theorem 7.3.1]{bhatia2013matrix}
and $\lambda_i(Y) \leq \sqrt{1/2 - \varepsilon}$ for $i=2,\ldots,m_0$ we get
\begin{align*}
&\N{X+\gamma P_{\hat \CW}(u_1(X)\otimes u_1(X)) - Y-\gamma P_{\hat \CW}(u_1(Y)\otimes u_1(Y))}_F\\
&\leq \N{X-Y}_F + \gamma\N{u_1(X)\otimes u_1(X) - u_1(Y)\otimes u_1(Y)}_F\\
&\leq  \N{X-Y}_F + \gamma\frac{\N{X - Y}_F}{\sqrt{\frac{1}{2} + \varepsilon} - \sqrt{\frac{1}{2} - \epsilon}}
\leq \left(1 + \frac{\gamma}{\epsilon}\right)\N{X - Y}_F.
\end{align*}
Furthermore, we have $\N{X + \gamma P_{\hat \CW}(u_1(X)\otimes u_1(X))}_F^2 \geq 1$ according to Lemma \ref{lem:iteration_well_defined},
and therefore $P_{\bbS}$ acts on $X + \gamma P_{\hat \CW}(u_1(X)\otimes u_1(X))$ and
$Y + \gamma P_{\hat \CW}(u_1(Y)\otimes u_1(Y))$ as a projection onto the convex set $\{X : \N{X}_F \leq 1\}$.
Therefore it acts as a contraction
and the result follows from
\begin{align*}
\N{\Fiter(X) - \Fiter(Y)}_F \leq \N{X+\gamma P_{\hat \CW}(u_1(X)\otimes u_1(X)) - Y-\gamma P_{\hat \CW}(u_1(Y)\otimes u_1(Y))}_F.
\end{align*}
\end{proof}
\noindent
The convergence to fixed points of any subsequence of $(M_j)_{j \in \bbN}$
now follows as a corollary of Lemma \ref{lem:technical_lemma_subsequences} in the Appendix.
\begin{thm}
\label{cor:limit_behavior}
Let $\epsilon > 0$, $\gamma > 0$, $M_0 \in \hat \CW \cap \bbS$ with $\lambda(M_0) \geq 1/\sqrt{2} + \varepsilon$ and let $M_{j+1} := \Fiter(M_{j})$ as generated by Algorithm \ref{alg:approximation_neural_network_profiles}. Then $(M_{j+1})_{j \in \bbN}$
has a convergent subsequence, and any such subsequence converges to a fixed point of $\Fiter$, respectively
a stationary point of \eqref{eq:nonlinear_program}.
\end{thm}
\begin{proof}
By Lemma \ref{lem:properties_of_F} the operator $\Fiter$ is continuous on $\Mspaceiter := \{M \in \hat \CW \cap \bbS : \lambda_1(M) \geq (\frac{1}{2} + \epsilon)^{-1/2}\}$
for any $\epsilon > 0$. Moreover, by Lemma \ref{lem:iteration_and_stopping_condition} we have $(M_{j+1})_{j \in \bbN} \subset \Mspaceiter$,
and by Lemma \ref{lem:convergence_results_1} we have $\N{M_{j+1} - M_j}_F \rightarrow 0$.
Therefore we can apply Lemma \ref{lem:technical_lemma_subsequences} to see that any convergent subsequence converges to a
fixed point of $\Fiter$. Moreover, since $(M_{j+1})_{j \in \bbN}$ is bounded, there exists at least one convergent subsequence by Bolzano-Weierstrass.
Finally, any fixed point $\bar M$ of $F_\gamma$ can be written as $\bar M = \lambda_1(\bar M)P_{\hat \CW}(u_1(\bar M)\otimes u_1(\bar M))$
by Lemma \ref{lem:equivalent_characterization_lemma} and Lemma \ref{lem:iteration_and_stopping_condition}. Since $\lambda_1(\bar M) > 1/\sqrt{2}$,
it is an isolated eigenvalue satisfying $\lambda_1(\bar M) = \N{\bar M}$, and thus $\bar M$
satisfies the first order optimality condition \eqref{eq:first_order_optimality}  of \eqref{eq:nonlinear_program} by Theorem \ref{thm:optimality_conditions}.
\end{proof}
\begin{rem}\label{rem:gap}
The analysis of the convergence of Algorithm \ref{alg:approximation_neural_network_profiles} we provide above does not use the structure of the space $\CW$ and it focuses exclusively on the behavior of the first eigenvalue $\lambda_1$. As a consequence it does guarantee that its iterations have monotonically increasing spectral norm and that they generically converges to stationary points of \eqref{eq:nonlinear_program}. However, it does not ensure convergence to non-spurious, minimal rank local minimizers of \eqref{eq:nonlinear_program}.
{In the numerical experiments of Section \ref{sec:numerical_NNprofiles}, where $\{w_j:j\in [m]\}$ are sampled randomly from certain distributions,  an overwhelming majority of sequences $(M_{j})_{j\in \bbN}$
converges to a near rank-$1$ matrix with an eigenvalue close to one, whose corresponding eigenvector approximates a network profile with good accuracy.}  To explain this success, we would need a finer and quantitative analysis of the increase of the spectal norm during the iterations, for instance by quantifying the gap
\begin{equation}\label{eq:quant}
\left [ \Theta \N{P_{\hat \CW}(u_1(X) \otimes u_1(X))}_F -\lambda_1(X) \right] \geq 0,
\end{equation}
by means of a suitable constant $0<\Theta<1$. As clarified in the proof of Lemma \ref{lem:equivalent_characterization_lemma}, the smaller the constant $\Theta>0$ is, the larger is the increase of the spectral norm $\|M_{j+1} \| > \|M_j\|$ between iterations of the Algorithm \ref{alg:approximation_neural_network_profiles}.
The following result is an attempt to gain  a quantitative estimate for $\Theta$ by injecting more information about the structure of the space $\CW$.
\end{rem}

In order to simplify the analysis, let us assume $\delta=0$ or $\hat \CW= \CW$.

\begin{prop}
Assume that $\{W_\ell:=w_\ell\otimes w_\ell: \ell \in [m] \}$ forms a frame for $\CW$, {\it i.e.}, there exist constants $c_\CW,C_\CW>0$ such that for all $X \in \CW$
\begin{equation}\label{frametens}
c_\CW \|X\|_F^2 \leq \sum_{\ell=1}^m \langle X, w_\ell\otimes w_\ell \rangle_F^2 \leq C_\CW \|X\|_F^2.
\end{equation}
Denote $\{ \tilde W_\ell: \ell \in [m] \}$ the canonical dual frame  so that
\begin{equation}\label{proktens}
P_{\CW} (X) = \sum_{\ell=1}^m \langle X, \tilde W_\ell\rangle_F  W_\ell,
\end{equation}
for any symmetric matrix $X$.
Then, for $X \in \CW$ {and the notation $\lambda_j:=\lambda_j(X)$, $\lambda_1 =\|X\|$ and $u_j:=u_j(X)$,} we have
\begin{eqnarray}
\lambda_1 &=& \|P_{\CW} (u_1 \otimes u_1) \|_F\left ( \sum_{j=1}^{m_0}  \sum_{\ell=1}^m \lambda_j \frac{\langle u_j \otimes u_j, \tilde W_\ell \rangle_F \langle W_\ell, u_1 \otimes u_1 \rangle_F  }{\|P_{\CW} (u_1 \otimes u_1) \|_F} \right ) \nonumber \\
&\leq &\|P_{\CW} (u_1 \otimes u_1) \|_F \underbrace{\left( \frac{C_\CW}{c_\CW}\right)^{1/2} \left (  \sum_{\lambda_j >0}  \lambda_j \|P_{\CW} (u_j \otimes u_j) \|_F \right )}_{:=\Theta} \label{eq:gain}
\end{eqnarray}
\end{prop}

\begin{proof}
Let us fix $X \in \CW$. Then we have two ways of representing $X$, its frame decomposition and its sepectral decomposition:
$$
X= \sum_{\ell=1}^m \langle X, \tilde W_\ell\rangle_F  W_\ell = \sum_{j=1}^{m_0} \lambda_j u_j \otimes u_j.
$$
{
By using both the decompositions and again the notation $W_\ell = w_\ell \otimes w_\ell$ we obtain
\begin{eqnarray*}
\lambda_1 &=& u_1^T X u_1 =\sum_{\ell=1}^m \langle  \sum_{j=1}^{m_0} \lambda_j u_j \otimes u_j, \tilde W_\ell\rangle_F  u_1^T W_\ell u_1 \\
&=& \sum_{\ell=1}^m \langle  \sum_{j=1}^{m_0} \lambda_j u_j \otimes u_j, \tilde W_\ell\rangle_F  \left\langle w_\ell, u_1\right\rangle^2 \\
&=& \sum_{\ell=1}^m\sum_{j=1}^{m_0}\lambda_j \langle    u_j \otimes u_j, \tilde W_\ell\rangle_F  \left\langle w_\ell, u_1\right\rangle^2 \\
&=& \|P_{\CW} (u_1 \otimes u_1) \|_F\left ( \sum_{j=1}^{m_0}  \sum_{\ell=1}^m \lambda_j \frac{\langle u_j \otimes u_j, \tilde W_\ell \rangle_F \langle w_\ell, u_1 \rangle^2  }{\|P_{\CW} (u_1 \otimes u_1) \|_F} \right ).
\end{eqnarray*}}
By observing that  $\sum_{\ell=1}^m\langle u_j \otimes u_j, \tilde W_\ell \rangle_F^2 \leq A^{-1} \| P_{\CW}(u_j \otimes u_j)\|_F^2$ (canonical dual frame upper bound), and using Cauchy-Schwarz inequality we can further estimate
\begin{eqnarray*}
\lambda_1 &\leq& \|P_{\CW} (u_1 \otimes u_1) \|_F \left ( c_\CW^{-1/2} \sum_{\lambda_j >0}  \frac{\lambda_j \|P_{\CW} (u_j \otimes u_j) \|_F}{\|P_{\CW} (u_1 \otimes u_1) \|_F} \right) \left ( \sum_{\ell=1}^m\langle w_\ell, u_1 \rangle^4 \right )^{1/2}  \\
&\leq& \|P_{\CW} (u_1 \otimes u_1) \|_F \left( \frac{C_\CW}{c_\CW}\right)^{1/2} \left (  \sum_{\lambda_j >0}  \lambda_j \|P_{\CW} (u_j \otimes u_j) \|_F \right ),
\end{eqnarray*}
where in the last inequality we applied the estimates
\begin{eqnarray*}
 \sum_{\ell=1}^m\langle w_\ell, u_1 \rangle^4   &=&  \sum_{\ell=1}^m\langle w_\ell \otimes w_\ell, u_1 \otimes u_1 \rangle_F^2 \\
 &\leq& C_\CW \|P_{\CW}(u_1 \otimes u_1)\|_F^2.
\end{eqnarray*}
\end{proof}
The meaning of estimate \eqref{eq:gain} is explained by the following mechanism: whenever  the deviation of an iteration $M_j$ of Algorithm \ref{alg:approximation_neural_network_profiles} from being a rank-$1$ matrix in $\CW$ is large, in the sense that $\|P_{\CW} (u_1 \otimes u_1) \|_F$ is small, then the constant $\Theta = \left( \frac{C_\CW}{c_\CW}\right)^{1/2} \left (  \sum_{\lambda_j >0}  \lambda_j \|P_{\CW} (u_j \otimes u_j) \|_F \right )$ is also small and the iteration $M_{j+1} = F_\gamma (M_j)$ will efficiently increase the spectral norm. The gain will reduce as soon as the iteration $M_j$ gets closer and closer to a rank-$1$ matrix. It would be perhaps possible to get an even more precise analysis of the behavior of Algorithm \ref{alg:approximation_neural_network_profiles}, by considering simultaneously the dynamics of (the gaps between) different eigenvalues (not only focusing on $\lambda_1$). Unfortunately, we could not find yet a proper and conclusive argument.
%


\section{Numerical experiments about the recovery of network profiles}
\label{sec:numerical_NNprofiles}
\begin{figure}[t]
\centering
\begin{subfigure}{.5\textwidth}
  \centering
  \includegraphics[width=1.0\linewidth]{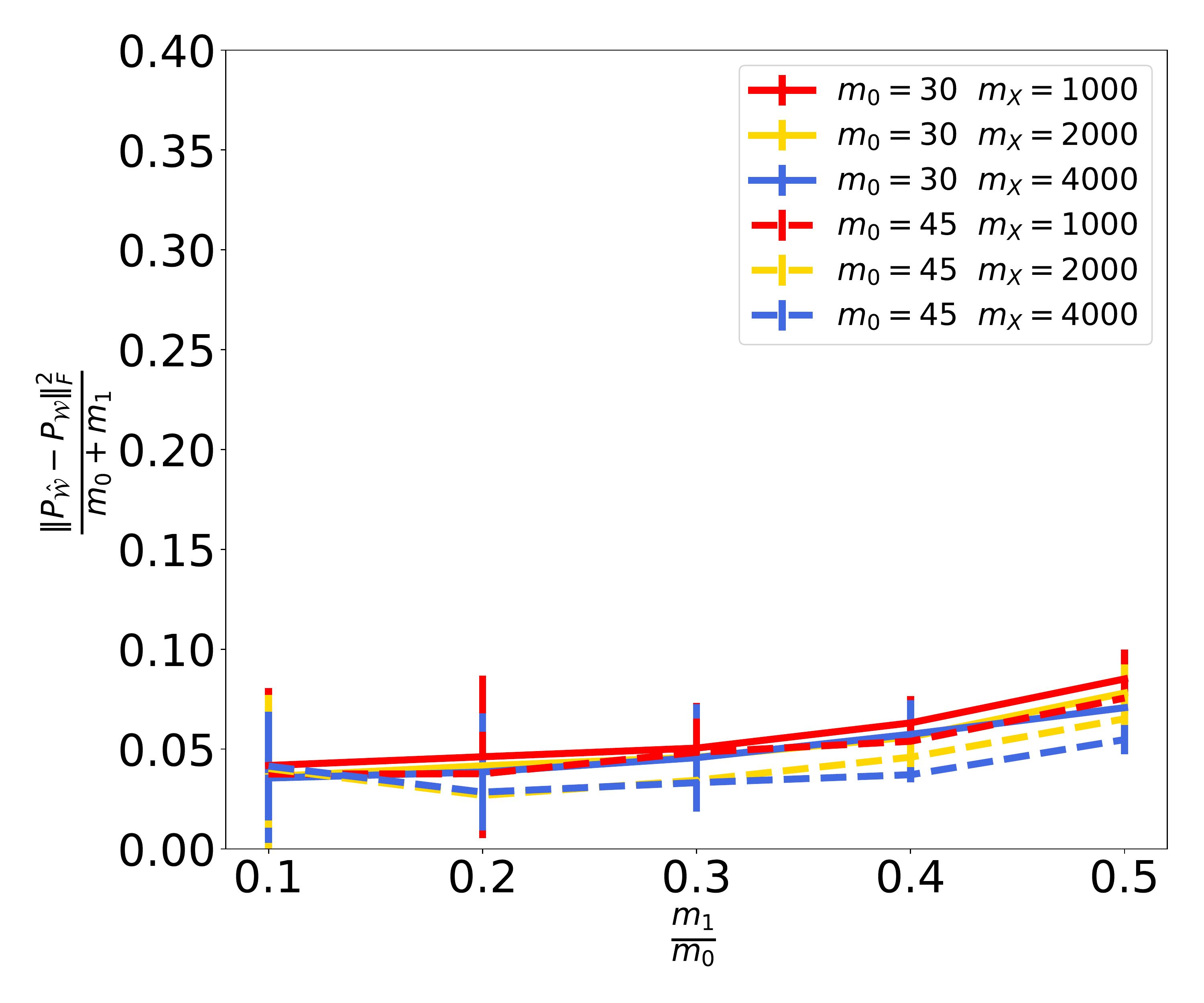}
  \caption{Sigmoid activation function}
  \label{fig:proj_error_sig_qo}
\end{subfigure}%
\begin{subfigure}{.5\textwidth}
  \centering
  \includegraphics[width=1.0\linewidth]{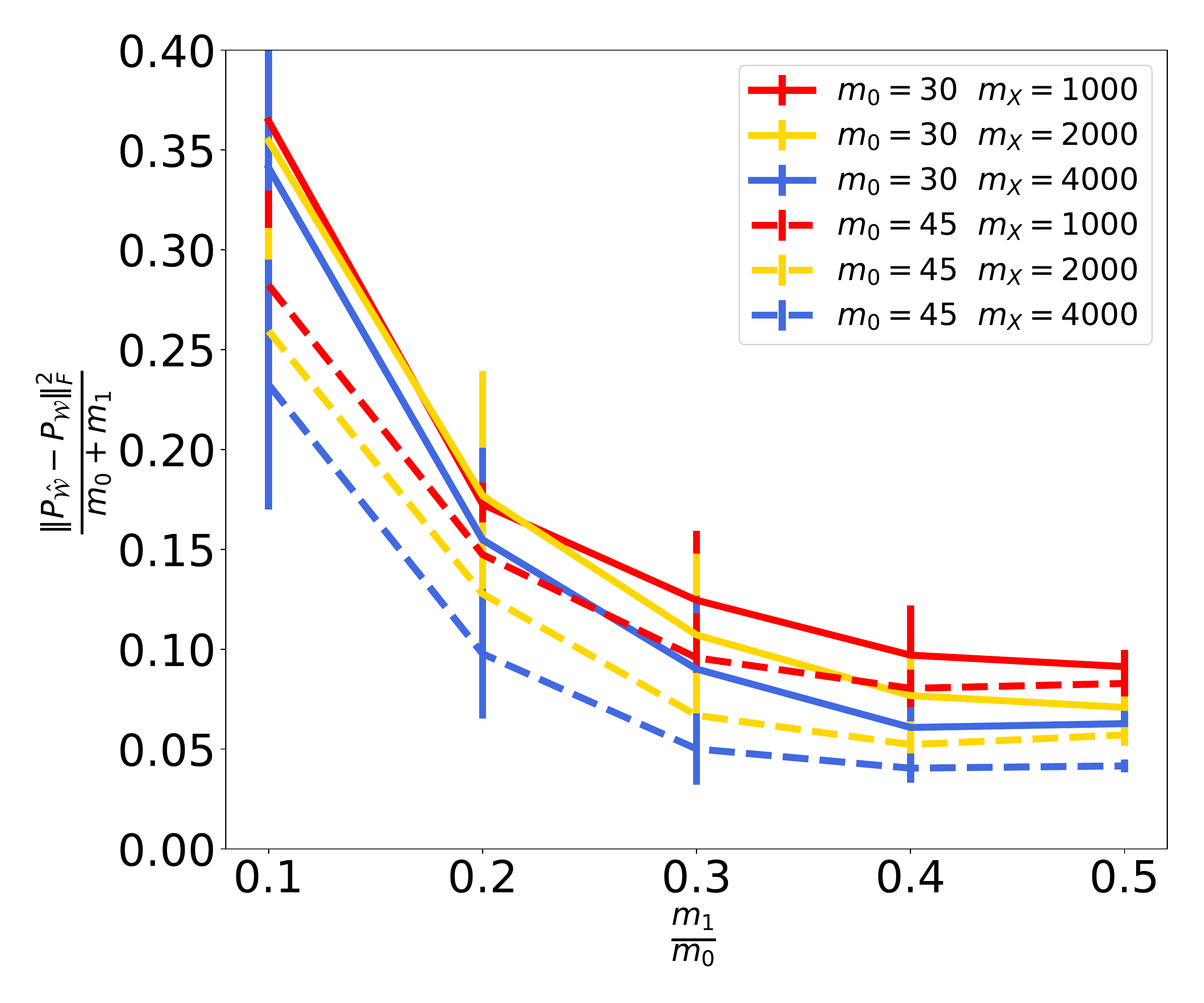}
  \caption{$\tanh$ activation function}
  \label{fig:proj_error_tan_qo}
\end{subfigure}
\caption{Error in approximating $\CW$ for perturbed orthogonal weights and
different activation functions.
}
\label{fig:proj_error_qo}
\end{figure}
In this section we present numerical experiments about the recovery of  network weights
$\{a_i : i \in [\Ninnerlayer]\}$ and $\{v_\ell: \ell\in [\Nouterlayer]\}$ from few point
queries of the network.
The recovery procedure leverages the theoretical insights that have been provided in previous sections. Without much loss of generality,  we neglect the active subspace reduction
and focus on the case $d = \Ninnerlayer$.
We construct an approximation $P_{\hat \CW} \approx P_{\CW}$ using Algorithm
\ref{alg:approximation_matrix_space}. Then we randomly generate a number of matrices
$\{M^k_0: k \in [K]\} \in \hat \CW \cap \bbS$, and compute the sequences
$M^k_{j+1} = F_{\gamma}(M_j^k)$ as in Algorithm \ref{alg:approximation_neural_network_profiles}.
For each limiting matrix $\{M_{\infty}^{k}: k \in [K]\}$,
we compute the largest eigenvector $u_1(M_{\infty}^{k})$, and then
cluster $\{u_1(M_{\infty}^{k}): k \in [K]\}$ into $m=\Ninnerlayer + \Nouterlayer$ classes using kMeans++.
After projecting the resulting cluster centers onto $\bbS^{d-1}$, we obtain vectors
$\{\hat w_j: j \in [\Ninnerlayer + \Nouterlayer]\}$ that are used as approximations to  $\{a_i: i\in [\Ninnerlayer]\}$
and $\{v_\ell: \ell \in[\Nouterlayer]\}$.

We perform experiments for different scenarios, where either the activation function,
or the construction of the network weights varies. Guided by our theoretical results,
we pay particular attention to how the network architecture, \emph{e.g.}, $\Ninnerlayer$ and $\Nouterlayer$,
influences the simulation results. The entire procedure is rather flexible and can be adjusted in different ways, {\it e.g.} changing the
distribution $\mu_X$. To provide a fair account of the success, we fix
hyperparameters of the approach throughout all experiments. Test scenarios, hyperparameters,
and error measures are reported below in more detail. Afterwards, we present and discuss the results.
\begin{figure}
\centering
\begin{subfigure}{.5\textwidth}
  \centering
  \includegraphics[width=0.75\linewidth]{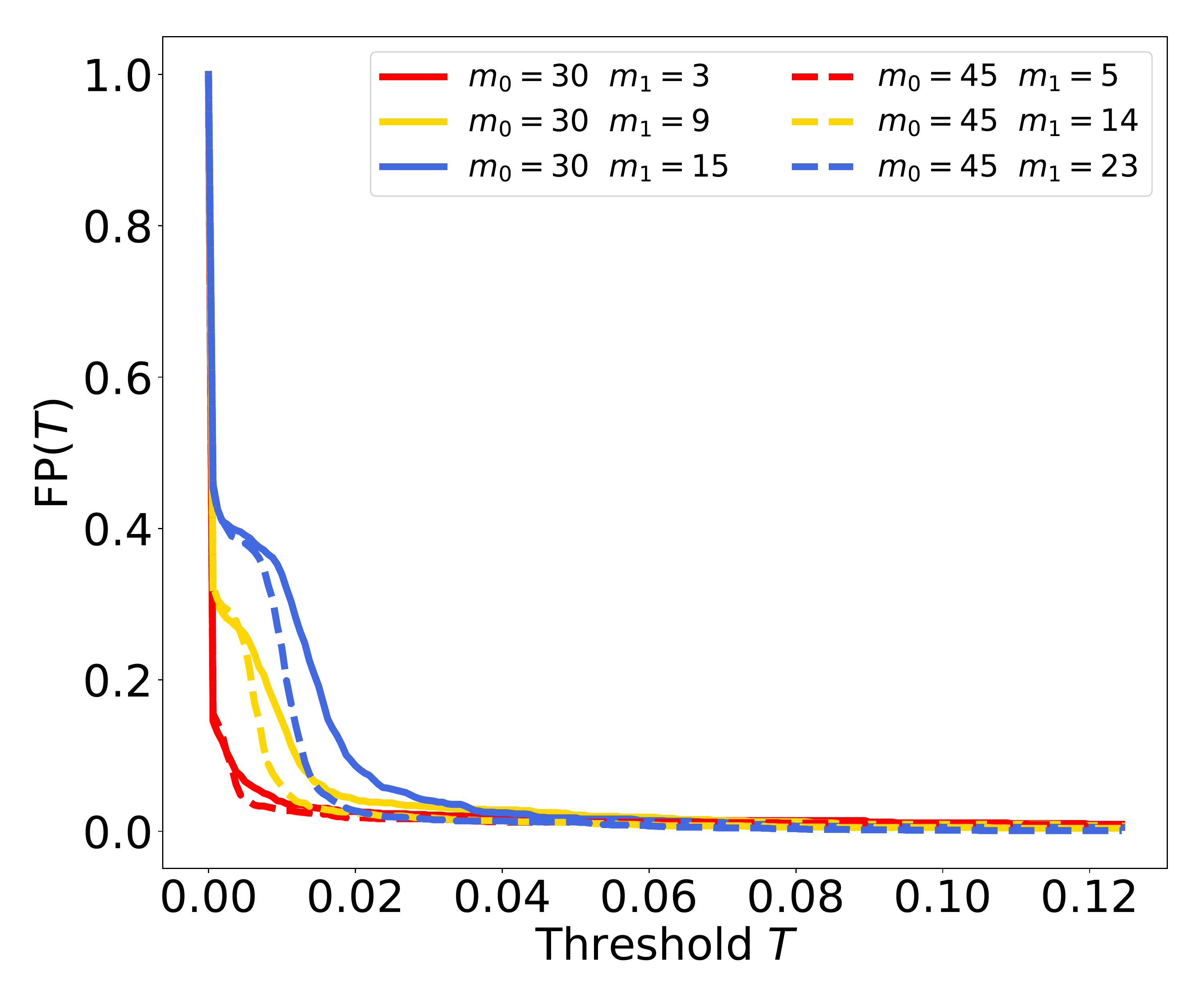}
  \caption{False positives, sigmoid activation function}
  \label{fig:fp_sig_qo}
\end{subfigure}%
\begin{subfigure}{.5\textwidth}
  \centering
  \includegraphics[width=0.75\linewidth]{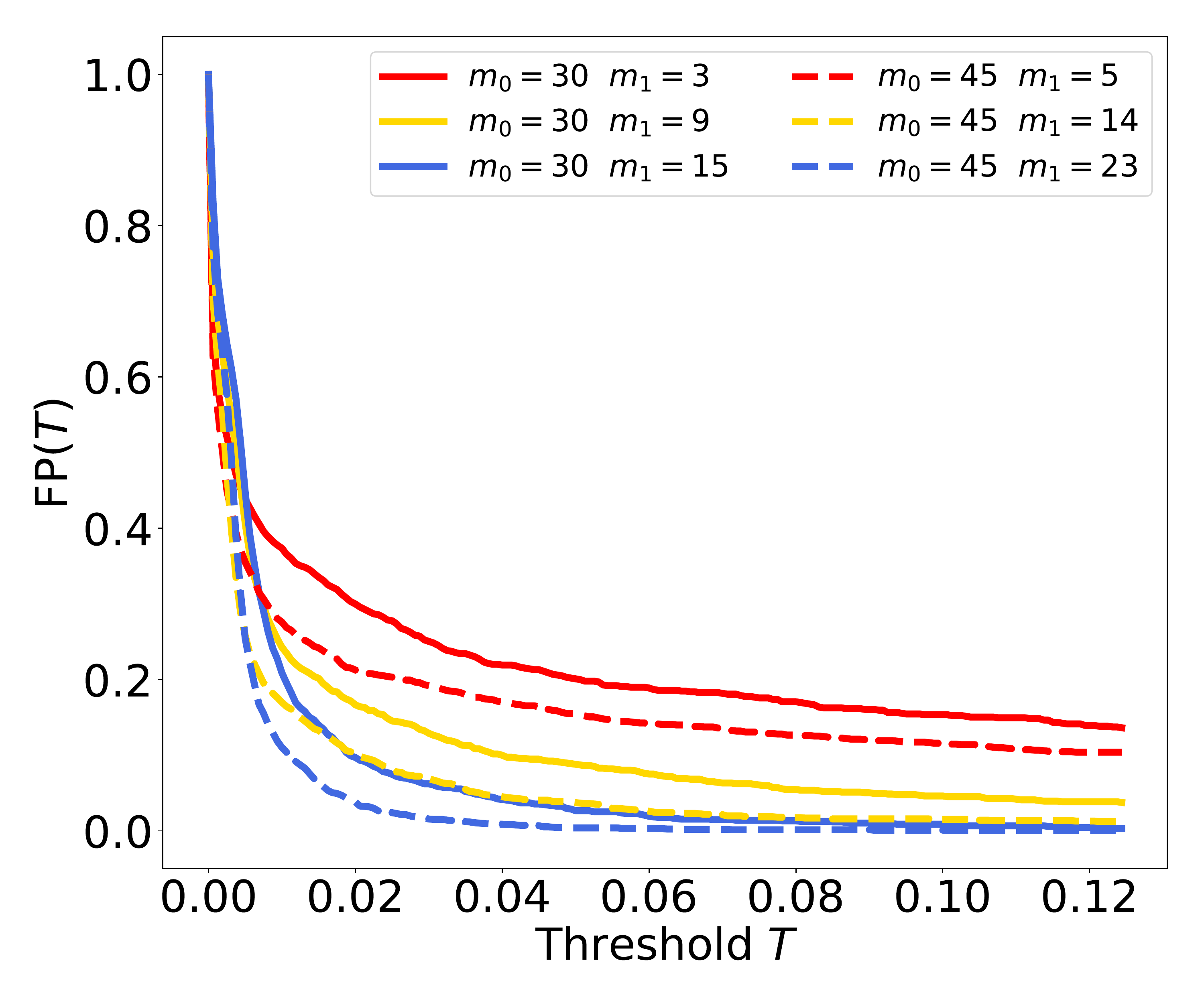}
  \caption{False positives, $\tanh$ activation function}
  \label{fig:fp_tanh_qo}
\end{subfigure}
\begin{subfigure}{.5\textwidth}
  \centering
  \includegraphics[width=1.0\linewidth]{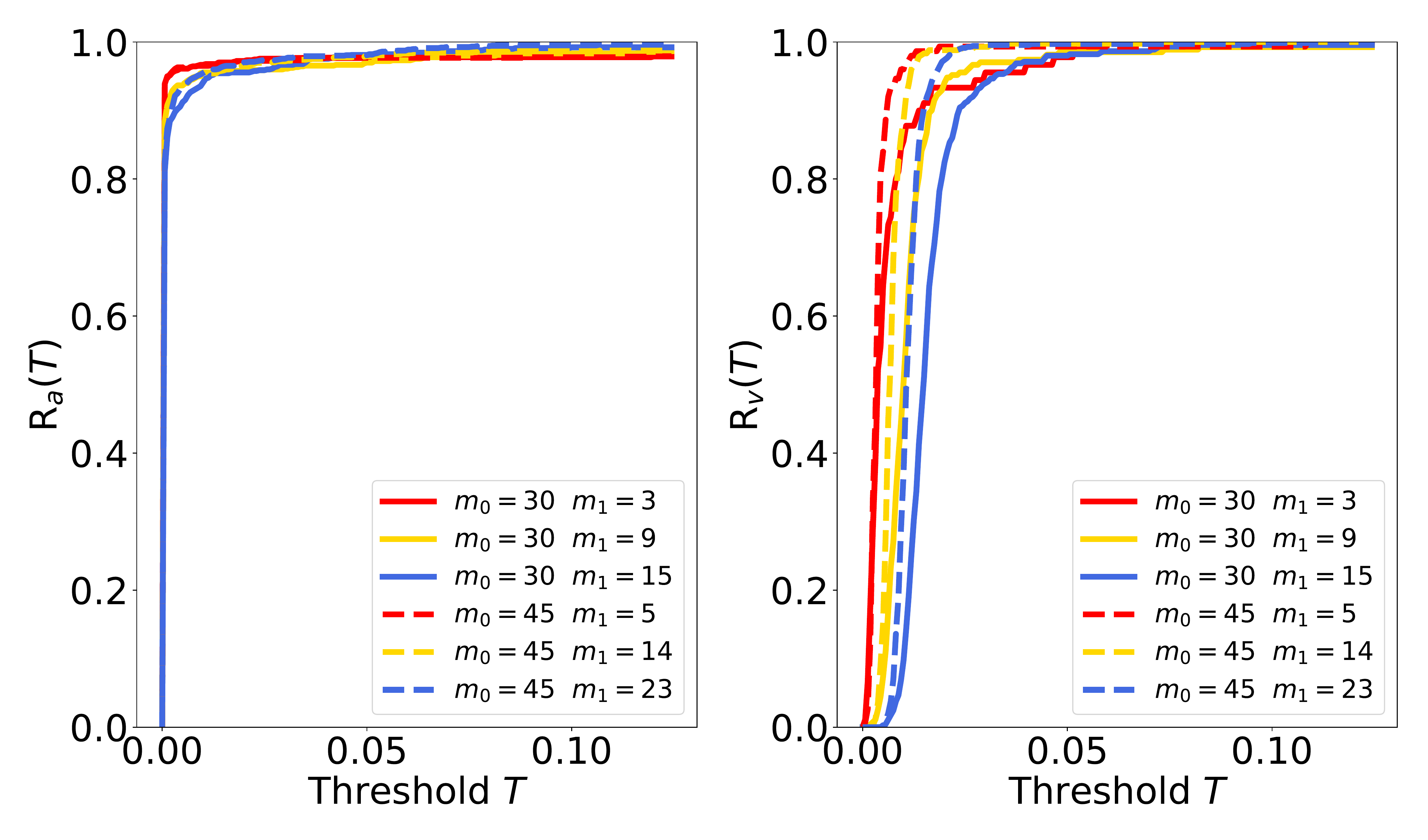}
  \caption{Recovery rates, sigmoid activation function }
  \label{fig:recov_sig_qo}
\end{subfigure}%
\begin{subfigure}{.5\textwidth}
  \centering
  \includegraphics[width=1.0\linewidth]{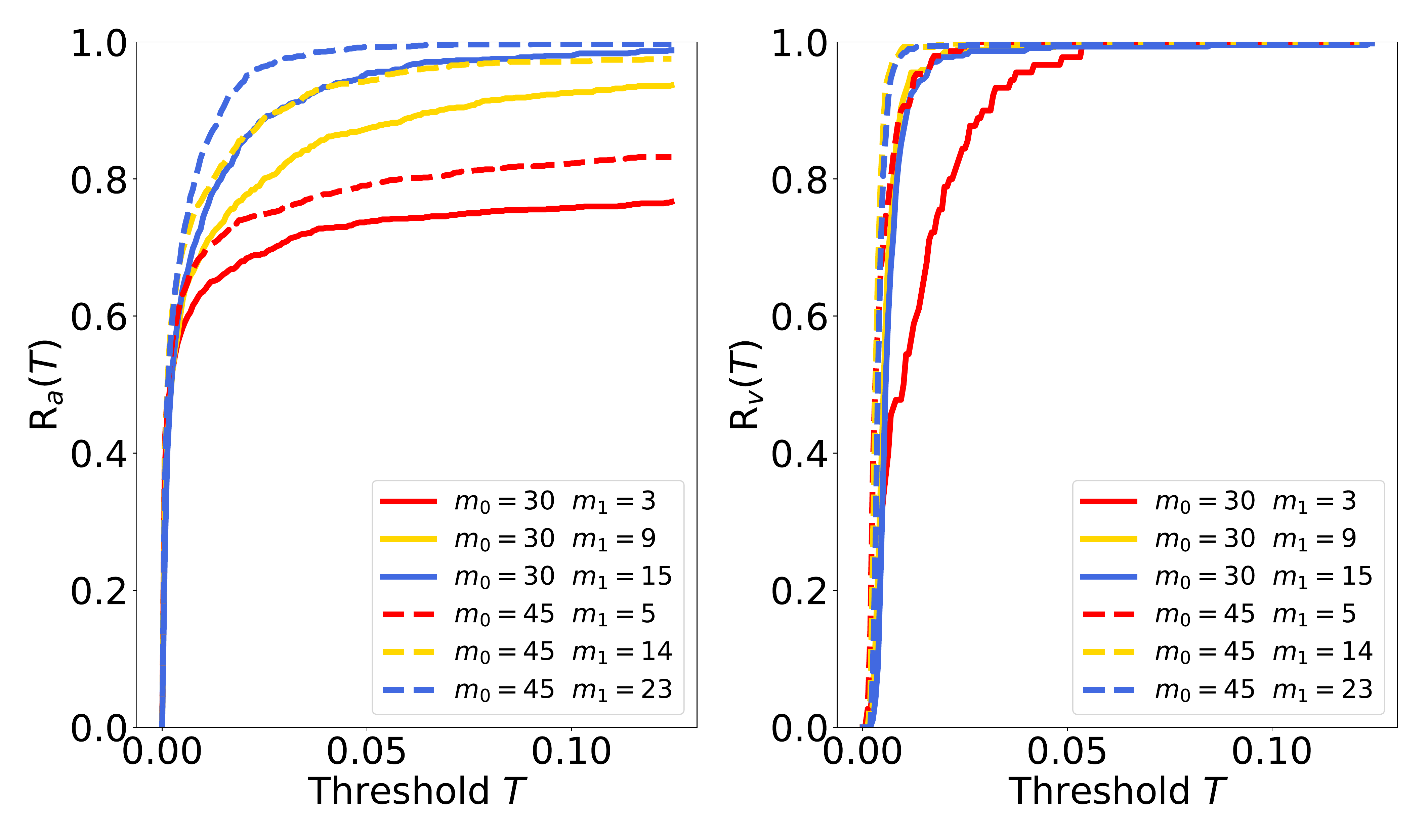}
  \caption{Recovery  rates, $\tanh$ activation function}
  \label{fig:recov_tanh_qo}
\end{subfigure}
\caption{False positive and recovery rates for perturbed orthogonal weights and for different activation functions.
}
\label{fig:fp_recov_qo}
\end{figure}
\paragraph{Scenarios and construction of the networks}
The network is constructed by choosing activation functions and network weights
$\{a_i : i \in [\Ninnerlayer]\}$, $\{b_\ell : \ell \in  [\Nouterlayer]\}$, for which $v_\ell$ is then
defined via $v_\ell = \frac{AG_0 b_\ell}{\N{AG_0 b_\ell}_2}$, see Definition \ref{def:f}. To construct activation functions,
we set $g_i(t) = \phi(t + \theta_i)$ for $i \in [\Ninnerlayer]$, and $h_\ell(t) = \phi(t + \tau_\ell)$
for $\ell \in [\Nouterlayer]$. We choose either $\phi(t) = \tanh(t)$, or
$\phi(t) = \frac{1}{1+e^{-t}} - \frac{1}{2}$ (shifted sigmoid function), and sample offsets (called also biases) $\theta_i$, $\tau_\ell$ independently at random
from $\CN(0,0.01)$.

 As made clear by our theory, see Theorem \ref{thm:lower_bound_eigenvalue}, a sufficient condition for successful recovery of the entangled 
weights is $\Ferror =C_F-1$ to be small, where $C_F$ is the upper frame constant of the entangled weights as in  Definition \ref{def:f}. In the following numerical experiments we wish to verify how crucial is this requirement.
Thus, we test two different scenarios for the weights. The first scenario, which is designed to best fulfill the  sufficient condition $\Ferror\approx 0$,  models both $\{a_i : i \in [\Ninnerlayer]\}$ and $\{b_\ell : \ell \in  [\Nouterlayer]\}$
as perturbed orthogonal systems. For their construction, we first sample orthogonal bases uniformly at random,
and then apply a random perturbation. The perturbation is such that $(\sum_{i=1}^{\Ninnerlayer}(\sigma_i(A) - 1)^2)^{-1/2} \approx (\sum_{i=1}^{\Nouterlayer}(\sigma_i(B) - 1)^2)^{-1/2} \approx  0.3$, where $\sigma_i(A)$ and $\sigma_i(B)$ denote singular values of $A$ and $B$. In the second case we sample the (entangled) weights independently from $\Uni{\bbS^{m_0-1}}$. In this situation, as the dimensionality $d=m_0$ is relatively small, the system will likely not fulfill well the condition $\Ferror\approx 0$; however, as the dimension $d=m_0$ is choosen larger, the weights tend to be more incoherent and gradually approaching the previous scenario.
\paragraph{Hyperparameters}
Unless stated differently, we sample $m_X = 1000$ Hessian locations from $\mu_X = \sqrt{\Ninnerlayer}\textrm{Uni}\left(\bbS^{\Ninnerlayer-1}\right)$,
and use $\epsilon = 10^{-5}$ in the finite difference approximation \eqref{def:finite_differences_hessian}.
We generate $1000$ random matrices $\{M_0^k: k \in [1000]\}$ by sampling $m^k \sim \CN(0,\Id_{\Ninnerlayer + \Nouterlayer})$,
and by defining $M_0^k := \bbP_{\bbS}(\sum_{i=1}^{\Ninnerlayer + \Nouterlayer}m_i^k u_i)$, where the $u_i$'s
are as in Algorithm \ref{alg:approximation_matrix_space}. The constant $\gamma = 2$
{is used in the definition of $\Fiter$, and the iteration is stopped if $\lambda_1(M_{j+1}^k) - \lambda_1(M_{j}^k) < 10^{-5}$, or after 200 iterations.
kMeans++ is run with default settings using \textit{sklearn}. All reported results are
averages over 30 repetitions.}


\paragraph{Error measures}
Three error measures are reported:
\begin{itemize}
\item the normalized projection error $\frac{\Vert \hat P_{\CW} - P_{\CW}\Vert_F^2}{\Ninnerlayer + \Nouterlayer}$,
\item a false positive rate $\textrm{FP}(T)=\frac{\#\{j : E(\hat w_j) > T\}}{\Ninnerlayer + \Nouterlayer}$,
where $T>0$ is a threshold, and $E(u)$ is defined by,
\begin{align*}
E(u) := \min_{w \in \{\pm a_i,\pm v_\ell : i \in [\Ninnerlayer],\ell \in [\Nouterlayer]\}}\N{u - w}^2_2,
\end{align*}
\item recovery rate $\textrm{R}_a(T) = \frac{\#\{i : \CE(a_i) < T\}}{\Ninnerlayer}$, and $\textrm{R}_v(T) = \frac{\#\{\ell: \CE(v_\ell) < T\}}{\Nouterlayer}$, where
{
\begin{align*}
\CE(u) &:= \min_{w \in \{\pm \hat  w_j : j \in [\Ninnerlayer+\Nouterlayer] \}}\N{u - w}^2_2.
\end{align*}}
\end{itemize}
\paragraph{Results for perturbed orthogonal weights}
\label{subsec:exp_perturbed_orthogonal_weights}
\begin{figure}[t]
\centering
\begin{subfigure}{.5\textwidth}
  \centering
  \includegraphics[width=1.0\linewidth]{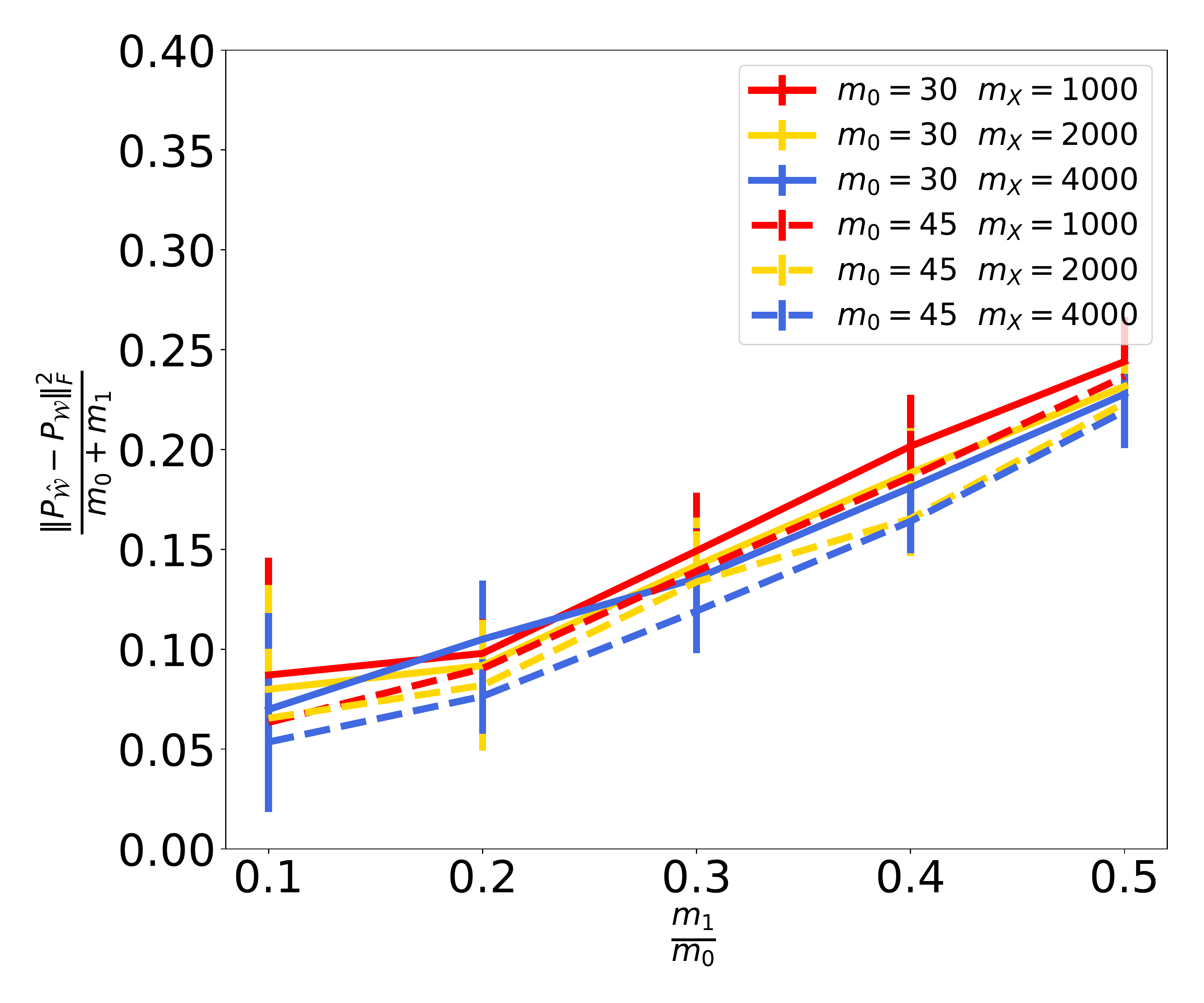}
  \caption{Sigmoid activation function}
  \label{fig:proj_error_sig_uni}
\end{subfigure}%
\begin{subfigure}{.5\textwidth}
  \centering
  \includegraphics[width=1.0\linewidth]{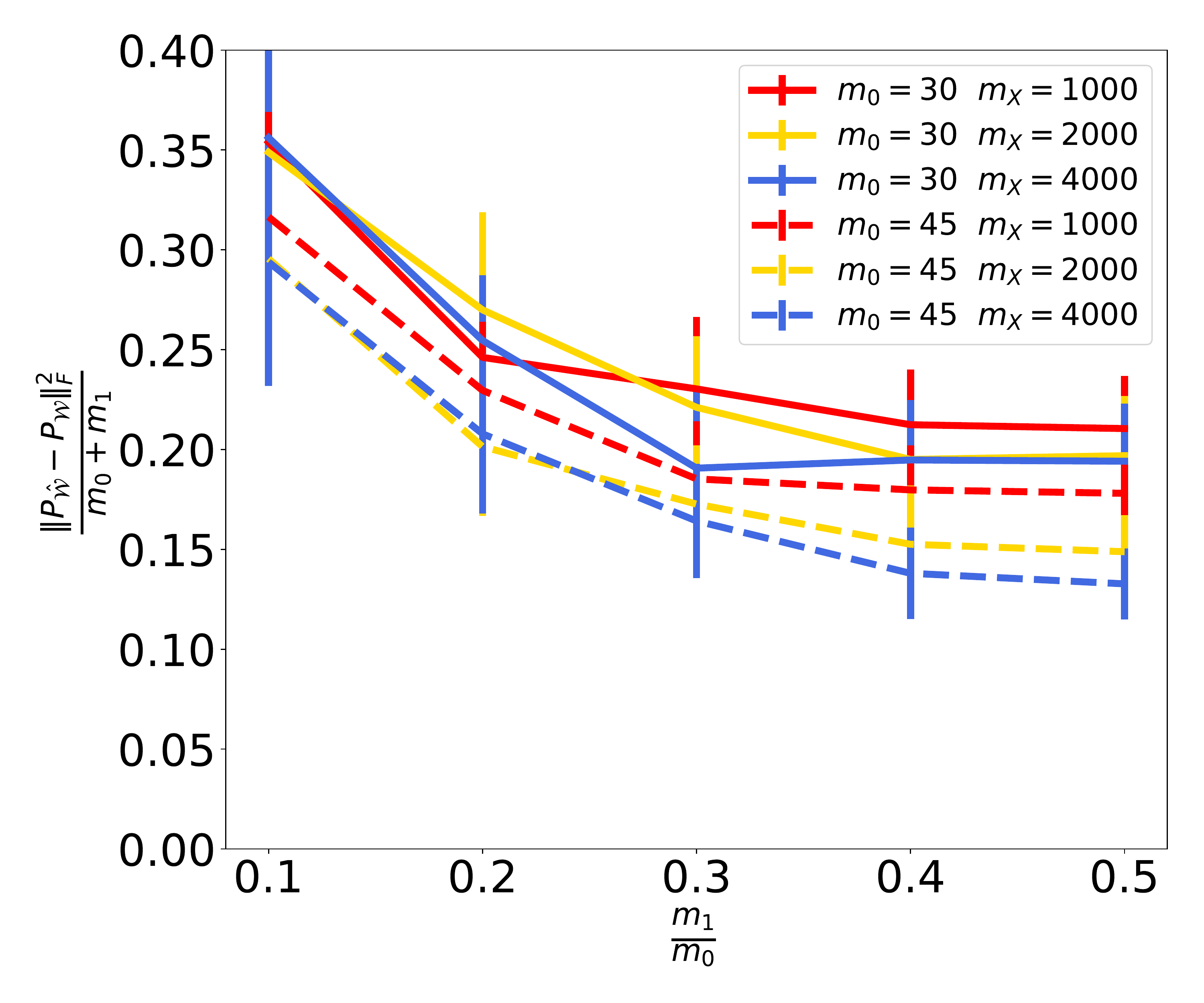}
  \caption{$\tanh$ activation function}
  \label{fig:proj_error_tan_uni}
\end{subfigure}
\caption{Error in approximating $\CW$ for weights sampled independently from the unit sphere, and for different
activation functions.
}
\label{fig:proj_error_uni}
\end{figure}
\begin{figure}
\centering
\begin{subfigure}{.5\textwidth}
  \centering
  \includegraphics[width=0.75\linewidth]{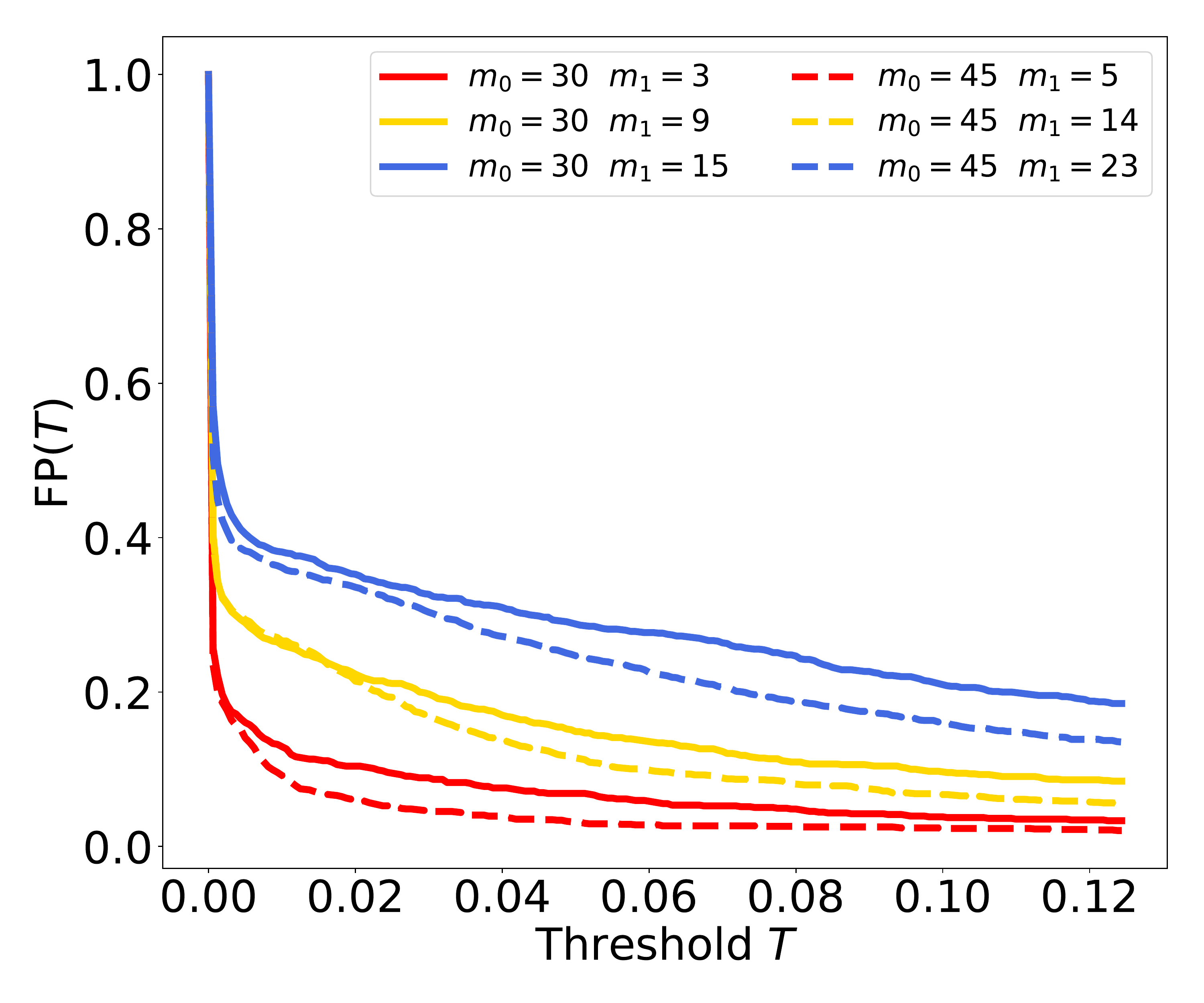}
  \caption{False positives, sigmoid activation function}
  \label{fig:fp_sig_uni}
\end{subfigure}%
\begin{subfigure}{.5\textwidth}
  \centering
  \includegraphics[width=0.75\linewidth]{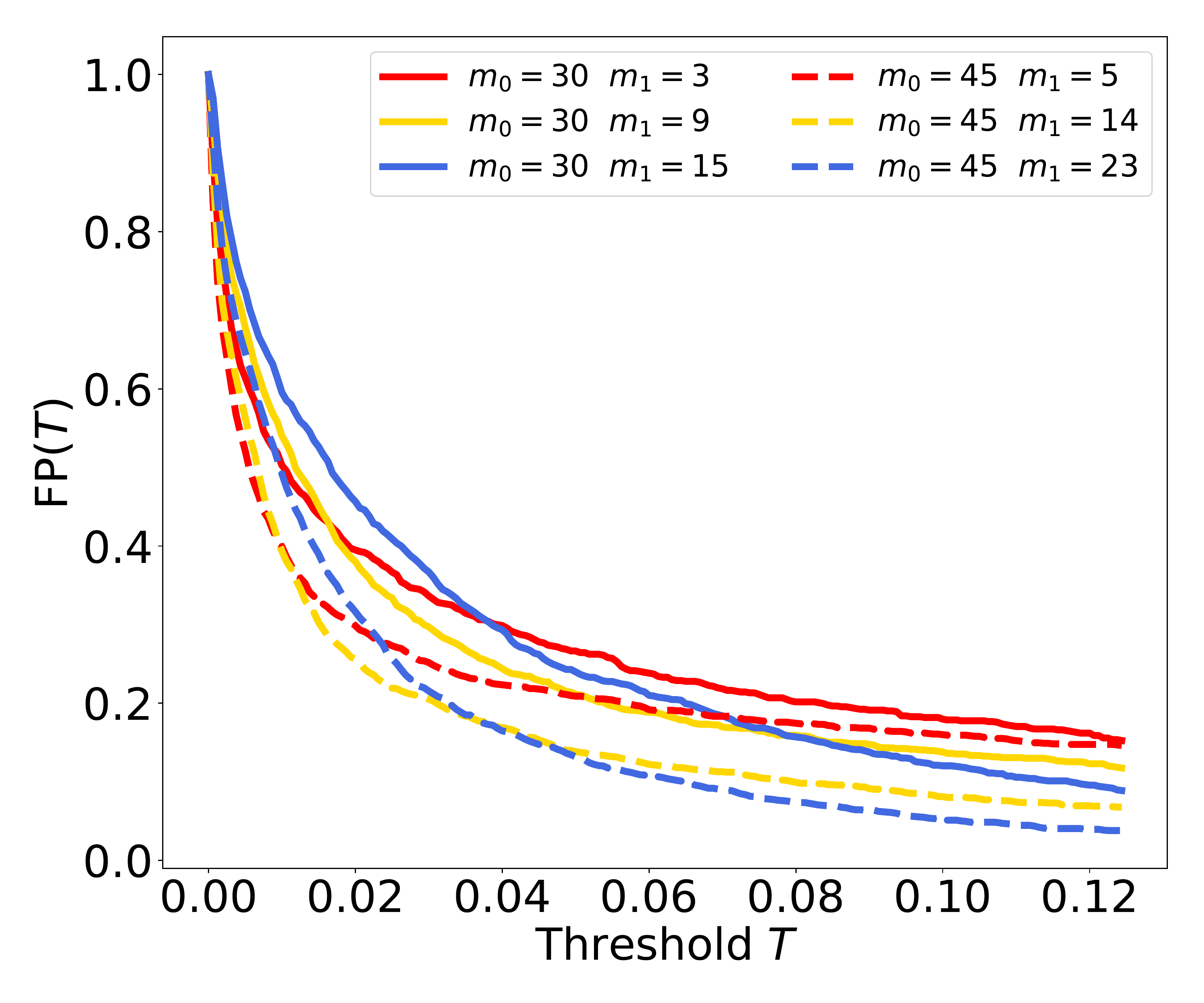}
  \caption{False positives, $\tanh$ activation function}
  \label{fig:fp_tanh_uni}
\end{subfigure}
\begin{subfigure}{.5\textwidth}
  \centering
  \includegraphics[width=1.0\linewidth]{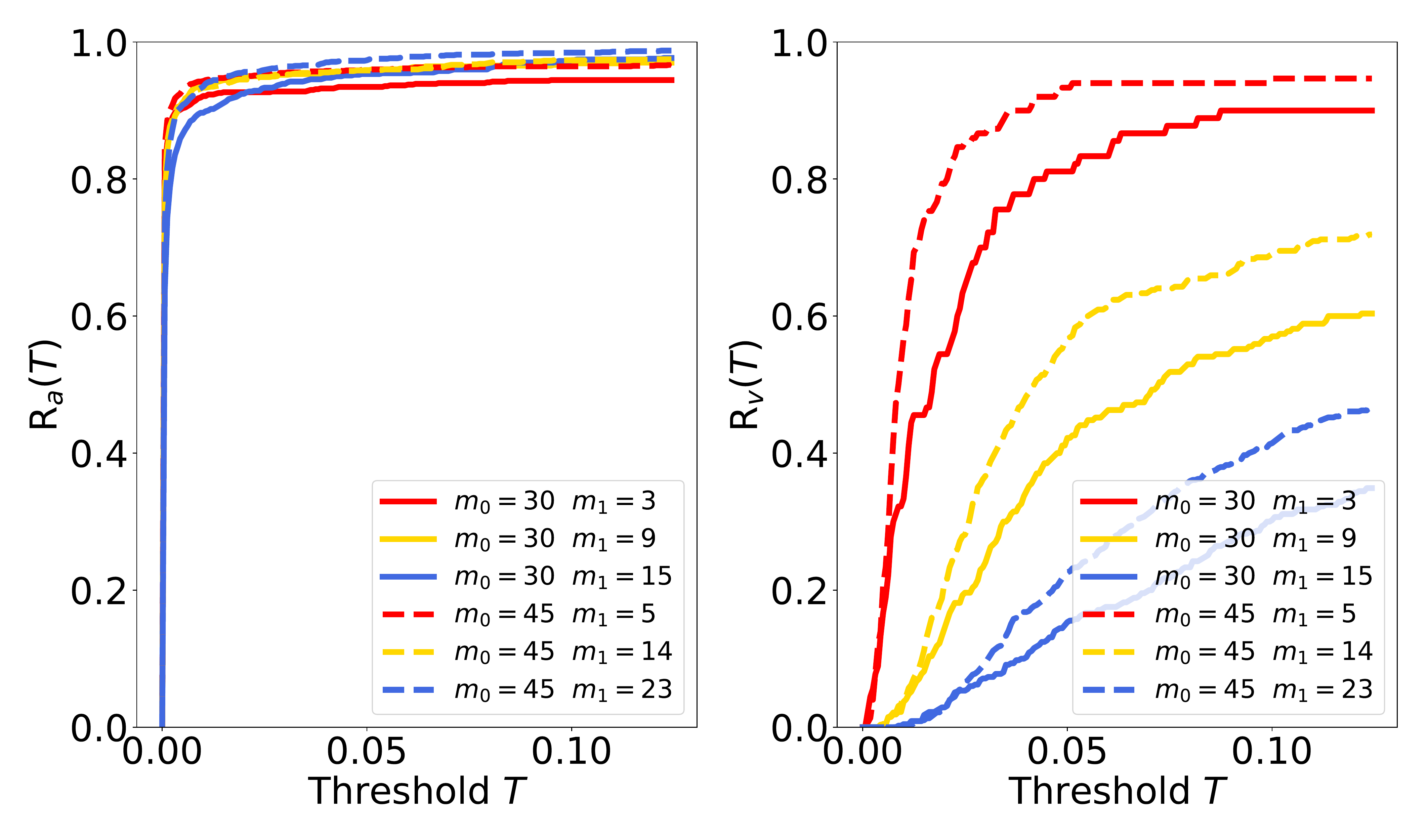}
  \caption{Recovery rates, sigmoid activation function}
  \label{fig:recov_sig_uni}
\end{subfigure}%
\begin{subfigure}{.5\textwidth}
  \centering
  \includegraphics[width=1.0\linewidth]{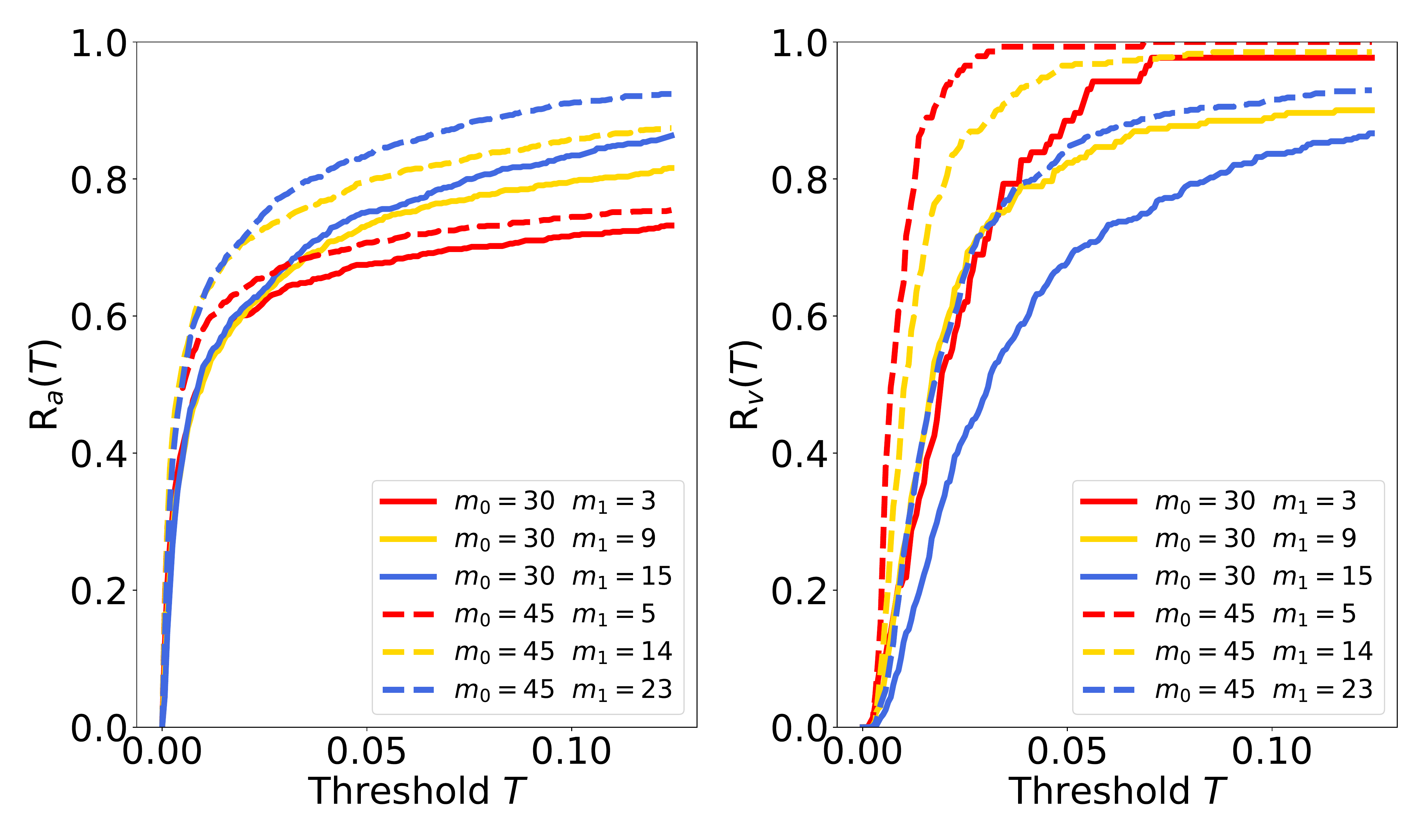}
  \caption{Recovery rates, $\tanh$ activation function}
  \label{fig:recov_tanh_uni}
\end{subfigure}
\caption{False positive and recovery rates for weights sampled uniformly at random from the unit-sphere
and for different activation functions.
}
\label{fig:fp_recov_uni}
\end{figure}
The results of the study are presented in Figures \ref{fig:proj_error_qo}
and \ref{fig:fp_recov_qo}, and show that our procedure typically recovers
many of the network weights, while suffering only few false positives.
Considering for example a sigmoidal network,
we have almost perfect recovery of the weights in both
layers at a threshold of $T = 0.05$ for any network architecture, see Figures \ref{fig:fp_sig_qo}, \ref{fig:recov_sig_qo}. For
a $\tanh$-network, the performance is slightly worse, but we still recover
most weights in the second layer, and a large portion in the
first layer at a reasonable threshold, see Figures \ref{fig:fp_tanh_qo}, \ref{fig:recov_tanh_qo}.

Inspecting the plots more closely, we can notice some shared trends and differences between
sigmoid- and $\tanh$-networks. In both cases, the
performance improves when increasing the input dimensionality or, equivalently, the number of neurons in the first layer, even though the
number of weights that need to be recovered increases accordingly. This is particularly the case
for $\tanh$-networks as visualized in Figures \ref{fig:fp_tanh_qo} and \ref{fig:recov_tanh_qo},
and is most likely caused by reduced correlation of the weights in higher dimensions. As previously mentioned, the correlation
is encoded within the constant $\nu=C_F-1$ used in the  analysis in Section \ref{sec:individual_and_assignment}.

For fixed $\Ninnerlayer$ on the other hand, different activation functions react differently to changes of $\Nouterlayer$.
For $\Nouterlayer$ larger, considering a sigmoid network, the projection error increases,
and the recovery of weights in the second layer worsens as shown in Figures \ref{fig:proj_error_sig_qo} and \ref{fig:recov_sig_qo}. This is
expected by Theorem \ref{thm:approx_space_bound}.
Inspecting the results for $\tanh$-networks, the projection
error actually decreases when increasing $\Nouterlayer$, see Figure \ref{fig:proj_error_tan_qo},
and the recovery performance gets better. Figure \ref{fig:recov_tanh_qo} shows
that especially weights in the first layer are more easily recovered if
$\Nouterlayer$ is large, such that the case $\Ninnerlayer = 45$, $\Nouterlayer = 23$ allows for
perfect recovery at a threshold $T = 0.05$. This behavior can not be fully explained by our general
theory, e.g. Theorem \ref{thm:approx_space_bound}.
\paragraph{Results for random weights from the unit-sphere.}
When sampling the weights independently from the unit-sphere, the recovery problem seems more challenging for moderate dimension $d=\Ninnerlayer$ and
for both activation functions. This confirms the expectation that the smallness of $\nu = C_F-1$ is somehow crucial. Figures
\ref{fig:recov_sig_uni} and \ref{fig:recov_tanh_uni} suggest that especially recovering the
weights of the second layer is more difficult than in the perturbed orthogonal case.
Still, we achieve good performance in many cases. For sigmoid networks, Figure \ref{fig:recov_sig_uni} shows that
we always recover most weights in the first layer, and a large portion of weights
in the second layer if $\Nouterlayer/\Ninnerlayer$ is small. Moreover, keeping $\Nouterlayer/\Ninnerlayer$ constant while increasing $\Ninnerlayer$
improves the performance significantly, as we expect from an improved constant $\nu = C_F-1$. Figures \ref{fig:fp_sig_uni}, \ref{fig:recov_sig_uni} show almost perfect recovery for
$\Ninnerlayer = 45,\ \Nouterlayer = 5$, while suffering only few false positives.

For $\tanh$-networks, Figure \ref{fig:recov_tanh_uni} shows that increasing $\Ninnerlayer$
benefits recovery of weights in both layers, while increasing $\Nouterlayer$ benefits
recovery of first layer weights and harms recovery of second layer weights. We still achieve
small false positive rates in Figure \ref{fig:fp_tanh_uni}, and good recovery for $\Ninnerlayer = 45$,
and the trend continues when further increasing $\Ninnerlayer$.

{Finally, a notable difference between the perturbed orthogonal case and the unit-sphere case is the behavior of the projection error $\|P_{\hat \CW}- P_{\CW}\|_F/(\Ninnerlayer+\Nouterlayer)$ for networks with sigmoid activation function. Comparing Figures
\ref{fig:proj_error_sig_qo} and \ref{fig:proj_error_sig_uni}, the dependency of the projection error on $\Nouterlayer$ is stronger when sampling independently from the unit-sphere.
This is explained by Theorem \ref{thm:approx_space_bound} since $\N{B}^2$ is independent of $\Nouterlayer$ in the perturbed orthogonal case, and grows like $\mathcal{O}(\sqrt{\Nouterlayer})$ when sampling from the unit-sphere.}

\section{Open problems}
\label{sec:openproblems}

With the previous theoretical results of Section \ref{sec:individual_and_assignment} and the numerical experiments of Section \ref{sec:numerical_NNprofiles} we show how to reliably recover the entangled weights $\{\hat w_j: j \in [\Ninnerlayer +\Nouterlayer]\} \approx \{w_j: j \in [\Ninnerlayer +\Nouterlayer]\} = \{a_i: i \in [\Ninnerlayer ]\} \cup  \{v_\ell: \ell \in [\Nouterlayer ]\}$. However, some issues remain open. \\
(i) In Theorem \ref{thm:approx_space_bound} the dependency of $\alpha>0$ on the network architecture and on the input distribution $\mu_X$ is left implicit. However, it plays a crucial role for fully estimating the overall sample complexity. \\
(ii) While we could prove that Algorithm \ref{alg:approximation_neural_network_profiles} is increasing the spectral norm of its iterates in $\hat \CW \cap \mathbb S$,  we could not show yet that it converges always to nearly rank-$1$ matrices in $\hat \CW$, despite it is so numerically observed, see also Remark \ref{rem:gap}. We also could not exclude the existence of spurious local minimizers of the nonlinear program \eqref{eq:nonlinear_program}, as stated in Theorem \ref{thm:lower_bound_eigenvalue}. However, we conjecture that there are none or that they are somehow hard to observe numerically.\\
(iii) Obtaining the approximating vectors $\{\hat w_j: j \in [\Ninnerlayer +\Nouterlayer]\} \approx \{w_j: j \in [\Ninnerlayer +\Nouterlayer]\} = \{a_i: i \in [\Ninnerlayer ]\} \cup  \{v_\ell: \ell \in [\Nouterlayer ]\}$ does not suffice to reconstruct the entire network. In fact, it is impossible a priori to know whether $\hat w_j$ approximates one $a_i$ or some other $v_\ell$, up to sign and permutations, and the attribution to the corresponding layer needs to be derived from quering the network.\\
(iv) Once we  obtained, up to sign and permutations, $\{\hat a_i: i \in [\Ninnerlayer ]\}  \approx \{a_i: i \in [\Ninnerlayer ]\}$  and $  \{\hat v_\ell: \ell \in [\Nouterlayer ]\} \approx  \{v_\ell: \ell \in [\Nouterlayer ]\}$ from properly grouping $\{\hat w_j: j \in [\Ninnerlayer +\Nouterlayer]\}$, it would remain to approximate/identify the activations functions $g_i$ and $h_\ell$. In the case where $g_i(\cdot) = \phi(\cdot - \theta_i)$ and $h_\ell(\cdot) = \phi(\cdot - \tau_\ell)$, this would simply mean to be able to identify the shifts $\theta_i$, $i \in [\Ninnerlayer]$, and $\tau_\ell$, $\ell \in [\Nouterlayer]$. Such identification is also crucial for computing the matrix $G_0=\operatorname{diag}(g_i'(0),\dots,g_{\Ninnerlayer}'(0))$ which allows the disentanglement of the weights $b_\ell$ from the weights $A$ and $v_\ell = AG_0 b_\ell /\|AG_0 b_\ell \|_2$. At this point the network is fully reconstructed.\\
(v) The generalization of our approach to networks with more than two hidden layers is clearly the next relevant issue to be considered as a natural development of this work.
\\

While  problems (i) and (ii) seem to be difficult to solve by the methods we used in this paper, we think that problems (iii) and (iv) are  solvable both theoretically and numerically with just a bit more effort. For a self-contained conclusion of this paper, in the following sections we sketch some possible approaches to these issues,  as a glimpse towards future developments, which will be more exhaustively included in \cite{FFR20}. The generalization of our approach to networks  with more than two hidden layers as mentioned in (v) is suprisingly simpler than one may expect, and it is in the course of finalization \cite{FFR20}. For a network
$f(x):= f(x;W_1,\dots,W_L) = 1^T g_L(W_L^T g_{L-1}(W_{L-1}^T  \dots (g_1 (W_1^T  x))\dots)$, with $L>2$,
again by second order differentiation is possible to obtain an approximation space $$\hat\CW \approx \operatorname{span} \{w_ {1, i} \otimes w_ {1, i}, (W_2 G_1 w_ {1, j}) \otimes (W_2 G_1 w_ {1, j}), \dots,   (W_L G_L \dots W_2  G_1 w_ {1, j}) \otimes  (W_L G_L \dots  W_2  G_1 w_ {1, j})\},$$ of the matrix space spanned by the tensors of entangled weights, where $G_i$ are suitable diagonal matrices depending on the activation functions. The tensors $(W_kG_k \dots W_2 G_1 w_ {1, j}) \otimes  (W_k^ TG_k \dots  W_2^ T G_1 w_ {1, j})$  can be again identified by a minimal rank principle. The disentanglement goes again by a layer by layer procedure as in this paper, see also \cite{BLPL06}. 

\section{Reconstruction of the entire network}
\label{sec:reconstruction_entire_network}

In this section we address problems (iii) and (iv) as described in Section \ref{sec:openproblems}.
Our final goal is of course to construct a two-layer network $\hat f$ with number of nodes equaling $\Ninnerlayer$
and $\Nouterlayer$ such that $\hat f \approx f$. Additionally we also
study whether the individual building blocks (e.g. matrices $\hat A$, $\hat B$, and biases in both layers) of $\hat f$
match their corresponding counterparts of $f$.

To construct $\hat f$, we first discuss how recovered entangled  weights
$\{\hat w_{i} : i \in [\Ninnerlayer+\Nouterlayer]\}$ (see Section \ref{sec:numerical_NNprofiles})
can be assigned to either the first, or the second layer, depending on whether $\hat w_{j}$
approximates one of the $a_i$'s, or one of the $v_{\ell}$'s. Afterwards we discuss a modified
gradient descent approach that optimizes the deparametrized network (its entangled weights are known at this point!) 
over the remaining, unknown parameters of the network function, {\it e.g.}, biases $\theta_i$ and $\tau_\ell$. 

\subsection{Distinguishing first and second layer weights}
\label{subsec:distinguishing}
Attributing approximate entangled weights to first or second layer is generally a challenging
task. In fact, even the true weights $\{a_i: i \in [\Ninnerlayer] \}$, $\{v_\ell : \ell \in [\Nouterlayer]\}$
can not be assigned to the correct layer based exclusively on their entries when no additional
a priori information ({\it e.g.}, some distributional assumptions) is available.
Therefore, assigning $\hat w_{j}$,  $j \in [\Ninnerlayer + \Nouterlayer]$ to the correct layer
requires using again the network $f$ itself, and thus to query additional information.

The strategy we sketch here is designed for sigmoidal activation functions and networks
with (perturbed) orthogonal weights in each layer. Sigmoidal functions are monotonic,
have bell-shaped first derivative, and are bounded by two horizontal
asymptotes as the input tends to $\pm \infty$. If
activation functions $\{g_i: i \in [\Ninnerlayer]\}$ and $\{h_\ell: \ell \in [\Nouterlayer]\}$ are translated sigmoidal, their properties imply
\begin{equation}
\label{eq:grad_along_direction}
\N{\nabla f(tw)}_2 = \left( \sum_{i=1}^{\Ninnerlayer} g_i'( t a_i^T w)^2\left(\sum^{\Nouterlayer}_{\ell=1} h_\ell'(b_\ell^T g(tA^T w))b_{i\ell}\right)^2\right)^{\frac{1}{2}}
\rightarrow 0,\quad \textrm{as } t\rightarrow \infty,
\end{equation}
whenever any direction $w$ has nonzero correlation $a_i^T w \neq 0$ with each first layer neuron in $\{a_i:i \in [\Ninnerlayer]\}$. 

Assume now that $\{a_i : i \in [\Ninnerlayer]\}$ is a perturbed orthonormal system, and that
second layer weights $\{b_{\ell} : \ell \in [\Nouterlayer]\}$ are generic and dense (nonsparse).
Recalling the definition $v_\ell = AG_0b_\ell/\N{AG_0b_\ell}$, the vector $v_\ell$ has, in this case, generally nonzero angle with each vector in $\{a_i : i \in [\Ninnerlayer]\}$,
while $a_i^\top a_j \approx 0$ for any $i \neq j$. Utilizing this with observation \eqref{eq:grad_along_direction}, it follows that
$\N{\nabla f(t a_i)}$ is expected to tend to $0$ much slower than $\N{\nabla f(t v_{\ell}}$ as $t\rightarrow \infty$. In fact, if $\{a_i : i \in [\Ninnerlayer]\}$
was an exactly orthonormal system, $\N{\nabla f(t a_i)}$ eventually would equal a positive constant when $t\rightarrow \infty$. We illustrate in Figure \ref{fig:trajectories} the different behavior of the trajectories $t \to \N{\nabla f(tw)}_2$  for $w \in \{\hat w_j \approx a_i \mbox{ for some } i\}$ and for $w \in \{\hat w_j \approx v_\ell \mbox{ for some } \ell\}$.

\begin{figure}
\centering
  \includegraphics[width=.5\linewidth]{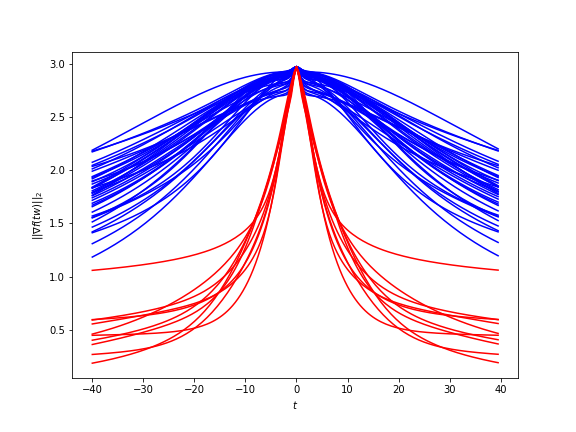}
  \caption{We illustrate the trajectories $t \to \N{\nabla f(tw)}_2$ for $w \in \{\hat w_j: j \in [m]\}$. The blue trajectories are those for $w \in \{\hat w_j \approx a_i \mbox{ for some } i\}$ and the red trajectories are those for $w \in \{\hat w_j \approx v_\ell \mbox{ for some } \ell\}$. We can observe the separation of the trajectories due to the different decay properties.}
\label{fig:trajectories}
\end{figure}

Practically, for $T \in \mathbb N$ and for each candidate vector in $\{\hat w_{j} : j\in [\Ninnerlayer+\Nouterlayer]\}$
we query $f$ to compute $\Delta_\epsilon f(t_k\hat w_j)$ for few steps $\{t_{k} : k \in [T]\}$
in order to approximate
\begin{align*}
\N{\N{\nabla f(t\hat w_j)}_2}_{L_2([-\infty,\infty])}^2
\approx \sum_{k=1}^{T} \N{\nabla f(t_k\hat w_j)}^2 \approx \sum_{k=1}^{T} \N{\Delta_\epsilon f(t_k\hat w_j)}^2 := \hat \CI(w_j).
\end{align*}
Then we compute a permutation $\pi: [m] \rightarrow [m]$ to
order the weights so that $\hat \CI(w_{\pi(i)}) \geq \hat \CI(w_{\pi(j)})$ whenever $\pi(i) > \pi(j)$.
The candidates $\{w_{\pi(j)} : j=1,\ldots,\Ninnerlayer\}$ have the slowest decay,
respectively largest norms, and are thus assigned to the first layer. The remaining
candidates $\{w_{\pi(\ell)}: \ell=\Ninnerlayer+1,\ldots,\Nouterlayer\}$
are assigned to the second layer.

\begin{table}
	\begin{center}
		\begin{tabular}{ccccccc}
			\hline
			& \multicolumn{3}{c}{$\Ninnerlayer = 30$} & \multicolumn{3}{c}{$\Ninnerlayer = 45$}\\
			\hline
			Scenario & $\Nouterlayer =  3$ & $\Nouterlayer = 9$ & $\Nouterlayer = 15$ & $\Nouterlayer = 5$ & $\Nouterlayer = 14$ & $\Nouterlayer = 23$\\
			\hline
			& $\CL_1$,\ \  $\CL_2$ & $\CL_1$,\ \  $\CL_2$ & $\CL_1$,\ \  $\CL_2$ & $\CL_1$,\ \  $\CL_2$ & $\CL_1$,\ \  $\CL_2$ & $\CL_1$,\ \  $\CL_2$\\
			POD/sig & $0.99$, $0.99$ & $0.99$, $1.0$ & $0.99$, $1.00$ & $0.99$, $0.99$ & $0.99$, $1.0$ & $1.00$, $1.00$\\
			POD/$\tanh$ & $0.87$, $0.89$ & $0.97$, $0.98$ & $0.99$, $1.00$ & $0.91$, $0.97$ & $0.99$, $1.0$ & $1.00$, $1.00$\\ \\
			$\bbS^{\Ninnerlayer - 1}$/sig & $0.94$, $0.71$ & $0.89$, $0.61$ & $0.85$, $0.58$ & $0.95$, $0.72$ & $0.89$, $0.63$ & $0.86$, $0.65$ \\
			$\bbS^{\Ninnerlayer - 1}$/$\tanh$ &
			$0.80$, $0.48$ & $0.80$ $0.54$ & $0.77$, $0.58$ &
			$0.83$, $0.56$ & $0.82$ $0.57$ & $0.79$, $0.64$\\
			\hline
		\end{tabular}
	\end{center}
	\caption{Success rates $\CL_1$ and $\CL_2$ (see \eqref{eq:success_rates}) when assigning candidates $\{\hat w_{i} : i \in [\Ninnerlayer+\Nouterlayer]\}$ to
	either first or second layer of the network. We consider the same scenarios as in Section \ref{sec:numerical_NNprofiles},
	e.g. POD/sig stands for perturbed orthogonal design with sigmoid activation, and $\bbS^{\Ninnerlayer - 1}$/$\tanh$
	for weights sampled independently from the unitsphere with $\tanh$ activation.}
	\label{tab:assignment_to_layers}
\end{table}

\paragraph{Numerical experiments}
We have applied the proposed strategy to assign vectors $\{\hat w_{j} : j \in [\Ninnerlayer+\Nouterlayer]\}$,
which are outputs of experiments conducted in the Section \ref{sec:numerical_NNprofiles}, to either the first
or the second layer. Since each $\hat w_j$
does not exactly correspond to a vector in $\{a_i : i \in [\Ninnerlayer]\}$
or $\{v_\ell : \ell \in [\Nouterlayer]\}$, we assign a \emph{ground truth} label $L_j = 1$ to $\hat w_j$
if the closest vector to $\hat w_j$ belongs to $\{a_i : i \in [\Ninnerlayer]\}$,
and $L_j = 2$ if it belongs to the set $\{v_\ell : \ell \in [\Nouterlayer]\}$.
Denoting similarly the predicted label $\hat L_j = 1$ if $\pi(j) \in \{1,\ldots,\Ninnerlayer\}$
and $\hat L_j = 2$ otherwise, we compute the success rates
\begin{equation}
\label{eq:success_rates}
\CL_1 := \frac{\#\{j : L_j = 1 \textrm{ and } \hat L_j = 1\}}{\Ninnerlayer},\quad \CL_2 := \frac{\#\{j : L_j = 2 \textrm{ and } \hat L_j = 2\}}{\Nouterlayer}
\end{equation}
to assess the proposed strategy. Hyperparameters are
$\epsilon = 10^{-5}$ for the step length in the finite difference approximation $\Delta_\epsilon f(\cdot)$,
and $t_k = - 20 + k$ for $k \in [40]$.

The results for all four scenarios considered in Section \ref{sec:numerical_NNprofiles}
are reported in Table \ref{tab:assignment_to_layers}. We see that our simple strategy
achieves remarkable success rates, in particular if the network weights in each layer
represent perturbed orthogonal systems. If the weights are sampled uniformly from the unit sphere  
with moderated dimension $d=\Ninnerlayer$, then, as one may expect, the success rate drops. In fact, for small $d=\Ninnerlayer$, the vectors $\{a_i : i \in [\Ninnerlayer]\}$ tend to be less orthogonal, and thus the assumption $a_i^\top a_j \approx 0$ for $i\neq j$ is not satisfied anymore.

Finally, we stress that the proposed strategy is simple, efficient and relies
only on few additional point queries of $f$ that are negligible compared to the recovery step itself (for reasonable query size $T$).
In fact, the method relies on a single (nonlinear) feature of the map $t \mapsto \N{\nabla f(t\hat w_j)}_2$
in order to decide upon the label of $\hat w_j$. We identify it as an interesting
future investigation to develop more robust approaches, potentially using higher dimensional
features of trajectories $ t \to \N{\nabla f(t\hat w_j)}_2$, to achieve high success rates even if
$a_i^\top a_j \approx 0$ for $i\neq j$ may not hold anymore. 

\subsection{Reconstructing the network function using gradient descent}
\label{subsec:reconstructing_f}
The previous section allows assigning unlabeled candidates $\{\hat w_j : j \in [\Ninnerlayer + \Nouterlayer]\}$
to either the first or second layer, resulting in matrices
$\hat A = [\hat a_1|\ldots|\hat a_{\Ninnerlayer}]$ and $\hat V = [\hat v_1|\ldots|\hat v_{\Nouterlayer}]$
that ideally approximate $A$ and $V$ up to column signs and permutations. Assuming that the network $f \in \CF(\Ninnerlayer,\Ninnerlayer, \Nouterlayer)$
is generated by shifts of one activation function, {\it i.e.}, $g_i(t) = \phi(t + \theta_i)$ and $h_\ell(t) = \phi(t + \tau_{\ell})$
for some $\phi$, this means only signs, permutations, and bias
vectors $\theta \in \bbR^{\Ninnerlayer}$, $\tau \in \bbR^{\Nouterlayer}$ are missing to fully reconstruct $f$. In this section we show how to
identify these remaining parameters by applying a gradient descent method to minimize the  least squares of the output misfit of the deparametrized
network. In fact, as we clarify below, the original network $f$ can be explicitly described as a function of the known entangled weights $a_i$ and $v_\ell$ and of the
unknown remaining parameters (signs, permutations, and biases), see Proposition \ref{prop:representation_f_new_version} and Corollary \ref{cor:alternative_f_representation} below.

Let now $\CD_{m}$ denote the set of $m\times m$ diagonal matrices, and define
a parameter space $\Omega := \CD_{\Nouterlayer}\times \CD_{\Ninnerlayer}\times \CD_{\Ninnerlayer} \times \bbR^{\Ninnerlayer}\times \bbR^{\Nouterlayer}$.
To reconstruct the original network $f$, we propose to fit parameters $(D_1, D_2, D_3, w, z) \in \Omega$ of a function $\hat f: \bbR^{\Ninnerlayer}\times \Omega  \rightarrow \bbR$
defined by
\begin{align*}
\hat f(x;D_1, D_2, D_3, w, z) &= 1^\top\phi(D_1  \hat V^\top \hat A^{-\top}  D_2 \phi(D_3 \hat A^\top x + w) + z)
\end{align*}
to a number of additionally sampled points $\{(X_i,Y_i) : i \in [m_f]\}$
where $Y_i = f(X_i)$ and $X_i \sim \CN(0,\Id_{\Ninnerlayer})$. The parameter fitting can be formulated
as solving the least squares 
\begin{equation}
\label{eq:new_GD_recover_network_functional}
\min_{(D_1, D_2, D_3, w, z) \in \Omega}
J(D_1,D_2,D_3, w, z):=\sum_{i=1}^{m_f}\left(Y_i - \hat f(X_i;D_1, D_2, D_3, w, z)\right)^2.
\end{equation}
We note that, due to the identification of the entangled weights and deparametrization of the problem, $\dim(\Omega) = 3\Ninnerlayer + 2\Nouterlayer$, which implies
that the least squares has significantly fewer free parameters compared to
the number $\Ninnerlayer^2 + (\Ninnerlayer\times \Nouterlayer)+ (\Ninnerlayer+\Nouterlayer)$ of original parameters of the entire network. Hence, our previous theoretical results of Section \ref{sec:individual_and_assignment} and  numerical experiments of Section \ref{sec:numerical_NNprofiles} greatly scale down  the usual effort of fitting all parameters at once. We may also mention at this point that the optimization \eqref{eq:new_GD_recover_network_functional} might have multiple global solutions due to possible symmetries, see also \cite{fe94} and Remark \ref{rem:new_prop_simplification_odd_functions}, and we shall try to keep into account the most obvious ones in our numerical experiments below.\\
We will now show that there exists parameters $(D_1, D_2, D_3, w, z) \in \Omega$
that allow for exact recovery of the original network, whenever $\hat A$ and $\hat V$
are correct up to signs and permutation. We first need the following proposition
that provides a different reparametrization of the network using $\hat A$ and $\hat V$.
The proof of the proposition requires only elementary linear algebra,
and properties of sign and permutation matrices. Details are deferred to Appendix \ref{subsec:proof_prop_22}.

\begin{prop}
\label{prop:representation_f_new_version}
Let $f \in \CF(\Ninnerlayer,\Ninnerlayer, \Nouterlayer)$ with $g_i(t) = \phi(t + \theta_i)$ and $h_\ell(t) = \phi(t + \tau_{\ell})$, and define
the function $\tilde f: \bbR^{\Ninnerlayer}\times \CD_{\Ninnerlayer}
\times \CD_{\Nouterlayer} \times \bbR^{\Ninnerlayer} \times \bbR^{\Nouterlayer} \rightarrow \bbR$ via
\begin{align*}
\tilde f(x;D, D', w, z) &= 1^\top\phi(D' \hat B^\top \phi(D \hat A^\top x + w) + z),\quad
\textrm{with}\quad \hat b_l:=  \frac{\diag\left(\left(\phi'(w)\right)^{-1}\right)D\hat A^{-1}\hat v_\ell}{\N{\diag\left(\left(\phi'(w)\right)^{-1}\right) D \hat A^{-1}\hat v_\ell}}.
\end{align*}
If there are sign matrices $S_A$, $S_V$, and permutations $\pi_A$,
$\pi_V$ such that  $A\pi_A = \hat A S_A$, $V\pi_V = \hat V S_V$, then we have
$f(x) = \tilde f(x; S_A, S_V, \pi_A^\top \theta, \pi_V^\top \tau)$.
\end{prop}

We note here that replacing $\hat f$ by $\tilde f$ in \eqref{eq:new_GD_recover_network_functional} is tempting because
it further reduces the number of parameters ($\dim(\CD_{\Ninnerlayer}
\times \CD_{\Nouterlayer} \times \bbR^{\Ninnerlayer} \times \bbR^{\Nouterlayer}) = 2(\Ninnerlayer + \Nouterlayer)$),
but, by an explicit computation, one can show that  evaluating the gradient of $\tilde f$ with respect to $D$ requires also the evaluation of $D^{-1}$.
Having in mind that $D$ ideally converges to $S_A$ during the optimization, diagonal entries
of $D$ are likely to cross  zero while optimizing. Thus such minimization may result unstable,
and we instead work with  $\hat f$. The following
Corollary shows that also this form allows  finding optimal parameters leading to the original network.

\begin{cor}
\label{cor:alternative_f_representation}
Let $f \in \CF(\Ninnerlayer,\Ninnerlayer, \Nouterlayer)$ with $g_i(t) = \phi(t + \theta_i)$ and $h_\ell(t) = \phi(t + \tau_{\ell})$.
If there exist sign matrices $S_A$, $S_V$, and permutations $\pi_A$,
$\pi_V$ such that  $A\pi_A = \hat A S_A$, $V\pi_V = \hat V S_V$, there exist diagonal matrices $D_1,  D_2$ such that $f(x) = \hat f(x;D_1, D_2, S_A, \pi_A^\top \theta, \pi_V^\top \tau)$.
\end{cor}
\begin{proof}
Based on Proposition \ref{prop:representation_f_new_version} we can rewrite $f(x) = 1^\top\phi(S_V \hat B^\top \phi(S_A \hat A^\top x + \pi_A^\top w) + \pi_V^\top z)$,
so it remains to show that $S_V \hat B^\top = D_1 \hat V^\top \hat A^{-\top} D_2$ for diagonal matrices $D_1$, $D_2$. First we note
\[
\diag(\phi'(\pi_A^\top \theta)^{-1}) = \pi_A^\top \diag(\phi'(\theta)^{-1}) \pi_A = \pi_A^\top G^{-1} \pi_A.
\]
Using this, and $D = S_A$ in the definition of $\hat B$ in Proposition \ref{prop:representation_f_new_version},
it follows that
\[
\hat B^\top = \diag(\Vert \pi_A^\top G^{-1} \pi_A S_A \hat A^{-1} v_1 \Vert, \ldots, \Vert \pi_A^\top G^{-1} \pi_A S_A \hat A^{-1} v_{\Nouterlayer}\Vert) \hat V^\top \hat A^{-\top}S_A\pi_A^\top G^{-1} \pi_A
\]
Multiplying by $S_V$ from the left, we obtain
\[
S_V \hat B^\top = \underbrace{S_V \diag(\Vert \pi_A^\top G^{-1} \pi_A S_A \hat A^{-1} v_1 \Vert, \ldots, \Vert \pi_A^\top G^{-1} \pi_A S_A \hat A^{-1} v_{\Nouterlayer}\Vert)}_{=D_1}
 \hat V^\top \hat A^{-\top}\underbrace{S_A \pi_A^\top G^{-1} \pi_A}_{=D_2}.
\]
\end{proof}

\begin{rem}[Simplification for odd functions]
\label{rem:new_prop_simplification_odd_functions}
If $\phi$ in Proposition \ref{prop:representation_f_new_version} satisfies $\phi(-t) = -\phi(t)$,
then $\hat f(x; D_1, S D_2, S D_3, S w, z) = \hat f(x; D_1, D_2, D_3, w, z)$ for arbitrary
sign matrix $S \in \CD_{\Ninnerlayer}$. Thus, choosing $S = S_A$, there are also diagonal $D_1$ and $D_2$ with
$f(x) = \hat f(x; D_1, D_2, \Id_{\Ninnerlayer}, S_A \pi_A^\top w, \pi_V^\top \tau)$.
\end{rem}

Assuming $\hat A$ and $\hat V$ are correct up to sign and permutation,
Corollary \ref{cor:alternative_f_representation} implies that $J = 0$ is the global optimum,
and it is attained by parameters leading to the original network $f$. Furthermore Remark
\ref{rem:new_prop_simplification_odd_functions} implies that there is ambiguity
with respect to $D_3$, if $\phi$ is an odd function. Thus
we can also prescribe $D_3 = \Id_{\Ninnerlayer}$ and neglect optimizing this variable if $\phi$ is odd.

We now study numerically the feasibility of \eqref{eq:new_GD_recover_network_functional}.
First, we consider the case $\hat A = A$ and $\hat V = V$ to assess \eqref{eq:new_GD_recover_network_functional},
isolated so not to suffer possible errors from other parts of our learning procedure (see Section \ref{sec:numerical_NNprofiles}
and Section \ref{subsec:distinguishing}). Afterwards we take into consideration also these additional approximations, and present results for
$\hat A \approx A$ and $\hat V \approx V$.

\paragraph{Numerical experiments}

\begin{table}
	\begin{center}
		\begin{tabular}{cccccccc}
			\hline
			& & \multicolumn{3}{c}{$\Ninnerlayer = 30$} & \multicolumn{3}{c}{$\Ninnerlayer = 45$}\\
			\hline
			Scenario & & $\Nouterlayer =  3$ & $\Nouterlayer = 9$ & $\Nouterlayer = 15$ & $\Nouterlayer = 5$ & $\Nouterlayer = 14$ & $\Nouterlayer = 23$\\ \hline
POD/sig & MSE & $1.2e^{-5}$ & $5.4e^{-6}$ & $4.7e^{-5}$ & $1.1e^{-5}$ & $4.6e^{-6}$ & $5.6e^{-6}$ \\
       & $E_{\infty}$ & $4.3e^{-3}$ & $3.8e^{-3}$ & $4.9e^{-3}$ & $3.9e^{-3}$ & $3.4e^{-3}$ & $4.4e^{-3}$\\
       & $E_{\theta}$ & $4.1e^{-1}$ & $2.7e^{-1}$ & $1.7e^{-1}$ & $4.4e^{-1}$ & $2.6e^{-1}$ & $1.7e{-1}$\\
       & $E_{\tau}$ & $3.9e^{-2}$ & $1.9e^{-2}$ & $3.1e^{-2}$ & $4.4e^{-2}$ & $2.1e^{-2}$ & $3.3e^{-2}$\\ \hline
POD/$\tanh$ & MSE & $1.9e^{-7}$ & $1.5e^{-9}$ & $1.2e^{-10}$ & $1.1e^{-7}$ & $7.5e^{-10}$ & $8.4e^{-12}$\\
       & $E_{\infty}$ & $7.3e^{-4}$ & $5.4e^{-5}$ & $1.3e^{-5}$ & $4.6e^{-4}$ & $4.2e^{-5}$ & $3.6e^{-6}$\\
       & $E_{\theta}$ & $2.9e^{-3}$ &  $6.8e^{-8}$  & $4.2e^{-8}$ & $2.6e^{-3}$ & $2.1e^{-7}$ &  $1.5e^{-9}$\\
       & $E_{\tau}$ & $3.3e^{-4}$ & $1.1e^{-7}$ & $2.1e^{-8}$ & $1.1e^{-4}$ & $8.4e^{-8}$ & $9.5e^{-10}$\\ \hline
$\bbS^{\Ninnerlayer-1}$/sig & MSE & $1.3e^{-5}$ & $9.7e^{-6}$ & $1.4e^{-5}$ & $1.2e^{-5}$ & $9.4e^{-6}$ & $1.6e^{-5}$\\
       & $E_{\infty}$ & $4.9e^{-3}$ & $5.7e^{-3}$ & $8.2e^{-3}$ & $4.5e^{-3}$ & $5.5e^{-3}$ & $8.5e^{-3}$\\
       & $E_{\theta}$ & $4.5e^{-1}$ & $3.5e^{-1}$ & $3.0e^{-1}$ & $4.5e^{-1}$ & $2.7e^{-1}$ & $2.4e^{-1}$\\
       & $E_{\tau}$ & $3.7e^{-2}$ & $7.0e^{-2}$ & $1.2e^{-1}$ & $5.3e^{-2}$ & $5.5e^{-2}$ & $1.1e^{-1}$ \\ \hline
$\bbS^{\Ninnerlayer-1}$/$\tanh$ & MSE & $4.4e^{-7}$ & $4.8e^{-9}$ & $4.9e^{-10}$ & $7.7e^{-8}$ & $1.5e^{-9}$ & $1.6e^{-11}$\\
       & $E_{\infty}$ & $1.3e^{-3}$ & $1.4e^{-4}$ & $3.0e^{-5}$ & $5.0e^{-4}$ & $6.0e^{-5}$ & $5.7e^{-6}$\\
       & $E_{\theta}$ & $1.9e^{-2}$ & $1.5e^{-6}$ & $3.1e^{-7}$ & $3.7e^{-4}$ & $3.5e^{-7}$ & $4.9e^{-9}$\\
       & $E_{\tau}$ & $7.6e^{-4}$ &  $1.4e^{-6}$ & $5.7e^{-8}$ &  $7.5e^{-5}$ & $2.8e^{-7}$ &  $1.5e^{-9}$\\ \hline
		\end{tabular}
	\end{center}
	\caption{Errors of the reconstructed network using \eqref{eq:new_GD_recover_network_functional}
  when prescribing $\hat A = A$ and $\hat V = V$. The scenarios correspond to those in Section \ref{sec:numerical_NNprofiles} and Section \ref{subsec:distinguishing}.}
	\label{tab:function_learning_results}
\end{table}

\begin{table}
	\begin{center}
		\begin{tabular}{cccccccc}
			\hline
			& & \multicolumn{3}{c}{$\Ninnerlayer = 30$} & \multicolumn{3}{c}{$\Ninnerlayer = 45$}\\
			\hline
			Scenario & & $\Nouterlayer =  3$ & $\Nouterlayer = 9$ & $\Nouterlayer = 15$ & $\Nouterlayer = 5$ & $\Nouterlayer = 14$ & $\Nouterlayer = 23$\\ \hline
POD/sig & MSE & $6.4e^{-5}$ & $3.1e^{-2}$ & $2.9e^{-2}$ & $6.8e^{-5}$ & $1.1e^{-2}$ & $5.4e^{-4}$ \\
       & $E_{\infty}$ & $1.1e^{-2}$ & $7.9e^{-2}$ & $9.0e^{-2}$ & $1.4e^{-2}$ & $4.5e^{-2}$ & $4.7e^{-2}$\\
       & Trials [$\%$] & $63$ & $80$ & $90$ & $37$ & $67$ & $93$\\ \hline
POD/$\tanh$ & MSE & $-$ & $2.7e^{-2}$ & $8.9e^{-3}$ & $-$ & $4.2e^{-3}$ & $7.7e^{-3}$\\
       & $E_{\infty}$ &  $-$ & $1.9e^{-1}$ & $1.2e^{-1}$ & $-$ & $8.3e^{-2}$ & $9.7e^{-2}$ \\
       & Trials [$\%$] & $0$ & $23$ & $76$ & $0$ & $43$ & $96$\\ \hline
		\end{tabular}
	\end{center}
	\caption{Errors of the reconstructed network using \eqref{eq:new_GD_recover_network_functional}
  when using approximated $\hat A \approx A$ and $\hat V \approx V$ (up to sign and permutation). Trials indicates the percentage of repititions
  where $\hat A$ and $\hat V$ satisfy \eqref{eq:trial_condition}. The scenarios correspond to those in Section \ref{sec:numerical_NNprofiles} and Section \ref{subsec:distinguishing}.}
	\label{tab:function_learning_results_approx_A_V}
\end{table}
We minimize \eqref{eq:new_GD_recover_network_functional} by standard gradient descent and
learning rate $0.5$ if $\phi(t) = \frac{1}{1+e^{-t}} - \frac{1}{2}$ (shifted sigmoid),
respectively learning rate $0.025$ if $\phi(t) = \tanh(t)$. We sample $m_f = 10(\Ninnerlayer + \Nouterlayer)$
additional points, which is only slightly
more than the number of free parameters. Gradient descent is run for 500K iterations (due to small number of variables,
this is not time consuming), and only prematurely stopped it, if the iteration stalls. Initially we
set $D_2 = D_3 = \Id_{\Ninnerlayer}$, and all other variables are set to random draws from $\CN(0,0.1)$.

Denoting $\omega^* = (D_1^*, D_2^*, D_3^*, w^*, z^*) \in \Omega$ as the gradient descent output, we measure
the relative mean squared error (MSE) and the relative $L_{\infty}$-error
\[
\textrm{MSE} = \frac{\sum_{i=1}^{m_{\textrm{test}}} (\hat f(Z_i;\omega^*) - f(Z_i))^2}{\sum_{i=1}^{m_{\textrm{test}}}f(Z_i)^2},\quad
\textrm{E}_{\infty} = \frac{\max_{i \in [m_{\textrm{test}}]}\SN{\hat f(Z_i;\omega^*) - f(Z_i)}}{\max_{i \in [m_{\textrm{test}}]}\SN{f(Z_i)}},
\]
using $m_{\textrm{test}} = 50000$ samples $Z_i \sim \CN(0,\Id_{\Ninnerlayer})$. Moreover, we also report
the relative bias errors
\[
E_{\theta} = \frac{\N{w^* - \theta}^2}{\N{\theta}^2},\quad
E_{\eta} = \frac{\N{z^* - \eta}^2}{\N{\eta}^2},
\]
which indicate if the original bias vectors are recovered. We repeat each experiments 30
times, and report averaged values.

Table \ref{tab:function_learning_results} presents the results of the experiments and
shows that we reconstruct a network function that is very close to the original network $f$
in both $L_2$ and $L_{\infty}$ norm, and in every scenario. The maximal error is $\approx 10^{-3}$,
which is likely further reducible by increasing the number of gradient descent
iterations, or using finer tuned learning rates or acceleration methods. Therefore, the experiments
strongly suggest that we are indeed reconstructing a function that approximates $f$
uniformly well. Inspecting the errors $E_{\theta}$ and $E_{\eta}$ also supports
this claim, at least in all scenarios where the $\tanh$ activation is used. In many cases
the relative errors are below $10^{-7}$, implying that we recover the original bias vectors of the network.
Suprisingly, the accuracy of recovered biases slightly drops of few orders of magnitude in the sigmoid case, despite convincing results
when measuring predictive performance in $L_2$ and $E_{\infty}$. We believe that this is due to
faster flattening of the gradients around the stationary point compared to the case
of a $\tanh$ activation function, and that it can be improved by using
more sophisticated strategies of choosing a gradient descent step size.
We also tested \eqref{eq:new_GD_recover_network_functional} when fixing $D = \Id_{\Ninnerlayer}$
since $\tanh$ and the shifted sigmoid are odd functions, and thus Remark \ref{rem:new_prop_simplification_odd_functions}
applies. The results are consistently slightly better than Table \ref{tab:function_learning_results},
but are qualitatively similar.

We ran similar experiments for perturbed orthogonal weights and when using $\hat A$ and $\hat V$ precomputed with the methods
we described in Section \ref{sec:numerical_NNprofiles} and Section \ref{subsec:distinguishing}.
The quality of the results varies dependent on whether $\hat A \approx A$ and $\hat V \approx V$
(up to sign and permutation) holds, or a fraction of the weights has not been
recovered. To isolate cases where $\hat A \approx A$ and $\hat V \approx V$
holds, we compute averaged MSE and $L_{\infty}$ over all trials satisfying
\begin{equation}
\label{eq:trial_condition}
\sum_{i=1}^{\Ninnerlayer}\CE(a_i) +  \sum_{\ell = 1}^{\Nouterlayer}\CE(v_{\ell})  < 0.5,\quad \textrm{(see Section \ref{sec:numerical_NNprofiles} for the Definition of } \CE).
\end{equation}
We report the averaged errors, and the number of trials satisfying this condition in Table \ref{tab:function_learning_results_approx_A_V}.
It shows that the reconstructed function is close
to the original function, even if the weights are only approximately correct.
Therefore we conclude that that minimizing \eqref{eq:new_GD_recover_network_functional} provides
a very efficient way of learning the remaining network parameters from just few additional samples,
once entangled network weights $A$ and $V$ are (approximately) known.

\appendix
\section{Appendix}

The following Lemma implies that $\{a_1\otimes a_1,\ldots a_{\Ninnerlayer}\otimes a_{\Ninnerlayer},v_1 \otimes v_1,\ldots,v_{\Nouterlayer}\otimes v_{\Nouterlayer}\}$
satisfying the properties of Definition \ref{def:f} is a system of linearly independent matrices.
\begin{lem}
	\label{lem:linear_indepdency}
	Let $\{z_1,\ldots,z_m\}\subset \bbR^{m}$ have unit norm and satisfy $\sum_{i=1}^{m} \left\langle z_j, z_i\right\rangle^2  \leq \Uframe$ for all $j=1,\ldots,m$.
	If $1 < \Uframe < 2$, the system $\{z_1\otimes z_1,\ldots,z_{m}\otimes z_{m}\}$
	is linearly independent.
\end{lem}
\begin{proof}
	Assume to the contrary the $\{z_1\otimes z_1,\ldots,z_m\otimes z_m\}$ are not linearly independent, then there exists
	$\sigma\neq 0 \in \bbR^{m}$ with $0 = \sum_{i=1}^{m}\sigma_i z_i\otimes z_i$, or equivalently
	$0 = \sum_{i=1}^{m}\sigma_i \langle x, z_i\rangle^2$ for all $x \in \bbR^d$.
	Without loss of generality assume $\N{\sigma}_{\infty} = \max_i \sigma_i$ (otherwise we multiply the representation by $-1$),
	and denote by $i^*$ the index achieving the maximum. Then we have
	\begin{align*}
	0 &=\sum\limits_{i=1}^{N} \sigma_i \left\langle z_i, z_{i^*}\right\rangle^2 = \sigma_{i^*}\N{z_{i^*}}^2 +  \sum\limits_{i\neq i^*} \sigma_i \left\langle z_i, z_{i^*}\right\rangle^2
	\geq \sigma_{i^*}\N{z_{i^*}}^2 +\min_{i} \sigma_i \sum\limits_{i\neq i^*}\left\langle z_i, z_{i^*}\right\rangle^2
	\end{align*}
	Since $\min_{i} \sigma_i \geq 0$ immediately yields a contradiction, we continue with the case $\min_{i} \sigma_i < 0$.
	We can further bound
	\begin{align*}
	0 &\geq \sigma_{i^*}\N{z_{i^*}}^2 + \min_{i} \sigma_i \sum\limits_{i\neq i^*}\left\langle z_i, z_{i^*}\right\rangle^2
	\geq \sigma_{i^*} + \min_{i} \sigma_i \left(\Uframe - 1\right)\N{z_i^*}^2 = \sigma_{i^*} + \min_{i} \sigma_i \left(\Uframe - 1\right),
	\end{align*}
	and by division through $\left(\Uframe - 1\right)$, and subtracting $\min_{i} \sigma_i$ we obtain
	$\SN{\min_{i} \sigma_i} \geq \sigma_{i^*}(\Uframe - 1)^{-1}$.
	Since $(C_{F} - 1)^{-1} > 1$, this yields the contradiction $\N{\sigma}_{\infty} \geq \SN{\min_{i} \sigma_i} > \sigma_{i^*} = \N{\sigma}_{\infty}$.
\end{proof}
The linear independence of the system $\{a_1\otimes a_1,\ldots a_{\Ninnerlayer}\otimes a_{\Ninnerlayer},v_1 \otimes v_1,\ldots,v_{\Nouterlayer}\otimes v_{\Nouterlayer}\}$
implies that it is a Riesz basis for $\CW := \operatorname{span}{a_1\otimes a_1,\ldots a_{\Ninnerlayer}\otimes a_{\Ninnerlayer},v_1 \otimes v_1,\ldots,v_{\Nouterlayer}\otimes v_{\Nouterlayer}}$. As such
there exists constants $\Lriesz$, $\Uriesz$ such that for every $\sigma \in \bbR^{\Ninnerlayer + \Nouterlayer}$
\begin{align}
\label{eq:riesz_constants}
\Lriesz \N{\sigma}_2^2 \leq \N{\sum_{i=1}^{\Ninnerlayer}\sigma_i a_i\otimes a_i + \sum_{i=1}^{\Nouterlayer}\sigma_{\Ninnerlayer + i} v_i \otimes v_i}^2 \leq \Uriesz \N{\sigma}_2^2.
\end{align}

\subsection{Additional proofs for Section \ref{sec:approx_matrix_space}}\label{subsec:proofofsec:approx_matrix_space}
\begin{proof}[Proof of Lemma \ref{lem:errorhessian}]
	Fix any pair $k,n \in [d]$ and define $\phi(t) = f(x + t e_k + \epsilon e_n) - f(x + t e_k)$, where $e_k$ denotes the $k$-th standard vector. By the mean value theorem and for $\Delta^2_{\epsilon}f(x) \in \R^{d\times d}$ given as in \eqref{def:finite_differences_hessian}, there exist $0 < \xi_1, \xi_2 < \epsilon$ such that
	\begin{align*}
	(\Delta^2_\epsilon[f](x))_{kn} &= \frac{\phi(\epsilon)-\phi(0)}{\epsilon^2} = \frac{\phi'(\xi_1)}{\epsilon} \\
	&= \frac{\frac{\partial f}{\partial x_k}(x+\xi_1 e_k + \epsilon e_n) - \frac{\partial f}{\partial x_k}(x+\xi_1 e_k)}{\epsilon}\\
	&= \frac{\partial^2 f}{\partial x_k\partial x_n}(x + \xi_1 e_k + \xi_2 e_n).
	\end{align*}
	Hence, we obtain
	\begin{align*}
	\abs{\nabla^2 f((x))_{kn} - (\Delta^2_\epsilon[f](x))_{kn}} = \left\vert \frac{\partial^2 f}{\partial x_k\partial x_n}(x) - \frac{\partial^2 f}{\partial x_k\partial x_n}(x + \xi_1 e_k + \xi_2 e_n)\right\vert .
	\end{align*}
	Assume $k,n$ to be fixed and denote $\tilde{x} = x + \xi_1 e_k + \xi_2 e_n$. By recalling our definition of $\nabla^2 f(x)$ in \eqref{eq:hessian_decomp}, it follows $\frac{\partial^2 f}{\partial x_k\partial x_n}(x) = \varphi_1(x) + \varphi_2(x)$, where
	\begin{align*}
	\varphi_1(x) &:= \sum^{\Nouterlayer}_{\ell=1}\sum^{\Ninnerlayer}_{i,j=1} h_\ell''(b_\ell^T g(A^T x))g_i'(a_i^T x)g_j'(a_j^T x)a_{ki}a_{nj}b_{i\ell}b_{j\ell}, \\
	\varphi_2(x) &:= \sum^{\Nouterlayer}_{\ell=1}\sum^{\Ninnerlayer}_{i=1}h_\ell'(b_\ell^T g(A^T x))g_i''(a_i^T x) a_{ki}a_{ni}b_{i\ell}.
	\end{align*}
	Thus
	\begin{align*}
	\abs{(\nabla^2 f(x))_{kn} - (\Delta^2_\epsilon[f](x))_{kn}} \leq \abs{\varphi_1(x) - \varphi_1(\tilde{x})} + \abs{\varphi_2(x) - \varphi_2(\tilde{x})}.
	\end{align*}
	As before, we start by applying the Lipschitz continuity to the summands of $\abs{\varphi_1(x) - \varphi_1(\tilde{x})}$:
	\begin{align*}
	&\abs{h_\ell''(b_\ell^T g(A^T x))g_i'(a_i^T x)g_j'(a_j^T x) - h_\ell''(b_\ell^T g(A^T \tilde{x}))g_i'(a_i^T \tilde{x})g_j'(a_j^T \tilde{x})} \\
	\leq& \abs{h_\ell''(b_\ell^T g(A^T x))g_i'(a_i^T x)g_j'(a_j^T x) - h_\ell''(b_\ell^T g(A^T x))g_i'(a_i^T \tilde{x})g_j'(a_j^T \tilde{x})} \\
	+& \abs{h_\ell''(b_\ell^T g(A^T x))g_i'(a_i^T \tilde{x})g_j'(a_j^T \tilde{x}) - h_\ell''(b_\ell^T g(A^T \tilde{x}))g_i'(a_i^T \tilde{x})g_j'(a_j^T \tilde{x})} \\
	\leq& \eta_2\abs{g_i'(a_i^T x)g_j'(a_j^T x) - g_i'(a_i^T \tilde{x})g_j'(a_j^T \tilde{x})} + \kappa_1^2\abs{h_\ell''(b_\ell^T g(A^T x)) - h_\ell''(b_\ell^T g(A^T \tilde{x}))}\\
	\leq& \eta_2\left[\abs{g_i'(a_i^T x)g_j'(a_j^T x) - g_i'(a_i^T \tilde{x})g_j'(a_j^T x)}  + \abs{g_i'(a_i^T \tilde{x})g_j'(a_j^T x) - g_i'(a_i^T \tilde{x})g_j'(a_j^T\tilde{x} )} \right]\\
	+& \kappa_1^2 \eta_3 \left\vert\sum^{\Ninnerlayer}_{I=1} b_{I\ell} \left( g_I(a_I^T x) - g_I(a_I^T \tilde{x})\right)\right\vert\\
	\leq& \eta_2 \left[\kappa_1\abs{g_i'(a_i^T x) - g_i'(a_i^T \tilde{x})} + \kappa_1\abs{g_j'(a_j^T x) - g_j'(a_j^T \tilde{x})} \right] + \kappa_1^3 \eta_3 \left\vert\sum^{\Ninnerlayer}_{I=1} b_{Il}a_I^T(x-\tilde{x})\right\vert \\
	\leq& \eta_2 \kappa_1 \kappa_2\left[\abs{a_i^T(x - \tilde{x})} +\abs{a_j^T(x - \tilde{x})}\right] + \kappa_1^3 \eta_3 \left\vert\sum^{\Ninnerlayer}_{I=1} b_{I\ell}a_I^T(x-\tilde{x})\right\vert \\
	\leq& \eta_2 \kappa_1 \kappa_2\left[ \abs{ \xi_1 a_{ki} + \xi_2 a_{ni}} + \abs{ \xi_1 a_{kj} + \xi_2 a_{nj}} \right]
	+ \kappa_1^3 \eta_3 \left\vert \sum^{\Ninnerlayer}_{I=1} b_{I\ell}(\xi_1 a_{kI} + \xi_2 a_{nI})\right\vert\\
	\leq& \eta_2 \kappa_1 \kappa_2 \epsilon\left[ \abs{a_{ki}} + \abs{a_{ni}} + \abs{a_{kj}}+ \abs{a_{nj}} \right]
	+ \kappa_1^3 \eta_3 \epsilon \sum^{\Ninnerlayer}_{I=1} \abs{b_{I\ell}}(\abs{a_{kI}} + \abs{a_{nI}})\\
	\leq& \tilde{C} \epsilon \left[ \abs{a_{ki}} + \abs{a_{ni}} + \abs{a_{kj}}+ \abs{a_{nj}} + \sum^{\Ninnerlayer}_{I=1} \abs{b_{I\ell}}(\abs{a_{kI}} + \abs{a_{nI}})\right],
	\end{align*}
	where $\tilde{C} = \max\set{ \eta_2 \kappa_1 \kappa_2 ,\kappa_1^3 \eta_3}$. Hence,
	\begin{align*}
	\abs{\varphi_1(x) - \varphi_1(\tilde{x})} \leq \sum^{\Nouterlayer}_{l=1}\sum^{\Ninnerlayer}_{i,j=1}\tilde{C} \epsilon \left[ \abs{a_{ki}} + \abs{a_{ni}} + \abs{a_{kj}}+ \abs{a_{nj}} + \sum^{\Ninnerlayer}_{I=1} \abs{b_{I\ell}}(\abs{a_{kI}} + \abs{a_{nI}})\right]\abs{b_{i\ell}b_{j\ell}a_{ki}a_{nj}}.
	\end{align*}
	Now
	\begin{align*}
	&\sum^{\Nouterlayer}_{\ell=1}\sum^{\Ninnerlayer}_{i,j=1}\tilde{C}\epsilon \left[ \abs{a_{ki}} + \abs{a_{ni}} + \abs{a_{kj}}+ \abs{a_{nj}} \right]\abs{b_{i\ell}b_{j\ell}a_{ki}a_{nj}}\\
	&=\sum^{\Nouterlayer}_{\ell=1}\sum^{\Ninnerlayer}_{i,j=1}\tilde{C} \epsilon \left[ \abs{a_{ki}^2a_{nj}} + \abs{a_{ni}a_{ki}a_{nj}} + \abs{a_{kj}a_{ki}a_{nj}}+ \abs{a_{nj}^2a_{ki}} \right]\abs{b_{i\ell}b_{j\ell}}.
	\end{align*}
	Applying the triangle inequality of the Frobenius norm results in
	\begin{align*}
	&\tilde{C} \epsilon \left( \sum^{d}_{k,n=1} \left[ \sum^{\Nouterlayer}_{\ell=1}\sum^{\Ninnerlayer}_{i,j=1}\left[ \abs{a_{ki}^2a_{nj}} + \abs{a_{ni}a_{ki}a_{nj}} + \abs{a_{kj}a_{ki}a_{nj}}+ \abs{a_{nj}^2a_{ki}} \right]\abs{b_{i\ell}b_{j\ell}} \right]^2\right)^{\frac{1}{2}}\\
	\leq& 2 \tilde{C} \epsilon \left( \sum^{d}_{k,n=1} \left[ \sum^{\Nouterlayer}_{\ell=1}\sum^{\Ninnerlayer}_{i,j=1}\abs{a_{ki}^2a_{nj}} \abs{b_{i\ell}b_{j\ell}} \right]^2\right)^{\frac{1}{2}}
	+ 2 \tilde{C} \epsilon \left( \sum^{d}_{k,n=1} \left[ \sum^{\Nouterlayer}_{\ell=1}\sum^{\Ninnerlayer}_{i,j=1}\abs{a_{ni}a_{ki}a_{nj}} \abs{b_{i\ell}b_{j\ell}} \right]^2\right)^{\frac{1}{2}}\\
	\leq& 2 \tilde{C} \epsilon \sum^{\Nouterlayer}_{\ell=1}\sum^{\Ninnerlayer}_{i,j=1} \abs{b_{i\ell}b_{j\ell}} \left(  \sum^{d}_{k,n=1} \left[ \abs{a_{ki}^2a_{nj}}  \right]^2\right)^{\frac{1}{2}}
	+ 2 \tilde{C} \epsilon \sum^{\Nouterlayer}_{\ell=1}\sum^{\Ninnerlayer}_{i,j=1} \abs{b_{i\ell}b_{j\ell}} \left(  \sum^{d}_{k,n=1} \left[ \abs{a_{ni}a_{ki}a_{nj}}  \right]^2\right)^{\frac{1}{2}}\\
	\end{align*}
	 \begin{align*}
	\leq& 2 \tilde{C} \epsilon \sum^{\Nouterlayer}_{\ell=1}\sum^{\Ninnerlayer}_{i=1} \abs{b_{i\ell}}\sum^{\Ninnerlayer}_{j=1}\abs{b_{j\ell}} \left\lbrace  \left(  \sum^{d}_{k,n=1} \left[ \abs{a_{ki}^2a_{nj}}  \right]^2\right)^{\frac{1}{2}}
	+\left(  \sum^{d}_{k,n=1} \left[ \abs{a_{ni}a_{ki}a_{nj}}  \right]^2\right)^{\frac{1}{2}}\right\rbrace \\
	\leq& 2 \tilde{C} \epsilon \sum^{\Nouterlayer}_{\ell=1}\sum^{\Ninnerlayer}_{i=1} \abs{b_{i\ell}}\sum^{\Ninnerlayer}_{j=1}\abs{b_{j\ell}} \left\lbrace  \left(  \sum^{d}_{k=1} a_{ki}^4 \sum^{d}_{n=1} a_{nj}^2  \right)^{\frac{1}{2}} +
	\left(  \sum^{d}_{k=1} a_{ki}^2 \sum^{d}_{n=1} a_{nj}^2 \right)^{\frac{1}{2}}\right\rbrace \\
	\leq& 4 \tilde{C} \epsilon \sum^{\Nouterlayer}_{\ell=1}\sum^{\Ninnerlayer}_{i=1} \abs{b_{i\ell}}\sum^{\Ninnerlayer}_{j=1}\abs{b_{j\ell}}  \norm{a_i}_2 \norm{a_j}_2 \\
	\leq & 4  \tilde{C} \epsilon \sum^{\Nouterlayer}_{\ell=1} \norm{b_\ell}^2_1\\
	\leq& 4  \tilde{C} \epsilon m_1 m.
	\end{align*}
	The last inequalities are due to $\norm{a_i}_2 = 1$ for all $i \in [\Ninnerlayer]$ and $\norm{b_\ell}_1 \leq \sqrt{\Ninnerlayer}\norm{b_\ell}_2 = \sqrt{\Ninnerlayer}$ for all $\ell \in [\Nouterlayer]$. A similar computation yields
	\begin{align*}
	\tilde{C} \epsilon \left( \sum^{d}_{k,n=1} \left[ \sum^{\Nouterlayer}_{\ell=1}\sum^{\Ninnerlayer}_{i,j=1} \sum^{\Ninnerlayer}_{I=1} (\abs{a_{kI}} + \abs{a_{nI}})]\abs{a_{ki}a_{nj}}\abs{ b_{I\ell}b_{i\ell}b_{j\ell}} \right]^2\right)^{\frac{1}{2}} \leq 4\tilde{C} \epsilon \Nouterlayer \Ninnerlayer^{\frac{3}{2}}.
	\end{align*}
	Combining both results gives
	\begin{align*}
	\left\lbrace \sum^{d}_{k,n=1}\left[\abs{\varphi_1(x) - \varphi_1(x + \xi_{1,kn}e_k + \xi_{2, kn}e_n)}\right]^2 \right\rbrace^{\frac{1}{2}} \leq 8\tilde{C} \epsilon \Nouterlayer \Ninnerlayer^{\frac{3}{2}}.
	\end{align*}
	Here we denote $\xi_{1,kn},  \xi_{2, kn}$, to make clear that $\xi_1, \xi_2$ are changing for every partial derivative of second order. However, all $\xi_{1,kn},  \xi_{2, kn}$ are bounded by $\epsilon$, so our result still holds.
	Applying the same procedure to $\abs{\varphi_2(x) - \varphi_2(\tilde{x})}$ yields
	\begin{align*}
	&\abs{h_\ell'(b^T_\ell g(A^T x)) g_i''(a_i^T x) - g_{\ell}'(b_\ell g(A^T \tilde{x}))g_i''(a_i^T \tilde{x})} \\
	\leq& \eta_1 \kappa_3 \epsilon ( \abs{a_{ki}} + \abs{a_{ni}}) + \kappa_2 \eta_2 \kappa_1 \epsilon \left\vert \sum^{\Nouterlayer}_{I=1} b_{I\ell}(\abs{a_{Ik}} + \abs{a_{In}})\right\vert.
	\end{align*}
	By setting $\hat{C} = \max\set{\eta_1\kappa_3, \kappa_1\kappa_2\eta_2}$, we can can develop the same bounds for both parts of the right sum as for $\varphi_1$, and get
	\begin{align*}
	\left\lbrace \sum^{d}_{k,n=1}\left[\abs{\varphi_2(x) - \varphi_2(x + \xi_{1,kn}e_k + \xi_{2, kn}e_n)}\right]^2 \right\rbrace^{\frac{1}{2}} \leq 8\hat{C} \epsilon \Nouterlayer \Ninnerlayer^{\frac{3}{2}}.
	\end{align*}
	Finally, we get
	\begin{align*}
	&\norm{\nabla^2 f(x) - \Delta^2_{\epsilon}f(x)}_F \leq  \left\lbrace \sum^{d}_{k,n=1}\left[\left\vert \frac{\partial^2 f}{\partial x_k\partial x_n}(x) - \frac{\partial^2 f}{\partial x_k\partial x_n}(x + \xi_{1,kn}e_k + \xi_{2, kn}e_n)  \right\vert \right]^2 \right\rbrace^{\frac{1}{2}} \\
	=& \left\lbrace \sum^{d}_{k,n=1}\left[\left\vert \varphi_1(x) + \varphi_2(x) - \varphi_1(x + \xi_{1,kn}e_k + \xi_{2, kn}e_n) - \varphi_2(x + \xi_{1,kn}e_k + \xi_{2, kn}e_n) \right\vert \right]^2 \right\rbrace^{\frac{1}{2}}\\
	\leq& 8\tilde{C}\epsilon \Nouterlayer \Ninnerlayer^{\frac{3}{2}} + 8\hat{C}\epsilon \Nouterlayer \Ninnerlayer^{\frac{3}{2}}.
	\end{align*}
	Setting $C_{\Delta} = 16\max\set{\tilde{C}, \hat{C}}$ finishes the proof.
\end{proof}
\subsection{Additional results and proofs for Section \ref{sec:individual_and_assignment}}
\begin{lem}
\label{lem:example_rank_one_matrices}
Let $\{w_\ell \otimes w_\ell: \ell \in [m]\}$ be a set of $m < 2d - 1$ rank one matrices in $\bbS$ such that
any subset of $\lceil m/2 \rceil + 1$ vectors $\{w_{\ell_j}: j \in [\lceil m/2 \rceil+1]\}$ is linearly independent.
Then for any $X \in \operatorname{span}{\{w_\ell \otimes w_\ell: \ell \in [m]\}} \cap \bbS$ with $\Rank(X)= 1$, there
exists $\ell^*$ such that $X = w_{\ell^*}\otimes w_{\ell^*}$.
\end{lem}
\begin{proof}
Let $X = \sum_{\ell=1}^{m}\alpha_\ell w_\ell \otimes w_\ell \in \bbS$, and denote $\CI = \{\ell \in [m] : \alpha_\ell \neq 0\}$.
If $1 < \SN{\CI} \leq \lceil m/2 \rceil + 1$, the vectors $\{w_{i}:i \in \CI\}$ are linearly independent,
and thus $\rank(X) = \SN{\CI} > 1$. Otherwise, we split $\CI = \CI_1\cup \CI_2$
with $\SN{\CI_1} = \lceil m/2 \rceil + 1$ and $\SN{\CI_2} \leq m - \lceil m/2 \rceil -1 \leq m/2 - 1$.
If we accordingly split $X=X_1 + X_2$ with $X_j := \sum_{\ell \in \CI_j}\alpha_\ell w_\ell \otimes w_\ell$, the assumption
implies $\rank(A_1) = \lceil m/2 \rceil + 1$ and $\rank(A_2) \leq  m/2 - 1$. Since furthermore
$\rank(X) \geq \rank(X_1) - \rank(X_2)$, it follows that $\rank(X) \geq \lceil m/2 \rceil + 1 - (m/2 - 1) \geq 2$.
\end{proof}
\begin{cor}
\label{cor:uniqueness of matrices}
Assume $m < 2d - 1$, and $\{w_\ell: \ell \in [m]\}$ satisfies the upper frame bound \eqref{eq:frame_properties_A_B}
with $\Ferror := \Uframe - 1 < \lceil\frac{m}{2}\rceil^{-1}$.
Then for $X \in \CW \cap \bbS$ of $\Rank(X) = 1$, there exists $\ell^*$ such that $X = w_{\ell^*}\otimes w_{\ell^*}$.
\end{cor}
\begin{proof}
{
To apply Lemma \ref{lem:example_rank_one_matrices}, we establish a lower bound for the
size of the smallest linearly dependent subset of $\{w_\ell : \ell \in [m]\}$, denoted
commonly also by $\textrm{spark}(\{w_\ell : \ell \in [m]\})$, see \cite{tropp2004greed}. Following \cite{tropp2004greed},
it is bounded from below by
\begin{align*}
\textrm{spark}(\{w_\ell : \ell \in [m]\}) &\geq \min\{k : \mu_1(k-1) \geq 1\},\\
\textrm{where}\quad
\mu_1(k-1) &:= \max_{\substack{\CI \subset [m]\\ \SN{\CI} = k-1}}\max_{j \not \in \CI}\sum\limits_{i \in \CI}\SN{\left\langle w_i, w_j\right\rangle}.
\end{align*}
Using the frame
property \eqref{eq:frame_properties_A_B}, we can bound
\begin{align*}
\mu_1(k-1) = \max_{\substack{\CI \subset [m]\\ \SN{\CI} = k-1}}\max_{j \not \in \CI}\sum\limits_{i \in \CI}\SN{\left\langle w_i, w_j\right\rangle} \leq
\sqrt{k-1}\max_{\substack{\CI \subset [m]\\ \SN{\CI} = k-1}}\max_{j \not \in \CI}\sqrt{\sum\limits_{i \in \CI} \left\langle w_i, w_j\right\rangle^2} \leq \sqrt{(k-1) \nu}.
\end{align*}
Taking additionally into account $\Ferror < \lceil\frac{m}{2}\rceil^{-1}$, it follows that
\begin{align*}
\textrm{spark}(\{w_\ell : \ell \in [m]\}) &\geq \min\{k : \mu_1(k-1) \geq 1\}
\geq \min\{k : \sqrt{(k-1) \nu} \geq 1\}\\
&= \min\left\{k : k \geq 1+\frac{1}{\nu}\right\} > 1 + \left\lceil\frac{m}{2}\right\rceil.
\end{align*}
The result follows by applying Lemma \ref{lem:example_rank_one_matrices}.
}
\end{proof}
\begin{lem}
\label{lem:bijection}
Let $\delta < 1$. For any $W \in \CW$ we have
\[
\N{P_{\hat \CW}(W)}_F \leq \N{W}_F \leq (1 - \Perror)^{-1}\N{P_{\hat \CW}(W)}_F.
\]
In particular $P_{\hat \CW} : \CW \rightarrow \hat \CW$ is a bijection.
\end{lem}
\begin{proof}
The right inequality follows by
\[
\N{P_{\hat \CW}(W)}_F = \N{W + P_{\hat \CW}(W) - P_{\CW}(W)}_F \geq \left(1 - \N{P_{\hat \CW} - P_{\CW}}_F\right)\N{W}_F \geq
\left(1 - \Perror\right)\N{W}_F.
\]
\end{proof}
\label{subsec:appendix_sec_4}
\begin{lem}
\label{lem:technical_coefficient_and_sigma}
Let $M \in \hat \CW \cap \bbS$ with $M = \sum_{i=1}^{m} \sigma_i \hat W_i$, and $Z \in \CW$ satisfy $M = P_{\hat \CW}(Z)$. Then
\begin{align}
\label{eq:max_bound_sigma}
\N{\sigma}_{\infty} &\leq \frac{1}{(1-\Perror)(1-\Ferror)},\\
\label{eq:technical_bound_dot_product}
\SN{\sigma_j \N{\hat W_j}_F^2 - \left\langle \hat W_j, M\right\rangle} &\leq  \frac{\Perror}{1-\Perror} + (\Perror + \Ferror)\N{\sigma}_{\infty}
\end{align}
Moreover, for any unit norm vector $v$ and any $\hat W_j$, we have
\begin{align}
\label{eq:quadratic_norm_estimate}
\SN{\N{\hat W_jv}^2 - v^T \hat W_j v} \leq 2\Perror.
\end{align}
\end{lem}
\begin{proof}
We first note that $1 = \N{M}_F =  \N{P_{\hat \CW}(Z)}_F \geq (1-\Perror)\N{Z}_F$ implies $\N{Z}_F \leq (1-\Perror)^{-1}$. For
\eqref{eq:max_bound_sigma}, we assume without loss of generality $\max \sigma_k = \N{\sigma}_{\infty}$ (otherwise we perform the proof for $-M$),
and denote $j = \argmax_i \sigma_i$. Then we have
\begin{align*}
(1-\Perror)^{-1} &\geq \N{Z}_F \geq \N{Z} \geq w_j^T Z w_j = \sum\limits_{i=1}^{m}\sigma_i \left\langle w_j, w_i\right\rangle^2
= \N{\sigma}_{\infty} + \sum\limits_{i\neq j}\sigma_i \left\langle w_j, w_i\right\rangle^2 \\
&\geq \N{\sigma}_{\infty}\left(1 - \sum\limits_{i\neq j} \left\langle w_j, w_i\right\rangle^2\right) \geq \N{\sigma}_\infty(1 - (\Uframe - 1)) \geq  \N{\sigma}_\infty(1 - \Ferror).
\end{align*}

\noindent
For \eqref{eq:technical_bound_dot_product} we first notice that
\begin{align*}
\left\langle \hat W_j, M\right\rangle &= \left\langle \hat W_j, \sum\limits_{i=1}^{m}\sigma_i \hat W_i \right\rangle =
\sigma_j \N{\hat W_j}_F^2 + \left\langle \hat W_j,\sum\limits_{i\neq j}\sigma_i  \hat W_i \right\rangle
=\sigma_j \N{\hat W_j}_F^2 + \left\langle \hat W_j, \sum\limits_{i\neq j}\sigma_i  W_i \right\rangle\\
&=\sigma_j \N{\hat W_j}_F^2 + \left\langle \hat W_j - W_j, \sum\limits_{i\neq j}\sigma_i  W_i \right\rangle +  \sum\limits_{i\neq j}\sigma_i  \left\langle W_j, W_i \right\rangle,
\end{align*}
and thus is suffices to bound the last two terms. For the first term we get
\begin{align*}
\SN{\left\langle \hat W_j - W_j, \sum\limits_{i\neq j}\sigma_i  W_i \right\rangle} \leq \Perror \N{\sum\limits_{i\neq j}\sigma_i  W_i}_F =  \Perror \N{Z - \sigma_j W_j}_F
\leq \frac{\Perror}{1-\Perror} +  \N{\sigma}_{\infty}\Perror,
\end{align*}
and for the second
\begin{align*}
\SN{\sum\limits_{i\neq j}\sigma_i \left\langle W_j, W_i\right\rangle} \leq \N{\sigma}_{\infty}\sum\limits_{i\neq j} \left\langle w_j, w_i\right\rangle^2 \leq \N{\sigma}_{\infty} \Ferror.
\end{align*}
\noindent
For \eqref{eq:quadratic_norm_estimate}, we first rewrite
\[
\SN{\N{\hat W_j u}_2^2 - u^T \hat W_j u} = \SN{\left\langle  \hat{W_j}^2, u\otimes u\right\rangle - \left\langle \hat W_j , u \otimes u\right\rangle} \leq \N{ \hat W_j^2 - \hat W_j }.
\]
Now denote $\Delta := \hat W_j - W_j$. Since $W_j^2 = W_j$ we have
\begin{align*}
\N{\hat W_j ^2 - \hat W_j } &= \N{\hat W_j ^2 - \hat W_j } = \N{(\Delta + W_j)^2 - W_j - \Delta}= \N{\Delta^2 + W_j \Delta + \Delta W_j - \Delta}\\
&= \N{\hat W_j \Delta - \Delta(\Id - W_j)} \leq \N{\Delta}\left(\N{\hat W_j } + \N{\Id - W_j}\right) \leq 2\Perror,
\end{align*}
since $\Id - W_j$ is a projection matrix onto $\operatorname{span}\{w_j\}^{\perp}$.
\end{proof}
\begin{proof}[Proof of Lemma \ref{lem:lower_bound_lambdam}]
We first calculate a lower bound for $\lambda_D$ in terms of the  $\min_i\sigma_i$ by
\begin{align*}
\lambda_D &= \sum\limits_{i=1}^{m}\sigma_i\left\langle u_D\otimes u_D,\hat W_i\right\rangle = \sum\limits_{i=1}^{m}\sigma_i\left\langle u_D\otimes u_D, W_i\right\rangle +
\sum\limits_{i=1}^{m}\sigma_i\left\langle u_D\otimes u_D, \hat W_i - W_i\right\rangle\\
&\geq \sum\limits_{i=1}^{m}\sigma_i\left\langle u_D, w_i\right\rangle^2 -
\N{\sum\limits_{i=1}^{m}\sigma_i  (\hat W_i - W_i)}_F \geq C\min_i \sigma_i - \N{P_{\hat\CW}(Z) - Z}_F \geq C\min_i \sigma_i - \frac{\Perror}{1-\Perror},
\end{align*}
where $C = \Lframe$ if $\min_i \sigma_i > 0$ and $C = \Uframe$ if $\min_i \sigma_i \leq 0$.
We are left with bounding $\sigma_{j^*} := \min_i \sigma_i$. Clearly, if $\sigma_{j^*} > 0$, the result follows immediately.
Therefore, we concentrate on the case $\sigma_{j^*}\leq 0 $ in the following.
We first use \eqref{eq:first_order_optimality} to get
\begin{align}
\lambda_1 \left\langle \hat W_{j^*}, M\right\rangle &= \left\langle \hat W_{j^*},u_1\otimes u_1\right\rangle =
\left\langle W_{j^*},u_1\otimes u_1\right\rangle + \left\langle \hat W_{j^*} - W_{j^*},u_1\otimes u_1\right\rangle \nonumber\\
&\geq \left\langle w_{j^*}, u_1\right\rangle^2 - \N{\hat W_{j^*} - W_{j^*}} \geq - \Perror. \label{eq:potato}
\end{align}
Applying now Lemma \ref{lem:technical_coefficient_and_sigma}, and $\Vert \hat W_{j^*}\Vert \geq 1- \Perror$,
we obtain from \eqref{eq:potato}
\begin{align*}
-\frac{\Perror}{\lambda_1}&\leq \left\langle \hat W_{j^*}, M\right\rangle \leq \sigma_{j^*}\N{\hat W_{j^*}}_F^2 + \SN{\sigma_{j^*}\N{\hat W_{j^*}}_F^2 - \left\langle \hat W_{j^*}, M\right\rangle}
\leq \sigma_{j^*}(1 - \Perror)^2 + \frac{\Perror}{1-\Perror} + (\Perror + \Ferror)\N{\sigma}_{\infty},\\
\Rightarrow \sigma_{j^*} &\geq -\frac{\Perror}{\lambda_1(1-\Perror)^2} - \frac{\Perror}{(1-\Perror)^3} - \frac{(\Perror + \Ferror)}{(1-\Perror)^2}\N{\sigma}_{\infty}.
\end{align*}
Using this in the previously derived bound for $\lambda_m$, and using $\Uframe < 1+\Ferror$, we have
\begin{align*}
\lambda_D &\geq \Uframe \sigma_{j^*}- \frac{\Perror}{1-\Perror} \geq
-(1+\Ferror)\left(\frac{\Perror}{\lambda_1(1-\delta)^2} + \frac{\Perror}{(1-\Perror)^3} +\frac{(\Perror + \Ferror)}{(1-\Perror)^2}\N{\sigma}_{\infty}\right) - \frac{\Perror}{1-\Perror}.
\end{align*}
Since $\Perror, \Ferror < 1/4$ we obtain from \eqref{eq:max_bound_sigma} that $\N{\sigma}_{\infty}\leq 2$, and
\begin{align*}
\lambda_D & \geq -2\Perror \lambda_1^{-1} - 3\Perror - 2(\Perror + \Ferror)\N{\sigma}_{\infty} \geq -2\Perror \lambda_1^{-1} - 8\Perror - 4\Ferror.
\end{align*}
 \end{proof}
\begin{lem}
\label{lem:technical_lemma_subsequences}
Let $(A,d)$ be a metric space and $F : A \rightarrow A$ be a continuous function. Let $(X_j)_{j \in \bbN}$ be
a sequence generated by $X_j = F^j(X_0)$ for some $X_0 \in A$, and assume $d(X_{j+1}, X_j) \rightarrow 0$. Then
any convergent subsequence of $(X_j)_{j \in \bbN}$ converges to a fixed point of $F$.
\end{lem}
\begin{proof}
Let $(X_{j_k})_{k \in \bbN}$ be a convergent subsequence of $(X_{j})_{j \in \bbN}$ with limit $\bar X = \lim_{k\rightarrow \infty} X_{j_k}$.
Then the subsequence $X_{j_k + 1}$ satisfies $d(X_{j_k + 1},\bar X) \leq d(X_{j_k + 1}, X_{j_k}) + d(X_{j_k}, \bar X) \rightarrow 0$
as $k\rightarrow \infty$, and thus also $(X_{j_k + 1})_{k \in \bbN}$ converges to $\bar X$. By construction $X_{j_k + 1} = F(X_{j_k})$. Taking the limit $k\rightarrow \infty$
on both sides, and using the continuity of $F$, we get
\[
\bar X = \lim_{k\rightarrow \infty} X_{j_k+1} =
\lim_{k\rightarrow \infty} F(X_{j_k}) = F\left(\lim_{k\rightarrow \infty} X_{j_k}\right) = F(\bar X).
\]
\end{proof}

\subsection{Proof of Proposition \ref{prop:representation_f_new_version}}
\label{subsec:proof_prop_22}
\begin{proof}[Proof of Proposition \ref{prop:representation_f_new_version}]
The first step is to replace first layer weights $A$ by $\hat A S_A$. This can be achieved
by inserting the permutation $\pi_A$ in the first layer and replacing by $\hat A S_A$ according to
\begin{equation}
\begin{aligned}
\label{eq:repara_f_1}
f(x) &= 1^\top \phi(B^\top \phi(A^\top x + \theta) + \tau) = 1^\top \phi(B^\top \pi_A \phi((A\pi_A)^\top x + \pi_A^\top \theta) + \tau) \\
&=1^\top \phi((\pi_A^\top B)^\top \phi(S_A \hat A^\top x + \pi_A^\top \theta) + \tau).
\end{aligned}
\end{equation}
Next we need to replace the matrix $\pi_A^\top B$ using $V$ respectively $\hat V$. Let
$n \in \bbR^{\Nouterlayer}$ be defined as $n_\ell = \N{AGb_\ell}^{-1}$, and $N = \diag(n)$. By definition
of the entangled weights, we have $V = AGBN$, implying the relation $B =  G^{-1}A^{-1}V N^{-1}$. Using assumptions
$A = \hat A S_A \pi_A^\top$ and $V = \hat VS_V\pi_V^\top$, and the properties $S_A^{-1} = S_A$, $\pi_A^{-1} = \pi_A^\top$, it follows that
\begin{align*}
\pi_A^\top B = \pi_A^\top G^{-1} \pi_A S_A \hat A^{-1} \hat V S_V \pi_V^\top N^{-1} = (\pi_A^\top G \pi_A)^{-1} S_A \hat A^{-1} \hat V S_V \pi_V^\top N^{-1}
\end{align*}
Since $G = \diag(\phi'(\theta))$, we have $\pi_A^\top G \pi_A = \diag(\pi_A^\top(\phi'(\theta))) = \diag((\phi'(\pi_A^\top\theta)))=:\tilde G$.
Inserting into \eqref{eq:repara_f_1}, we get
\begin{align*}
f(x) &=1^\top \phi(N^{-1} \pi_V S_V \hat V^\top \hat A^{-\top} S_A \tilde G^{-1} \phi(S_A \hat A^\top x + \pi_A^\top \theta) + \tau)
\end{align*}
The dot product with a $1$-vector is permutation invariant, hence we can get an additional $\pi_V^\top$ into the second layer. Then,
using that the diagonal matrix $\tilde N := \pi_V^\top N \pi_V$ commutes with $S_V$ we get
\begin{align*}
f(x) & =1^\top \phi(\tilde N S_V \hat V^\top \hat A^{-\top} S_A \tilde G^{-1} \phi(S_A \hat A^\top x + \pi_A^\top \theta) + \pi_V^\top \tau) \\
&= 1^\top \phi(\tilde N^{-1} S_V \hat V^\top \hat A^{-\top} S_A \tilde G^{-1} \phi(S_A \hat A^\top x + \pi_A^\top \theta) + \pi_V^\top \tau)\\
&= 1^\top \phi(S_V \tilde N^{-1}  \hat V^\top \hat A^{-\top} S_A \tilde G^{-1} \phi(S_A \hat A^\top x + \pi_A^\top \theta) + \pi_V^\top \tau).
\end{align*}
It remains to show that $\hat B = \tilde G^{-1}  S_A \hat A^{-1} \hat V \tilde N^{-1}$, which is implied if 	$\tilde N_{\ell \ell} = \Vert \tilde G^{-1} S_A \hat A^{-1}\hat v_\ell\Vert$.
By the normalization property $\N{b_\ell} = 1$ (see Definition \ref{def:f}) and $B =  G^{-1}A^{-1}V N^{-1}$, we first have
\begin{align*}
1 = \N{b_\ell} = \frac{\N{G^{-1}A^{-1}v_\ell}}{n_\ell}\quad \textrm{ and thus }\quad n_\ell = \N{G^{-1}A^{-1}v_\ell}.
\end{align*}
Using this, and the assumptions $A^{-1} = \pi_A S_A \hat A^{-1}$, $V \pi_V = \hat V S_V$, we obtain
\begin{align*}
\tilde N = \pi_V^\top N \pi_V &= \diag(\pi_V^\top n) = \diag\left(\pi_V^\top\left(\N{G^{-1}A^{-1}v_1},\ldots,\N{G^{-1}A^{-1}v_{\Nouterlayer}}\right)\right)\\
&=\diag\left(\left(\N{G^{-1}A^{-1}(V\pi_V)_1},\ldots,\N{G^{-1}A^{-1}(V\pi_V)_{\Nouterlayer}}\right)\right)\\
&=\diag\left(\left(\N{G^{-1}\pi_A S_A \hat A^{-1}(\hat V S_V)_1},\ldots,\N{G^{-1}\pi_A S_A \hat A^{-1}(\hat V S_V)_{\Nouterlayer}}\right)\right)\\
&=\diag\left(\left(\N{G^{-1}\pi_A S_A \hat A^{-1}\hat v_1},\ldots,\N{G^{-1}\pi_A S_A \hat A^{-1}\hat v_{\Nouterlayer}}\right)\right),
\end{align*}
where we used that $S_V$ affects $v_{\ell}$ only by multiplication with $\pm 1$. The result follows since $\pi_A^\top$
is orthogonal and thus $\Vert G^{-1}\pi_A S_A \hat A^{-1}\hat v_\ell \Vert = \Vert \pi_A^\top G^{-1}\pi_A S_A \hat A^{-1}\hat v_\ell \Vert=
 \Vert \tilde G^{-1} S_A \hat A^{-1}\hat v_\ell\Vert$.
\end{proof}

\bibliographystyle{alpha}

\end{document}